\documentclass{article}

\usepackage{microtype}
\usepackage{listings}
\usepackage{graphicx}
\usepackage{subcaption}
\usepackage{booktabs} % for professional tables
\usepackage{placeins}
\usepackage{hyperref}

\usepackage[accepted]{arxiv}

\usepackage{amsmath}
\usepackage{amssymb}
\usepackage{mathtools}
\usepackage{amsthm}
\usepackage{bbm}
\usepackage{bm}
\usepackage{enumitem}
\usepackage[capitalize,noabbrev]{cleveref}

\theoremstyle{plain}
\newtheorem{theorem}{Theorem}[section]
\newtheorem{proposition}[theorem]{Proposition}

\theoremstyle{definition}

\theoremstyle{remark}

\definecolor{mutedred}{rgb}{0.8,0.3,0.3}

\icmltitlerunning{Reducing Per-Sample Interference in Stochastic Optimization} % don 't forget incoming information

\begin{document}

\twocolumn[
  \icmltitle{Reducing Per-Sample Interference in Stochastic Optimization}

  \begin{icmlauthorlist}
    \icmlauthor{Apostolos Avranas$^{\dagger}$}{}
  \end{icmlauthorlist}
  % \icmlaffiliation{comp}{Amadeus, 821 Av. Jack Kilby, 06270, France}

  %\icmlcorrespondingauthor{Apostolos Avranas}{apostolos.avranas@amadeus.com}
  \icmlkeywords{Optimization, Non-catastrophic forgetting, Per-sample gradients}

  \vskip 0.1in
]

% this must go after the closing bracket ] following \twocolumn[ ...
% This command actually creates the footnote in the first column listing the
% affiliations and the copyright notice. The command takes one argument, which
% is text to display at the start of the footnote. The \icmlEqualContribution
% command is standard text for equal contribution. Remove it (just {}) if you
% do not need this facility.
% Use ONE of the following lines. DO NOT remove the command.
% If you have no special notice, KEEP empty braces:
% Or, if applicable, use the standard equal contribution text:
% \printAffiliationsAndNotice{\icmlEqualContribution}
\begin{abstract}
Modern optimizers combine gradients from the current mini-batch with historical optimization state, such as momentum or adaptive moments. While effective, this standard practice can produce parameter updates that actively increase the loss of \mbox{individual} samples. We term this phenomenon \textbf{per-sample interference} and propose redefining the parameter update as an optimization problem that explicitly minimizes it. 

Because the exact formulation of the problem is computationally prohibitive, we introduce a highly efficient surrogate. By reducing the problem's dimensionality to the batch size and restricting the optimization to the last linear layer, we overcome memory and speed bottlenecks. This strategy hinges on our unexpected finding that this layer alone can reliably capture core second-order statistics of the full network. The resulting surrogate problem integrates readily into standard optimizers like SGD and AdamW, and can be solved using a small number of GPU-friendly iterations. Crucially, the method exhibits \textbf{favorable scaling properties}, as the relative computational overhead shrinks as the model size or input grows. Experiments on image classification benchmarks confirm \emph{reduced per-sample interference} and \emph{improved generalization}. We open-source our associated code.\footnotemark
\begingroup
\let\thefootnote\relax
\footnotetext{\hspace*{-2em}\textsuperscript{1} \href{https://github.com/avranasa/sample-harm-reduction}{https://github.com/avranasa/sample-harm-reduction}}
\endgroup
\end{abstract}

\printAffiliationsAndNotice{$^{\dagger}$Amadeus, 821 Av. Jack Kilby, 06270, France}  % custom notice

\section{Introduction}

Stochastic optimization is the backbone of modern deep learning. In practice, parameter updates are computed by aggregating gradients over a mini-batch and combining them with historical information such as momentum or adaptive moments. These mechanisms are crucial for stabilizing training, yet they implicitly assume that combining batch-averaged gradients with past optimization state is always beneficial for learning from the current mini-batch.

In this work, we challenge this assumption. The interaction between historical optimization state and batch-averaged gradients can produce parameter updates that actively increase the loss of individual samples in the current batch. While this effect may be obscured at the batch level, it represents a systematic loss of information at the per-sample level. We refer to this phenomenon as \emph{interference}. From this perspective, standard optimizers may inadvertently trade individual sample improvement for optimization stability.

To correct this imbalance, our goal is to protect each sample from the potentially conflicting influence of batch averaging and historical optimization states. We introduce a \emph{novel} framework that \textbf{formalizes the parameter update as an optimization problem, explicitly minimizing the interfering influence on samples within the current mini-batch.}

Unfortunately, directly solving the resulting optimization problem is practically prohibitive. We employ a series of steps to derive a surrogate problem admitting a computationally efficient solution:
\begin{itemize}[leftmargin=1.4em, itemsep=1.0em, parsep=0pt, topsep=0pt]
\item \emph{Dimensionality reduction:} We mathematically transform the optimization problem out of parameter space and into the drastically smaller batch space.
\item \emph{Last-layer approximation:} The formulation involves the full-network per-sample gradients whose computation adds a massive overhead. We demonstrate that with careful manipulation, the problem can be approximated using the per-sample gradients of only the last linear layer. This relies on our surprising finding that this layer alone reliably captures the core second-order statistics of the full-network per-sample gradients, making the surrogate problem highly representative of the original.
\item \emph{Implicit gradient computation:} Even for the last layer, materializing per-sample gradients poses a bottleneck. We leverage Khatri-Rao factorization properties to bypass materializing these gradients, completely eliminating the associated memory and speed overheads.
\item \emph{Efficient solver:} We propose a tailored solver for the resulting surrogate problem that rapidly converges using a small number of GPU-friendly iterations.
\end{itemize}

Overall, our framework is optimizer-agnostic and can be incorporated into widely used methods such as Stochastic Gradient Descent (SGD) with momentum and AdamW. Importantly, the introduced computational and memory overhead is small and naturally shrinks as the model size or input dimensions scale up. Empirically, we confirm that the proposed method successfully mitigates interference to individual samples in the current batch. Furthermore, evaluations on image classification benchmarks show improved generalization, suggesting that controlling update interference is a practical and complementary dimension for improving deep learning optimization.

\section{Literature Review}
Our work is most closely related to research in continual learning, where the objective is to prevent catastrophic forgetting of previously acquired knowledge when learning new tasks. In contrast, our setting is fundamentally different: rather than protecting past knowledge from new data, we aim to prevent the interfering influence of past optimization state on the incoming data.

Within the continual learning literature, the closest method to our approach is Gradient Episodic Memory (GEM) \cite{GEM_for_continual_learning}, which uses a similar notion of interference that is based on gradient alignment. To determine the update direction, GEM formulates an optimization problem  whose dimensionality is reduced from the parameter space to the number of samples. However, challenges such as reliance on large episodic memories and the efficient solution of the resulting optimization problem are not fully addressed. Subsequent works \cite{A-GEM_for_continual_learning, epsilon_A-GEM} alleviate some of these limitations by replacing episodic memories with averaged gradients over past samples. Related approaches enforce orthogonality or projection-based constraints between gradients of different tasks \cite{OrthogonalGD_CL, GradientProjectionMemory_CL}. Finally, recent work \cite{LLM_catastrophic_forgetting} shows that catastrophic forgetting effects persist in contemporary large language models, with severity increasing as model size grows.

By viewing each sample within the current mini-batch as a distinct task, we can draw parallels with multi-task learning (MTL). In \cite{game_theory_MTL}, each task is modeled as a player in a game, leading to an optimization problem that learns weights over per-task gradients and relies on a task–task Gram matrix capturing pairwise gradient interactions. In \cite{Multi-Task_ProjectGrads}, the authors use a notion of inter-task interference similar to ours and, analogously to \cite{A-GEM_for_continual_learning}, propose projecting averaged gradients to reduce conflicts across tasks.  In \cite{orthogonal_decomposition_MTL}, orthogonal decomposition is used to avoid interference between gradients of different tasks. In \cite{min_worst_conflict_MTL}, they consider weighting the gradients of each task so as to minimize the worst conflict. In \cite{equal_speed_MTL}, the weights of each task aim to equalize the learning speed across tasks while in \cite{orthonormal_gradients_MTL} they aim to construct orthonormal gradients between different tasks. 

In contrast to the above works, \cite{adatask_optimizer_MTL} explicitly accounts for the base optimizer and proposes adaptations for the MTL setting.  Similarly, our goal is to adapt practical optimizers to alter the optimization trajectory in a way that reduces, at each iteration, the harmful impact on the samples of the current mini-batch. In this work we focus on two widely used optimizers: AdamW \cite{AdamW} and SGD with momentum \cite{sutskever_sgd, nesterov_sgd}. AdamW is a variant of Adam \cite{adam_optimizer} that decouples weight decay from the gradient-based update. Other strong Adam-style methods include Adafactor \cite{adafactor_optimizer_1, adafactor_optimizer_2}, which reduces memory usage, LAMB \cite{lamb_optimizer}, which employs layer-wise adaptation, and RAdam \cite{radam_optimizer}, which corrects the variance of the adaptive learning rate. For SGD-style methods, competitive optimizers used in large-scale training include MUON \cite{muon_optimizer}, which encourages orthogonal updates, and Lion \cite{lion_optimizer}, which uses the sign of momentum.

An important element of this work is working with per-sample gradients without incurring excessive memory and computational costs. Per sample gradients and statistics like their norms are useful in various contexts like differential privacy \cite{use_per_example_g_differential_privacy}, predicting generalization \cite{use_per_example_g_gradient_diversity} and developing new optimizers \cite{use_per_example_g_optimizers}. Unfortunately, computing per-sample gradients is demanding. For convolutional neural networks, \cite{use_per_example_g_conv} propose a method to compute them more efficiently. Also libraries like Opacus \cite{use_per_example_g_opacus_library} and Backpack \cite{use_per_example_g_backpack_library} were developed for that purpose. We note that computing full per-sample gradients is costly in both memory and runtime. However, an efficient technique exists to compute their norms without explicitly calculating the full gradients. This trick was used in \cite{use_per_example_g_optimizers, use_per_example_g_differential_privacy_llm, use_per_example_g_differential_privacy_low_cost}.

\section{Motivation and Problem Formulation}

Let $F(\mathbf{x};\boldsymbol{\theta})$ denote a model with parameters $\boldsymbol{\theta} \in \mathbb{R}^{\Theta}$ processing input $\mathbf{x}$. Standard training proceeds in mini-batches of $B$ samples $\{\mathbf{x}_i\}_{i=1}^{B}$, each with a per-sample loss $\mathcal{L}_i(\boldsymbol{\theta})=\mathcal{L}(F(\mathbf{x}_i;\boldsymbol{\theta}))$. In practice, standard optimizers such as SGD and AdamW \cite{sutskever_sgd, adam_optimizer, AdamW} update parameters using only the \emph{aggregated} batch gradient $\mathbf{g}=\nabla_{\boldsymbol{\theta}} \frac{1}{B}\sum_{i=1}^{B} \mathcal{L}_i(\boldsymbol{\theta})$, combined with an internal optimization state (e.g., momentum). While this process reliably decreases the \textit{average} loss, the ultimate goal is typically to perform well on \textit{every} sample. This discrepancy is particularly evident in classification tasks, where cross-entropy is a surrogate for accuracy. Minimizing the mean loss does not guarantee improvement for every sample and may even fail to increase the number of correctly classified examples.

Beyond this mismatch in objectives, batch aggregation also discards valuable information about the batch: we lose visibility into per-sample contributions and conflicts, i.e., which examples are helped and which are interfered with by a parameter update.

Finally, practical optimizers blend the current batch gradient with historical information, such as momentum or adaptive moments. Consequently, the actual parameter update may deviate significantly from the descent direction of the current batch, potentially increasing the loss for a subset of the \emph{current} samples. In this sense, the update can ``catastrophically'' interfere with the current batch samples. This motivates a per-sample view of the update: rather than considering only the aggregated batch gradient, we seek robust updates that respect the descent direction of each sample.

Let $\mathbf{G} \in \mathbb{R}^{B \times \Theta}$ be the matrix of per-sample gradients, where each row  $\mathbf{G}_i = \frac{1}{B}\nabla_{\boldsymbol{\theta}} \mathcal{L}_i(\boldsymbol{\theta})$ corresponds to the gradient of the $i$-th sample. Using a first-order Taylor expansion, the change in loss for sample $i$ is:
\begin{align}
    \label{eq:probability-approximation}
   \mathcal{L}_i(\boldsymbol{\theta}+\Delta\boldsymbol{\theta}) &\approx \mathcal{L}_i(\boldsymbol{\theta}) + B\mathbf{G}_i^\top \Delta\boldsymbol{\theta}
\end{align} 
When $\mathbf{G}_i^\top \Delta\boldsymbol{\theta} > 0$, the loss for sample $i$ increases.  We characterize this increase as ``interference". According to Taylor's first-order approximation, the severity of the interference is proportional to $\max(0, \mathbf{G}_i^\top \Delta\boldsymbol{\theta})$. Ideally, we seek a direction $\Delta\boldsymbol{\theta}$ that improves \textit{every} sample simultaneously. In practice, however, conflicting gradients make this \textit{impossible}. 

Since we cannot guarantee that no sample is interfered with, our goal shifts to minimizing the total interference across the mini-batch.  We construct an objective $J(\Delta\boldsymbol{\theta})$ that linearly penalizes sample interference (i.e., when $\mathbf{G}_i^\top \Delta\boldsymbol{\theta} > 0$) while assigning zero penalty to directions that successfully decrease a sample's loss ($\mathbf{G}_i^\top \Delta\boldsymbol{\theta} \le 0$). Defining the helper function:
$$H(\mathbf{v}):= \sum_{i=1}^{B} \max(0, [\mathbf{v}]_i),$$
where $[\mathbf{v}]_i$ denotes the $i$-th element of vector $\mathbf{v}\in\mathbb{R}^B$., we can express our objective as $J(\Delta\boldsymbol{\theta}) := H(\mathbf{G} \Delta\boldsymbol{\theta})$. We therefore seek the optimal parameter update:
\begin{equation}\label{eq:conceptual_problem}
\Delta\boldsymbol{\theta}^\star = \arg\min_{\Delta\boldsymbol{\theta}} \; J(\Delta\boldsymbol{\theta}).   
\end{equation}

\subsection{Introducing a Constraint on the Update Direction}

Solving \cref{eq:conceptual_problem} directly presents several pitfalls. First, the trivial solution $\Delta\boldsymbol{\theta}=\mathbf{0}$ is technically optimal (zero interference) but ineffective for training the model. Second, \cref{eq:conceptual_problem} relies solely on the gradients of the current batch. In contrast, modern practical optimizers compute at each iteration a parameter update $\Delta\boldsymbol{\theta}_b$ that incorporates information from previous iterations (e.g., via momentum for SGD or adaptive moments for Adam/AdamW). This historical context is crucial for stabilizing the optimization trajectory. Deviating significantly from the standard update $\Delta\boldsymbol{\theta}_b$ discards this valuable information.

To address those issues, we constrain the proposed update $\Delta\boldsymbol{\theta}$ to remain within the vicinity of the standard optimizer's update $\Delta\boldsymbol{\theta}_b$. Formally, we restrict $\Delta\boldsymbol{\theta}$ to lie within a ball (measured by the Euclidean norm $\|\cdot\|$) centered at $\Delta\boldsymbol{\theta}_b$, with a radius proportional to the magnitude of $\Delta\boldsymbol{\theta}_b$:
\begin{equation}\label{eq:opt_problem_initial}
\Delta\boldsymbol{\theta}^\star {=}\arg\min_{\Delta\boldsymbol{\theta}} H(\mathbf{G} \Delta\boldsymbol{\theta}) \;\text{s.t.}\; \|\Delta\boldsymbol{\theta} {-} \Delta\boldsymbol{\theta}_b\| {\leq} c \|\Delta\boldsymbol{\theta}_b\|,
\end{equation}
where $c > 0$ is a hyperparameter that controls the radius of the \emph{trust region}. This formulation ensures that while we actively seek to minimize the interference to individual samples, we do so without negating the general descent direction $\Delta\boldsymbol{\theta}_b$ determined by the global loss and the optimizer's state.

Importantly, the optimization problem in \eqref{eq:opt_problem_initial} is convex, ensuring it can be solved efficiently. While other formalizations could be employed to mitigate sample interference, alternative designs require careful consideration. For instance, modifying the objective to minimize the discrete count of interfered-with samples renders the problem combinatorial and non-convex. Furthermore, in our initial experiments, we explored a GEM-like formulation \cite{GEM_for_continual_learning}, solving  $\min_{\Delta\boldsymbol{\theta}} ||\Delta\boldsymbol{\theta} - \Delta\boldsymbol{\theta}_b||^2$ s.t. $\mathbf{G}_i^\top \Delta\boldsymbol{\theta} \le 0\; \forall i$. However, we found its constraint to be often overly restrictive, forcing the solution $\Delta\boldsymbol{\theta}$ substantially far from the reference $\Delta\boldsymbol{\theta}_b$ and causing severe optimization instability.

\section{Challenges}

\cref{eq:opt_problem_initial} provides a principled framework  for reducing  per-sample interference. However, directly integrating it into a standard training loop presents several practical challenges:
\begin{itemize}[leftmargin=1.0em, parsep=0pt, topsep=0pt]
    \item \textbf{High Dimensionality:} The variable $\Delta\boldsymbol{\theta}$ lies in $\mathbb{R}^{\Theta}$. Since $\Theta$ is typically in the millions or billions for modern models, performing constrained optimization over such a high-dimensional space is computationally prohibitive.
    \item \textbf{Per-Sample Gradient Cost:} \cref{eq:opt_problem_initial} involves the per-sample gradient matrix $\mathbf{G}$. Computing and storing it incurs massive memory and runtime overheads. It is crucial to avoid explicitly materializing this matrix.
    \item \textbf{Solver Efficiency:} The solver for \cref{eq:opt_problem_initial} must be highly efficient and GPU-friendly to ensure that the training throughput is not significantly compromised.
\end{itemize}
In the following subsections, we address these challenges systematically.
%A key feature of reverse-mode Automatic Differentiation is that it never stores the gradients of individual elements of a batch. This increases the memory efficiency of the process

\subsection{High Dimensionality of Optimization Variables}
The  dimensionality of the problem can be reduced from the number of parameters $\Theta$ to the batch size $B$ (note $B \ll \Theta$).

\begin{proposition}
\label{prop:optimal_form}
Let $\mathcal{S}^*$ denote the set of optimal solutions of the optimization problem \eqref{eq:opt_problem_initial}. Then:
\begin{enumerate}
    \item There exists an optimal solution $\Delta\boldsymbol{\theta}^* \in \mathcal{S}^*$ such that:
    \begin{equation} \label{eq:form_of_optimal_dtheta}
        \Delta\boldsymbol{\theta}^* = \Delta\boldsymbol{\theta}_b + \mathbf{G}^\top \mathbf{w},
    \end{equation}
    where $\mathbf{w} \in \mathbb{R}^{B}$ and $\mathbf{w} \le \mathbf{0}$.
    \item Among all optimal solutions $\Delta\boldsymbol{\theta} \in \mathcal{S}^*$, the unique solution that minimizes the distance to the update $\Delta\boldsymbol{\theta}_b$ (i.e.,  $\min_{\Delta\boldsymbol{\theta} \in \mathcal{S}^*}\|\Delta\boldsymbol{\theta} - \Delta\boldsymbol{\theta}_b\|$) satisfies the form given in \cref{eq:form_of_optimal_dtheta}.
\end{enumerate}
\end{proposition}
\begin{proof}
See Appendix \ref{appendix:proof-proposition-1} for the proof.
\end{proof}

Proposition \ref{prop:optimal_form} allows us to optimize over $\mathbf{w} \in \mathbb{R}^B$ instead of $\Delta\boldsymbol{\theta} \in \mathbb{R}^{\Theta}$. Substituting $\Delta\boldsymbol{\theta} = \Delta\boldsymbol{\theta}_b + \mathbf{G}^\top \mathbf{w}$ into \cref{eq:opt_problem_initial}, the optimization problem transforms into:
\begin{equation}\label{eq:opt_problem_initial_dual}
\begin{aligned}
\mathbf{w}^\star & = \arg\min_{\mathbf{w} \in \mathbb{R}^B}\;  H(\mathbf{Q} \mathbf{w} + \mathbf{G} \Delta\boldsymbol{\theta}_b) \\
&\text{s.t.}\;  \mathbf{w}^\top \mathbf{Q} \mathbf{w} \le c^2 \|\Delta\boldsymbol{\theta}_b\|^2, \; \mathbf{w} \leq \mathbf{0},
\end{aligned}
\end{equation}
where $\mathbf{Q} = \mathbf{G}\mathbf{G}^\top \in \mathbb{R}^{B \times B}$ is the Gram matrix (symmetric and positive semi-definite) of the per-sample gradients. 

A crucial property of the problem's defining quantities (i.e. $\mathbf{Q}, \mathbf{G} \Delta\boldsymbol{\theta}_b$, and  $\|\Delta\boldsymbol{\theta}_b\|^2$) is their \textbf{layer-wise decomposability}. Consider a network composed of $L$ layers, $F = F_L \circ F_{L-1} \circ \cdots \circ F_1$, where each layer $F_l(\cdot;\boldsymbol{\theta}_l)$ has its distinct parameters $\boldsymbol{\theta}_l \in \mathbb{R}^{\Theta_l}$. Let $\mathbf{G}_l \in \mathbb{R}^{B \times \Theta_l}$ denote the per-sample gradient matrix for layer $l$, and $\Delta\boldsymbol{\theta}_{b,l} \in \mathbb{R}^{\Theta_l}$ the corresponding slice of the update vector $\Delta\boldsymbol{\theta}_b$. With the full gradient matrix  $\mathbf{G} = [\mathbf{G}_1, \dots, \mathbf{G}_L]$ and partial Gram matrix $\mathbf{Q}_l = \mathbf{G}_l \mathbf{G}_l^\top$, the defining quantities decompose as:
\[
\mathbf{Q} {{=} {\sum_{l=1}^L}} \mathbf{Q}_l, \; \mathbf{G} \Delta\boldsymbol{\theta}_b {{=} {\sum_{l=1}^L}} \mathbf{G}_l \Delta\boldsymbol{\theta}_{b,l}, \;\|\Delta\boldsymbol{\theta}_b\|^2 {=} {\sum_{l=1}^L} \|\Delta\boldsymbol{\theta}_{b,l}\|^2.
\]

\subsection{Naive solutions do not work}
Although \cref{eq:opt_problem_initial_dual} mitigates the dimensionality bottleneck, the strict dependence on per-sample gradients remains, required \textit{both} to form the reduced problem (due to $\mathbf{Q}$ and $\mathbf{G}\Delta\boldsymbol{\theta}_b$ involving $\mathbf{G}$) \emph{and} to reconstruct the update $\Delta\boldsymbol{\theta}^\star=\Delta\boldsymbol{\theta}_b+\mathbf{G}^\top\mathbf{w}^\star$.

Before deriving our solution, we note that two straightforward implementation strategies fail in practice:

- \textbf{Memory-prohibitive approach (store full $\mathbf{G}$).} Compute and store $\mathbf{G}$. Solve \cref{eq:opt_problem_initial_dual} for $\mathbf{w}^\star$, and finally compute $\mathbf{G}^\top \mathbf{w}^\star$. In practice, this approach is only feasible for very small models, as it scales memory by a factor of $B$.

- \textbf{Slow approach (second backward).} During the backward pass, accumulate the quantities $\mathbf{Q}_l$, $\mathbf{G}_l \Delta\boldsymbol{\theta}_{b,l}$, and $\|\Delta\boldsymbol{\theta}_{b,l}\|^2$ per layer (without storing $\mathbf{G}_l$) to form \cref{eq:opt_problem_initial_dual}. Solve for $\mathbf{w}^\star$ and obtain $\mathbf{G}^\top\mathbf{w}^\star$ via a second backward pass on the weighted loss $\frac{1}{B}\sum_{i=1}^{B} [\mathbf{w}^\star]_i \mathcal{L}_i$. Beyond requiring custom backward kernels at each layer to produce these quantities, the second backward pass adds significant runtime overhead.

In the next subsection, we introduce an approximation that avoids explicitly forming $\mathbf{G}$ while still enabling efficient construction of \cref{eq:opt_problem_initial_dual}.

\subsection{Approximating the Full Quantities via Last Layers}
We avoid requiring the full per-sample gradient matrix $\mathbf{G}$ by restricting computations to the last few layers (from $\ell$ to $L$) and utilizing the partial quantities $\mathbf{Q}_{\ell{:}L} = \sum_{k=\ell}^L \mathbf{G}_k \mathbf{G}_k^\top$, $(\mathbf{G}\Delta\boldsymbol{\theta}_b)_{\ell{:}L} = \sum_{k=\ell}^L \mathbf{G}_k \Delta\boldsymbol{\theta}_{b,k}$, and $\|\Delta\boldsymbol{\theta}_{b,\ell{:}L}\|^2 = \sum_{k=\ell}^L \|\Delta\boldsymbol{\theta}_{b,k}\|^2$. Specifically we propose the simple approximation that the Gram matrix of the full network is proportional to that of the last layers: 
\begin{equation}\label{eq:last_layers_approximation_Q}
\begin{aligned}
\mathbf{Q} \approx \alpha \mathbf{Q}_{\ell{:}L},
\end{aligned}
\end{equation}
for an unknown scaling factor $\alpha > 0$. 

Remarkably, under vanilla SGD (i.e. $\Delta\boldsymbol{\theta}_b = -\mathrm{lr}\cdot \mathbf{G}^\top \mathbf{1}$ where $\mathbf{1}$ is the all-ones vector), this single assumption propagates to the remaining quantities:
$$\mathbf{G}\Delta\boldsymbol{\theta}_b = -\mathrm{lr} \mathbf{Q}\mathbf{1} \approx \alpha \big({-\mathrm{lr}} \mathbf{Q}_{\ell{:}L}\mathbf{1}\big) = \alpha (\mathbf{G}\Delta\boldsymbol{\theta}_b)_{\ell{:}L},$$
$$\|\Delta\boldsymbol{\theta}_b\|^2 = \mathrm{lr}^2 \mathbf{1}^\top \mathbf{Q} \mathbf{1} \approx \alpha \big(\mathrm{lr}^2 \mathbf{1}^\top \mathbf{Q}_{\ell{:}L} \mathbf{1}\big) = \alpha \|\Delta\boldsymbol{\theta}_{b,\ell{:}L}\|^2.$$

What makes this result powerful is that when we inject these scaled approximations into \cref{eq:opt_problem_initial_dual}, \textbf{the unknown factor $\alpha$ completely vanishes} from both the objective function and the constraints. The approximated surrogate problem is invariant to $\alpha$,  mapping perfectly back to \cref{eq:opt_problem_initial_dual}  but with all defining quantities replaced by their partial, last-layer counterparts:
\begin{equation}\label{eq:opt_problem_last_layers_approx}
\begin{aligned}
\mathbf{w}^\star & = \arg\min_{\mathbf{w} \in \mathbb{R}^B}\; H( \mathbf{Q}_{\ell{:}L} \mathbf{w} + (\mathbf{G} \Delta\boldsymbol{\theta}_b)_{\ell{:}L}) \\
&\text{s.t.}\; \mathbf{w}^\top \mathbf{Q}_{\ell{:}L} \mathbf{w} \le c^2 \|\Delta\boldsymbol{\theta}_{b,\ell{:}L}\|^2, \; \mathbf{w} \leq \mathbf{0},
\end{aligned}
\end{equation}

\begin{figure}[h!]
    \centering
    \begin{subfigure}[b]{0.49\columnwidth}
        \centering
        \includegraphics[width=\textwidth]{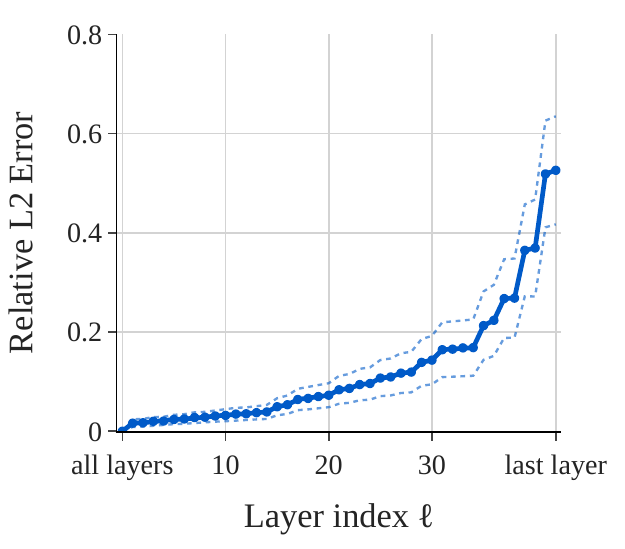}
        \includegraphics[width=\textwidth]{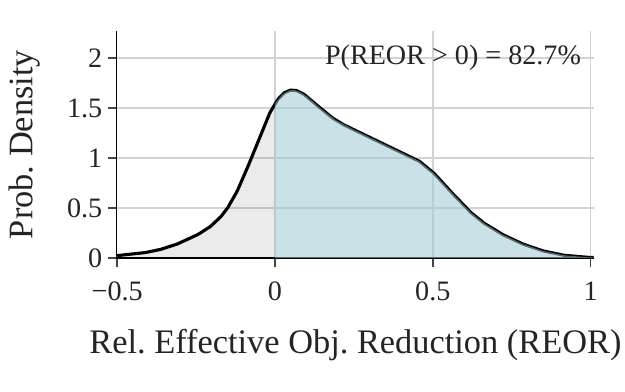}
        \caption{ResNet-20, CIFAR-10}
    \end{subfigure}
    \begin{subfigure}[b]{0.49\columnwidth}
        \centering
        \includegraphics[width=\textwidth]{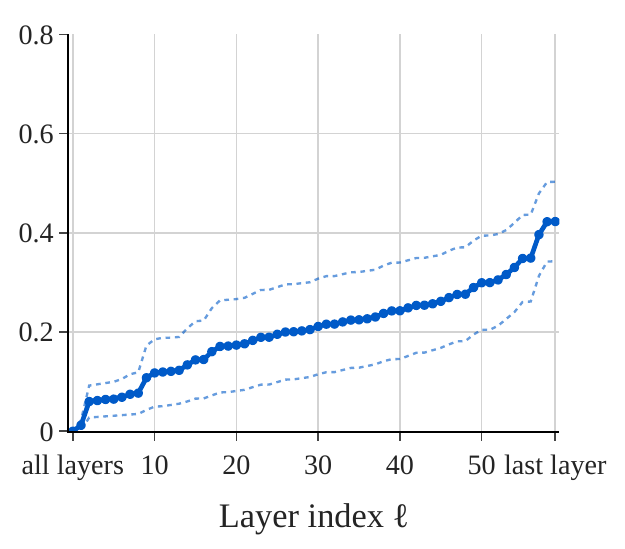}
        \includegraphics[width=\textwidth]{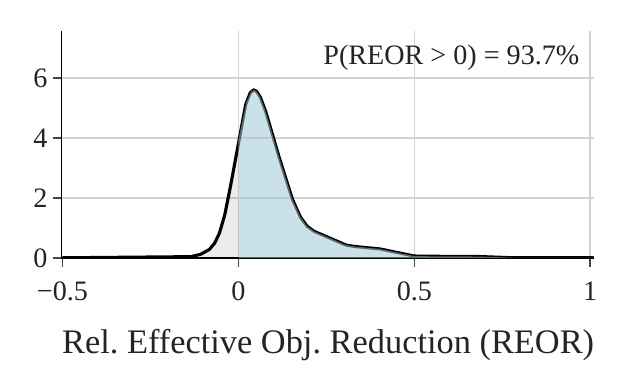}
        \caption{ViT-small, CIFAR-100}     
    \end{subfigure}
    \caption{Approximation quality at epoch 100. \textbf{Top Row:} Relative $L_2$ error of the Gram matrix approximation, $\frac{\|\mathbf{Q} - \alpha \mathbf{Q}_{\ell:L}\|_F }{ \|\mathbf{Q}\|_F}$, as a function of included layers (lower is better). The solid curve shows the mean, and the two dotted curves show one standard deviation above and below the mean at each point. \textbf{Bottom Row:} Distribution of the relative \textit{effective} objective reduction (REOR), $\frac{J(\Delta\boldsymbol{\theta}_b) - J(\Delta\boldsymbol{\theta}^\star)}{J(\Delta\boldsymbol{\theta}_b)}$, on the original problem (higher is better).}
    \label{fig:last_layers_approx_quality}
\end{figure}

To confirm the surprising robustness of this approximation we conduct the following experiments:
\paragraph{Experiment: Approximation Quality of $\mathbf{Q}$.}
We measure how much information about the full-network \emph{second-order statistics} $\mathbf{Q}$ is retained by the last layers, i.e., how accurate \cref{eq:opt_problem_last_layers_approx} is. To evaluate this, we compute the full per-sample gradients $\mathbf{G}$ and with that $\mathbf{Q}$ and $\mathbf{Q}_{\ell:L}$ for every layer $\ell$. In the top row of \cref{fig:last_layers_approx_quality} we depict the relative Frobenius error $\|\mathbf{Q} - \alpha \mathbf{Q}_{\ell:L}\|_F / \|\mathbf{Q}\|_F$, where $\alpha$ is chosen to minimize this error; this scaling is appropriate because the surrogate problem is invariant to $\alpha$ (details in \cref{appendix:additional_figures_approx}). The figures reveal that the error diminishes rapidly as earlier layers are included. Even restricting solely to the last linear layer ($\ell=L$) retains significant structural information, maintaining an average error of around 0.5 or less.\footnote{We note that for the ViT model, the matrices incorporate AdamW weighting, i.e. $\mathbf{Q}_1$ and $\mathbf{Q}_{1,\ell:L}$ (defined in \cref{sec:SGD and AdamW adaptations}).}

\paragraph{Experiment: End-to-End Interference Reduction.}
Our ultimate goal is to incorporate per-sample interference reduction into advanced optimizers (SGD with momentum, AdamW). However, their intricate update rules make applying the last-layer approximation in a strict theoretical manner intractable. To circumvent this, we propose using \cref{eq:opt_problem_last_layers_approx} as a general surrogate and adapt it for other optimizers. Although \cref{eq:opt_problem_last_layers_approx} is derived assuming vanilla SGD, and applying it elsewhere introduces an inconsistency, we show in practice that the resulting adaptations (\cref{eq:opt_problem_sgd} for SGD with momentum and \cref{eq:opt_problem_adamw} for AdamW, detailed in \cref{sec:SGD and AdamW adaptations}) remain faithful proxies for our ultimate objective of reducing total per-sample interference (\cref{eq:conceptual_problem}), even when restricting the approximation to the last  layer ($\ell=L$). We quantify the entire pipeline's effectiveness via the \textit{relative effective objective reduction} (REOR): $\left(J(\Delta\boldsymbol{\theta}_b) - J(\Delta\boldsymbol{\theta}^\star)\right)/J(\Delta\boldsymbol{\theta}_b)$, where $\Delta\boldsymbol{\theta}^\star {=} \Delta\boldsymbol{\theta}_b {+} \mathbf{G}^\top\mathbf{w}^\star$ and $\mathbf{w}^\star$ is the solution to the respective problem. Despite the underlying simplifications, the bottom row of \cref{fig:last_layers_approx_quality} shows that the majority of updates result in positive reductions, confirming that the optimized $\Delta\boldsymbol{\theta}^\star$ is systematically less interfering than the standard update $\Delta\boldsymbol{\theta}_b$. Additional figures and details are provided in \cref{appendix:additional_figures_approx}.

\subsection{Adaptation to Modern Optimizers}\label{sec:SGD and AdamW adaptations}

In this section, we extend \cref{eq:opt_problem_last_layers_approx} to modern optimizers: SGD with momentum and AdamW. A critical requirement for these adaptations is to completely avoid the explicit materialization of $\mathbf{G}$  at any stage, including when computing the final update $\Delta\boldsymbol{\theta}^\star = \Delta\boldsymbol{\theta}_b + \mathbf{G}^\top \mathbf{w}^\star$. We structure our surrogate problems so that their optimal solution $\bm{w}^\star$ allows $\Delta\boldsymbol{\theta}^\star$ to be implicitly realized. Specifically, \textit{we aim to compute the update via a standard backward pass on a weighted loss}, $\sum_{i=1}^{B} [\bm{w}^\star]_i \mathcal{L}_i$, naturally leveraging the underlying optimizer. Finally, as we advocate restricting the approximation to the very last layer, we henceforth drop the $\ell{:}L$ subscript and simply denote it by $L$.

\paragraph{SGD with Momentum:}
We consider the standard in deep learning SGD formulation \cite{sutskever_sgd}:
\begin{align*}
\mathbf{g}_t &\leftarrow \mathbf{g}_t + \mathrm{wd} \cdot \boldsymbol{\theta}_t ,\\
\mathbf{v}_t &\leftarrow  \mathrm{mom} \cdot \mathbf{v}_{t-1} + \mathbf{g}_t ,\\
\boldsymbol{\theta}_{t+1} &\leftarrow \boldsymbol{\theta}_t - \mathrm{lr} \cdot\mathbf{v}_t ,
\end{align*}
where $\mathrm{wd}$ is weight decay, $\mathrm{mom}$ is momentum factor, $\mathrm{lr}$ is the learning rate, and $\mathbf{v}_t$ is the momentum vector. Hence, the standard update $\Delta\boldsymbol{\theta}_b = \boldsymbol{\theta}_{t+1} - \boldsymbol{\theta}_t$ is equal to:
\begin{equation}\label{eq:sgd_dg}
    \Delta\boldsymbol{\theta}_b = - \mathrm{lr} \cdot \mathbf{v}_t
       = - \mathrm{lr}\cdot (\mathrm{mom} \cdot \mathbf{v}_{t-1}
         + \mathrm{wd} \cdot \boldsymbol{\theta}_t
         + \mathbf{g}_t) .
\end{equation}
Since $\mathbf{g}_t = \mathbf{G}^\top \mathbf{1}$, the target term $\mathbf{G}^\top \mathbf{w}^\star$ can be \textbf{automatically} incorporated by adjusting the loss to:
\begin{equation}\label{eq:weighted_CE}
    \mathcal{L}_{\bm{w}^\star}
       = \frac{1}{B}\sum_{i=1}^{B} \big(1 - \frac{[\mathbf{w}^\star]_i}{\mathrm{lr}}\big) \mathcal{L}_i .
\end{equation}
This is because the aggregated gradients of $\mathcal{L}_{\bm{w}^\star}$ are $\tilde{\mathbf{g}}_t = \mathbf{g}_t - \frac{1}{\mathrm{lr}}\mathbf{G}^\top \mathbf{w}^\star$ and substituting $\mathbf{g}_t$ with $\tilde{\mathbf{g}}_t$ in \cref{eq:sgd_dg} gives the desired updates $\Delta\boldsymbol{\theta}^\star = \Delta\boldsymbol{\theta}_b + \mathbf{G}^\top \mathbf{w}^\star$.

Finally, to make the formulation clearer, we perform the change of variables $\bm{w} \leftarrow -\frac{1}{\mathrm{lr}}\mathbf{w}$ so that the resulting optimization problem directly yields the \textit{additional relative weight} assigned to each sample. Substituting this into \cref{eq:opt_problem_initial_dual}, we obtain the optimization 
\begin{equation}\label{eq:opt_problem_sgd}
\boxed{
\begin{aligned}
\;&\arg\min_{\bm{w} \in \mathbb{R}^B}\; H\big(- \mathbf{Q}_L \bm{w}+ \frac{1}{\mathrm{lr}}\mathbf{G}_L \Delta\boldsymbol{\theta}_{b,L}\big) \\
&\text{s.t.}\; \bm{w}^\top \mathbf{Q}_L \bm{w}  \le \Big(\frac{c}{\mathrm{lr}}\Big)^2 \|\Delta\boldsymbol{\theta}_{b,L}\|^2, \; \mathbf{0} \leq \bm{w} \leq \mathbf{2}\quad
\end{aligned}
}
\end{equation}
Because the problem relies on the last-layer approximation, we add the constraint $\bm{w} \leq \mathbf{2}$ (where $\mathbf{2}=2\mathbf{1}$) as a safeguard to ensure that no sample is unreasonably weighted. After solving \cref{eq:opt_problem_sgd}, the interference-reducing parameter update $\Delta\boldsymbol{\theta}^\star$ can be implicitly computed using the standard SGD optimizer on  the weighted loss:
\begin{equation}\label{eq:weighted_CE_final}
    \mathcal{L}_{\bm{w}^\star} = \frac{1}{B}\sum_{i=1}^{B} (1 + [\bm{w}^\star]_i) \mathcal{L}_i
\end{equation}

\paragraph{AdamW:}
The AdamW \cite{AdamW} optimizer updates are:
\begin{align*}
\mathbf{m}_t &\leftarrow \beta_1 \mathbf{m}_{t-1} + (1-\beta_1) \mathbf{g}_t ,\\
\bm{v}_t &\leftarrow \beta_2 \bm{v}_{t-1} + (1-\beta_2) \mathbf{g}_t^2 ,\\
\boldsymbol{\theta}_{t+1} &\leftarrow \boldsymbol{\theta}_t - \mathrm{lr} \cdot \left( \frac{\mathbf{m}_t/(1-\beta_1^{t})}{\sqrt{\bm{v}_t/(1-\beta_2^{t})} + \varepsilon} + \mathrm{wd} \cdot \boldsymbol{\theta}_t \right),
\end{align*}
where $\varepsilon$ is a small constant, and all divisions, squaring, and square roots are performed element-wise. The standard update $\Delta\boldsymbol{\theta}_b=\boldsymbol{\theta}_{t+1}-\boldsymbol{\theta}_t$ is given by:
\begin{equation}\label{eq:standard adamW dg}
    \Delta\boldsymbol{\theta}_b = - \mathrm{lr} \cdot \Big(\mathrm{wd} \cdot \boldsymbol{\theta}_t
         + \frac{\beta_1}{(1-\beta_1)}\boldsymbol{\kappa} \odot  \mathbf{m}_{t-1}
         + \cdot \boldsymbol{\kappa} \odot \mathbf{g}_t\Big)
\end{equation}
with $\boldsymbol{\kappa} = \frac{(1-\beta_1)/(1-\beta_1^{t})}{\sqrt{\bm{v}_t/(1-\beta_2^{t})} + \varepsilon}$ and $\odot$ the Hadamard product. 

Optimizing here the weighted loss of  \cref{eq:weighted_CE_final}, creates gradients $\tilde{\mathbf{g}}_t = \mathbf{g}_t + \mathbf{G}^\top \bm{w}^\star$ which produce the modified second moment $\tilde{\bm{v}}_t = \beta_2 \bm{v}_{t-1} + (1-\beta_2)\tilde{\mathbf{g}}_t^2 $. Since $\beta_2$ is typically very close to 1 (e.g., $0.999$),
%evolves slowly and is relatively insensitive to the specific gradient of the current batch
we can safely assume $\tilde{\bm{v}}_t \approx \bm{v}_t$ , and consequently the modified scaling factor $\tilde{\boldsymbol{\kappa}} \approx \boldsymbol{\kappa}$. Therefore, the weighted loss results in the updates:
\begin{align} 
\Delta\boldsymbol{\theta}  &= - \mathrm{lr} \cdot \mathrm{wd} \cdot \boldsymbol{\theta}_t
         - \mathrm{lr} \cdot \frac{\beta_1}{(1-\beta_1)}\tilde{\boldsymbol{\kappa}} \odot  \mathbf{m}_{t-1}
         - \mathrm{lr} \cdot \tilde{\boldsymbol{\kappa}} \odot \tilde{\mathbf{g}}_t \nonumber \\  
         &\approx \Delta\boldsymbol{\theta}_b - \mathrm{lr}  \cdot  \mathbf{D}_{\boldsymbol{\kappa}} \mathbf{G}^\top \bm{w}^\star ,  \label{eq:approx_dtheta_adamw}  
\end{align}
where $\mathbf{D}_{\boldsymbol{\kappa}} = \operatorname{diag}(\boldsymbol{\kappa})$. Unlike SGD, the element-wise division by $\sqrt{\hat{\bm{v}}_t}$ introduces this diagonal scaling $\mathbf{D}_{\boldsymbol{\kappa}}$, meaning the optimal form described in Proposition~\ref{prop:optimal_form} no longer strictly applies. To proceed, we revisit our initial optimization problem \eqref{eq:opt_problem_initial} and directly substitute $\Delta\boldsymbol{\theta}$ using \cref{eq:approx_dtheta_adamw}. Given the structural similarities with SGD case, we choose to retain the non-negativity constraint $\bm{w} \geq \mathbf{0}$, alongside the bounding constraint $\bm{w} \leq \mathbf{2}$ ensuring reasonable weighting. Incorporating also the last-layer restriction the AdamW surrogate problem is:
\begin{equation}\label{eq:opt_problem_adamw}
\boxed{
\begin{aligned}
\;&\arg\min_{\bm{w} \in \mathbb{R}^B}\;  H\big(- \mathbf{Q}_{1,L} \bm{w}+ \frac{1}{\mathrm{lr}}\mathbf{G}_L \Delta\boldsymbol{\theta}_{b,L}\big) \\
&\text{s.t.}\;  \bm{w}^\top \mathbf{Q}_{2,L} \bm{w}  \le \Big(\frac{c}{\mathrm{lr}}\Big)^2 \|\Delta\boldsymbol{\theta}_{b,L}\|^2, \; \mathbf{0} \leq \bm{w} \leq \mathbf{2}\;
\end{aligned}
}
\end{equation}
where  $\mathbf{Q}_{1,L} = \mathbf{G}_L \mathbf{D}_{\boldsymbol{\kappa},L} \mathbf{G}^\top_L$ and $\mathbf{Q}_{2,L} = \mathbf{G}_L \mathbf{D}_{\boldsymbol{\kappa},L}^2 \mathbf{G}^\top_L$ are the weighted partial Gram matrices, $\mathbf{D}_{\boldsymbol{\kappa},L} = \operatorname{diag}(\boldsymbol{\kappa_L})$  and $\boldsymbol{\kappa_L}$ refers to the partial of $\boldsymbol{\kappa}$ corresponding to the $L$-th layer.

\subsection{The Final Nail for Per-Sample Gradients}\label{sec:final nail}
For linear layers, we can compute the defining quantities of our optimization problems (\cref{eq:opt_problem_sgd,eq:opt_problem_adamw}) without explicitly materializing the per-sample gradients. Combined with our last-layer approximation strategy ($\ell=L$), which in almost all state-of-the-art architectures is a linear projection, this property means we can \textbf{completely bypass the expensive computation and storage of $\mathbf{G}$, or any of its partials}.

Assume the $L$-th layer is a linear layer with parameters $\boldsymbol{\theta}_L \in \mathbb{R}^{d_{out} \times d_{in}}$, receiving input $\mathbf{x}_L \in \mathbb{R}^{B \times d_{in}}$ and  outputs $\mathbf{y}_L = \mathbf{x}_L\boldsymbol{\theta}_L^\top$. Let $\mathbf{g}_{\mathbf{y}_L} = \nabla_{\mathbf{y}_L} \mathcal{L} \in \mathbb{R}^{B \times d_{out}}$ denote the gradient of the loss with respect to these outputs. The per-sample gradient is $\mathbf{G}_L = \mathbf{g_y} \otimes \mathbf{x} \in \mathbb{R}^{B \times (d_{out} d_{in})}$, where $\otimes$ here denotes the Khatri-Rao (i.e., row-wise Kronecker) product. Using the properties of this product, the partial Gram matrix $\mathbf{Q}_L = \mathbf{G}_L \mathbf{G}_L^\top $ can be factorized as:
\begin{equation*}
    \mathbf{Q}_L = (\mathbf{x}_L \mathbf{x}_L^\top) \odot (\mathbf{g}_{\mathbf{y}_L} \mathbf{g}_{\mathbf{y}_L}^\top).
\end{equation*}
Computing $\mathbf{Q}_L $ through this factorization elegantly eliminates the need to materialize $\mathbf{G}_L$, circumventing the prohibitive $\mathcal{O}(B \cdot d_{in} \cdot d_{out})$ memory cost. Furthermore, it reduces the computational time complexity of forming $\mathbf{Q}_L$ from $\mathcal{O}(B^2 \cdot d_{in} \cdot d_{out})$ to just $\mathcal{O}(B^2(d_{in} + d_{out}))$.

For AdamW, the inclusion of the diagonal scaling $\mathbf{D}_{\boldsymbol{\kappa},L}$ makes the exact factorization of $\mathbf{Q}_{1,L} = \mathbf{G}_L \mathbf{D}_{\boldsymbol{\kappa},L} \mathbf{G}_L^\top$ impossible in a single step\footnote{The same approximation applies to $\mathbf{Q}_{2,L}$ by replacing $\boldsymbol{\kappa}_L$ with $\boldsymbol{\kappa}_L^2$.}. To resolve this, we first factorize $\boldsymbol{\kappa}_L \in \mathbb{R}^{d_{out} \times d_{in}}$ using a low-rank Singular Value Decomposition (SVD). Let $\boldsymbol{\kappa}_L \approx \mathbf{U}^r \mathbf{S}^r {\mathbf{V}^r}^\top$, where $\mathbf{U}^r \in \mathbb{R}^{d_{out} \times r}$, $\mathbf{V}^r \in \mathbb{R}^{d_{in} \times r}$, and $\mathbf{S}^r = \operatorname{diag}(\sigma_1, \dots, \sigma_r)$ contains the $r$ largest singular values ($r \leq \min(d_{in}, d_{out})$). Denoting the $k$-th columns of $\mathbf{U}^r$ and $\mathbf{V}^r$ by $\mathbf{u}_k$ and $\mathbf{v}_k$ respectively, we approximate $\mathbf{Q}_{1,L}$ as:
\begin{equation*}
    \mathbf{Q}_{1,L}^r = \sum_{k=1}^r \sigma_k \big( \mathbf{x}_L \operatorname{diag}(\mathbf{v}_k) \mathbf{x}_L^\top \big) \odot \big( \mathbf{g}_{\mathbf{y}_L} \operatorname{diag}(\mathbf{u}_k) \mathbf{g}_{\mathbf{y}_L}^\top \big).
\end{equation*}
This approximation becomes exact (i.e. $\mathbf{Q}_{1,L} =\mathbf{Q}_{1,L}^r $) when $r$ equals the rank of $\boldsymbol{\kappa}_L$. In practice, however, using even $r\leq 5$ is sufficient to bound the relative Frobenius error $\|\mathbf{Q}_{1,L} - \mathbf{Q}_{1,L}^r\|_F / \|\mathbf{Q}_{1,L}\|_F$ to less than $2\%$. 

In \cref{appendix:Khatri-rao}, we demonstrate that as batch sizes or layer dimensions scale, this Khatri-Rao methodology yields \emph{order-of-magnitude savings in both memory and runtime for both optimizers}. The appendix also provides further intuition for the SVD approximation, the corresponding PyTorch code, detailed profiling metrics, and an explanation of how the remaining quantity $\mathbf{G}_L \Delta\boldsymbol{\theta}_{b,L}$ is efficiently computed without materializing $\mathbf{G}_L$.

\subsection{Efficient Solver for the Optimization Problem}

The optimization problem we need to solve has the form:
\begin{equation} \label{eq:opt_problem_for_CP}
\begin{aligned}
\bm{w}^\star 
&= \arg\min_{\bm{w} \in \mathcal{C}}
    \;  H\big(- \mathbf{Q}_1 \bm{w} + \mathbf{q} \big)
\end{aligned}
\end{equation}
where $\mathcal{C} {=} \left\{ \bm{w}\in\mathbb{R}^B \middle| \bm{w}^\top \mathbf{Q}_2 \bm{w} \le r, \mathbf{0}\le \bm{w}\le \mathbf{2} \right\}$ is the domain, $\mathbf{Q}_1,\mathbf{Q}_2 {\in} \mathbb{R}^{B\times B}$ are symmetric positive semi-definite matrices (with $\mathbf{Q}_1{=}\mathbf{Q}_2$ for SGD case), $\mathbf{q}\in\mathbb{R}^B$, and $r{>}0$.

Despite being convex, solving \cref{eq:opt_problem_for_CP} efficiently is non-trivial due to: (i) the non-smoothness of $H$; (ii) the complex feasible domain arising from the intersection of an ellipsoidal constraint with bounding box constraints; and (iii) the necessity of a fully GPU-resident routine.

To address these challenges, we choose the Chambolle--Pock (CP) primal--dual algorithm \cite{chambolle_pock}, as it is particularly well-suited for non-smooth optimization under simple convex constraints.\footnote{In practice, we perform a few warm-up steps of Projected Gradient Descent to quickly provide a good starting point for CP.} 
In brief, the CP algorithm selects step sizes $\tau, \sigma > 0$ satisfying $\tau\sigma \|\mathbf{Q}_1\|_\sigma^2 < 1$, where $\|\mathbf{Q}_1\|_\sigma$ denotes the spectral norm of $\mathbf{Q}_1$, and an extrapolation parameter $\theta {\in} [0,1]$. Each iteration takes the form:
\begin{align*}
\mathbf{y}^{k+1} &= \Pi_{[0,1]^B}\!\left( \mathbf{y}^{k} + \sigma \left(-\mathbf{Q}_1 \widetilde{\bm{w}}^{k} + \mathbf{q} \right) \right)\\
\bm{w}^{k+1} &= \Pi_{\mathcal{C}}\!\left( \bm{w}^{k} + \tau \mathbf{Q}_1 \mathbf{y}^{k+1} \right) \\
\widetilde{\bm{w}}^{k+1} &= \bm{w}^{k+1} + \theta \left(\bm{w}^{k+1} - \bm{w}^k\right)   %\label{eq:cp_extrapolation_step} \label{eq:cp_primal_step}  \label{eq:cp_dual_step}
\end{align*}
Here, $\Pi_{[0,1]^B}$ denotes the Euclidean projection onto the unit hypercube $[0,1]^B$, which can be trivially computed by an element-wise clipping operation (time complexity $\mathcal{O}(B)$). In contrast, computing the projection $\Pi_{\mathcal{C}}$ onto the domain $\mathcal{C}$ is more involved; nevertheless, it can be adequately approximated in $\mathcal{O}(B^2)$ time as detailed in \cref{appendix:chambolle_pock_details}. The resulting overall time complexity is $\mathcal{O}(I_{solver} B^2)$, where $I_{solver}$ represents the total number of iterations (we restrict $I_{solver} \le 50$). Crucially, the solver relies only on \textbf{element-wise operations} and \textbf{matrix-vector multiplications}, making it highly efficient on GPUs. Full mathematical and implementation details are provided in \cref{appendix:chambolle_pock_details}.

\paragraph{Experiment: Evaluating solver efficiency.}
We evaluate the performance of our solver in terms of convergence speed and solution quality. We quantify quality by computing the \textit{relative solver objective reduction}, defined as $(H(\mathbf{q}) {-} H(-\mathbf{Q}_X\bm{w}^\star {+} \mathbf{q}))/H(\mathbf{q})$, where $\mathbf{Q}_X$ corresponds to the corresponding matrix $\mathbf{Q}_L$ for SGD (\cref{eq:opt_problem_sgd}) or $\mathbf{Q}_{1,L}$ for AdamW (\cref{eq:opt_problem_adamw}). We remark that $H(\mathbf{q})$ and $H(-\mathbf{Q}_X\bm{w}^\star {+} \mathbf{q})$ represent the interference generated by the standard update and our method, respectively.  Hence, the numerator represents the interference avoided by using our solver's output $\bm{w}^\star$. As illustrated in \cref{fig:solver_convergence_main}, the solver consistently converges within a small number of iterations while securing substantial objective reductions. Additional plots are provided in \cref{appendix:solver_efficiency_plots}.

\begin{algorithm}[t!]
\caption{Proposed Training Iteration for SGD}
\label{alg:training_loop}
\begin{algorithmic}[1]
\STATE {\bfseries Input:} Batch $\mathbf{x}$
\STATE $\mathbf{x}_L \leftarrow F_{1:L-1}(\mathbf{x}; \boldsymbol{\theta}_{1:L-1})$\hspace{35pt} \COMMENT{Backbone}
\STATE $\mathbf{y}_L \leftarrow F_L(\mathbf{x}_L; \boldsymbol{\theta}_L)$ \hspace{59pt} \COMMENT{Last layer}
\STATE $\mathcal{L} \leftarrow \mathcal{L}(\mathbf{y}_L)$ 
\STATE \textcolor{mutedred}{Compute $\mathbf{Q}_L, \mathbf{G}_L \Delta\boldsymbol{\theta}_{b,L}, \|\Delta\boldsymbol{\theta}_{b,L}\|$ }\COMMENT{Sections \ref{sec:final nail}, \ref{appendix:Khatri-rao},\ref{appendix:stability_preconditioning}}
%\STATE \textcolor{mutedred}{Solve \cref{eq:opt_problem_sgd} for $\bm{w}^\star$} \COMMENT{\cref{alg:cp_full} in \cref{appendix:chambolle_pock_details}}
\STATE \textcolor{mutedred}{Solve Equation (10) for $\bm{w}^\star$} \hspace{25pt} \COMMENT{\cref{alg:cp_full}}
\STATE  $\mathcal{L}_{\bm{w}^\star} \leftarrow  \sum_{i=1}^{B} (1 + [\bm{w}^\star]_i) \mathcal{L}_i$
\STATE Compute gradients of $\mathcal{L}_{\bm{w}^\star}$ and execute an SGD step
\end{algorithmic}
\end{algorithm}

\begin{figure}[t!]
    \centering
    \begin{subfigure}[b]{0.49\columnwidth}
        \centering
        \includegraphics[width=\textwidth]{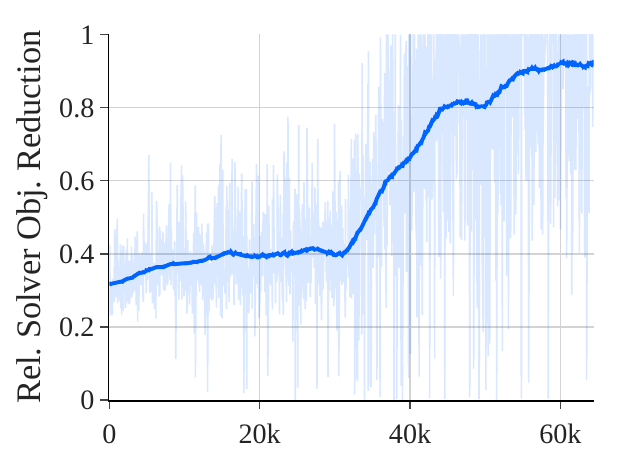}
        \vspace{2pt}
        \includegraphics[width=\textwidth]{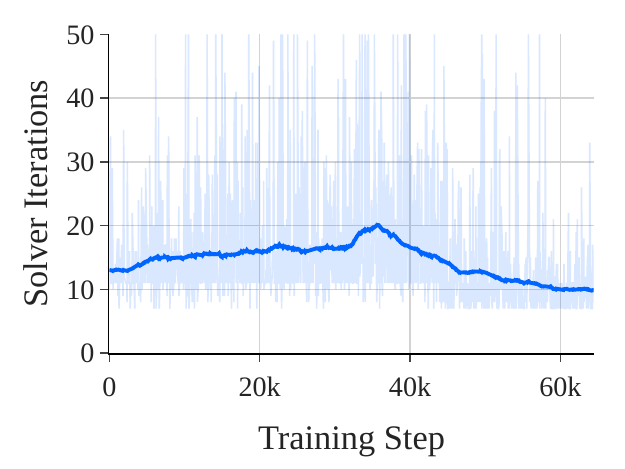}
        \caption{ResNet-44, CIFAR-10}
    \end{subfigure}
    \begin{subfigure}[b]{0.49\columnwidth}
        \centering
        \includegraphics[width=\textwidth]{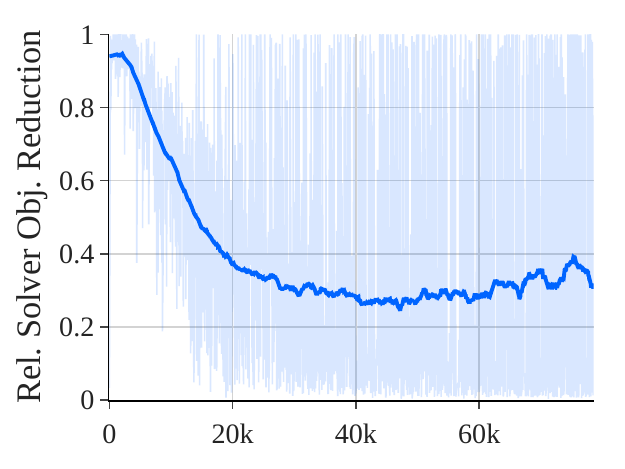}
        \includegraphics[width=\textwidth]{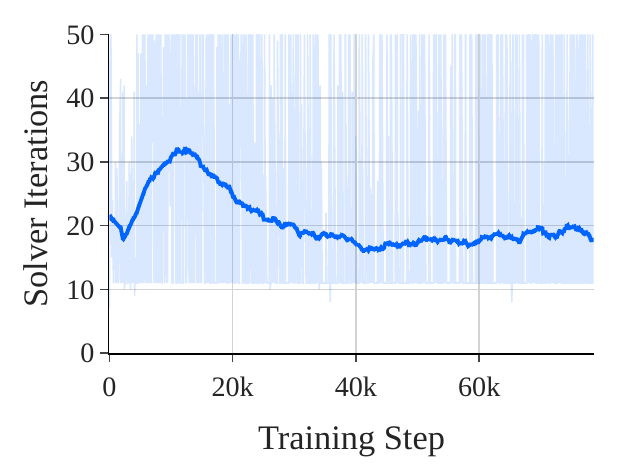}
        \caption{ViT-c, CIFAR-100}
    \end{subfigure}
    \caption{Solver convergence behavior. \textbf{Top Row:} Relative objective reduction achieved by the solver (higher is better). \textbf{Bottom Row:} Number of iterations required to reach convergence.}
    \label{fig:solver_convergence_main}
\end{figure}

\subsection{Full Training Algorithm and Computational Complexity}\label{sec:Full Algorithm}

\cref{alg:training_loop} details the proposed training iteration for SGD, executed at every step. The network is decomposed as $F(\mathbf{x};\boldsymbol{\theta}) = F_L(F_{1:L-1}(\mathbf{x}; \boldsymbol{\theta}_{1:L-1}); \boldsymbol{\theta}_L)$, where $F_L$ represents a standard linear layer  (whose output $\mathbf{y}_L$ corresponds to the logits in classification tasks). The main modifications to a standard training loop are Steps 5 and 6 (highlighted in red), which formulate and solve our surrogate optimization problem. Computing the weighted loss in Step 7 is the final difference, though its computational cost is negligible.

A practical consequence of using last-layer approximations is that these estimated statistics may occasionally diverge from their true global values, posing optimization risks. To prevent instabilities, we apply two \textbf{lightweight safeguards} (with time complexity $\mathcal{O}(B^2)$) before invoking the solver. Specifically, we maintain exponential moving averages (EMA) for $\|\Delta\boldsymbol{\theta}_{b,L}\|^2$ and the condition number of $\mathbf{Q}_L$. If the current step's $\|\Delta\boldsymbol{\theta}_{b,L}\|^2$ spikes abnormally compared to its EMA, we clip it. On the other hand, if $\mathbf{Q}_L$ becomes ill-conditioned relative to its EMA condition number, we apply adaptive identity preconditioning. In \cref{appendix:stability_preconditioning} we provide the exact pseudocode and further details.

Extending \cref{alg:training_loop} to AdamW is straightforward; the difference is that step 5 involves solving \cref{eq:opt_problem_adamw}, which requires modifying the preceding step to compute the matrices $\mathbf{Q}_{1,L}$ and $\mathbf{Q}_{2,L}$ instead of $\mathbf{Q}_L$.

In \cref{sec:more layers full alg}, we present a generalized formulation of our method that expands the last-layer approximation to include more layers (i.e. $\ell < L$). Because not all operations between  layer $\ell$ and $L$ are linear, the implicit computation techniques of \cref{sec:final nail} cannot be applied. Consequently, for this general form we explicitly compute the per-sample gradients (via PyTorch's \texttt{vmap}). In this section, we also profile the substantial memory and runtime overheads incurred  as per-sample gradients are explicitly computed for more layers. Finally, we show that including more layers may introduce numerical imprecisions that degrade overall performance.

\paragraph{Computational Overhead Complexity:}
The additional computations of our method compared to a standard training loop stem entirely from steps 5 and 6 in \cref{alg:training_loop}. Assuming the last linear layer has dimensions $(d_{in},d_{out})$, computing the quantities for the optimization problem in step 5 requires $\mathcal{O}\big(B^2(d_{in} + d_{out})\big)$ operations for SGD (\cref{appendix:Khatri-rao}). For AdamW, due to the additional low-rank SVD requirement, this increases to $\mathcal{O}\big(r d_{in} d_{out} + r B^2(d_{in} + d_{out})\big)$. Furthermore, applying our lightweight safeguards for numerical stability (\cref{appendix:stability_preconditioning}) requires $\mathcal{O}(I_L B^2)$ time, where $I_L$ represents the number of (Lanczos) iterations used to estimate the condition number of $\mathbf{Q}_L$ (or $\mathbf{Q}_{1,L}$). Moving to step 6, the solver operates in $\mathcal{O}(I_{solver} B^2)$ time, iterating at most $I_{solver}$ times (\cref{appendix:chambolle_pock_details}).  

Because we upper-bound the maximum allowed rank $r$ and the maximum iterations $I_L$ and $I_{solver}$, these terms behave as constants. Therefore, the total asymptotic complexity of the overhead reduces to $\mathcal{O}\big(B^2(d_{in} + d_{out})\big)$ for SGD and $\mathcal{O}\big(d_{in} d_{out} + B^2(d_{in} + d_{out})\big)$ for AdamW. Notice that this extra computation \textit{ depends exclusively on the batch size and the dimensions of the last  layer}. Consequently, as \textbf{the input dimensions or the backbone architecture scales up}—a common reality for modern, huge models where the backbone massively dominates total compute—\textbf{this overhead remains fixed}.  Finally, we observed completely negligible memory overheads in our benchmarks. The only exception was a minor $\sim 2\%$ increase on the main GPU when training ConvNeXt-Tiny on ImageNet under a 4-GPU Distributed Data Parallel (DDP) setup, resulting from the need to gather and compute the aggregate optimization quantities (e.g., $\mathbf{Q}$ and $\mathbf{G} \Delta\boldsymbol{\theta}_b$) over the collective global batch size.

\section{Experiments}
\label{sec:experiments}
We evaluate our method on the downstream task of image classification to demonstrate improved test accuracy with minimal memory and runtime overhead.

\subsection{Architectures, Datasets, and Training Details}

We evaluate performance on CIFAR-10, CIFAR-100, Tiny-ImageNet, and ImageNet-100. We consider both convolutional and transformer architectures. For ResNets on CIFAR, we follow \cite{ResNet} but substitute Batch Normalization \cite{BatchNorm} with Group Normalization \cite{GroupNorm}, as PyTorch's \texttt{vmap}---required for computing explicit per-sample gradients $\mathbf{G}$ in certain experiments (e.g., \cref{fig:last_layers_approx_quality})---lacks BatchNorm support. To avoid memory overflow when formulating the full $\mathbf{G}$ on CIFAR, we employ ResNet-20 and a small Vision Transformer (denoted ViT-small). Otherwise, we use ResNet-44 and the ViT implementation from \cite{vit_cifar_github}. For larger datasets (Tiny-ImageNet, ImageNet-100), we employ ResNet-50 and ConvNeXt-Tiny \cite{convnext}. Note that ResNet is trained with SGD, while ViT and \mbox{ConvNeXt} employ AdamW. Further details are provided in \cref{appendix:architectural_details}.

Across all experiments, we fix the trust-region parameter $c=0.2$.  Together with the constraint $\mathbf{w}\leq \mathbf{2}$ which avoids excessive sample weighting, these form the principal algorithmic choices of our method.  While additional hyperparameters are required, those are auxiliary, serving to ensure numerical stability when estimating the defining quantities and to ensure solver efficiency. Fundamentally, because the underlying optimization problem is convex, the core performance of our algorithm is expected to remain consistent regardless of the exact choice of solver, as long as numerical edge cases are prevented and a near-optimal solution is reached efficiently.

\subsection{Image Classification Performance}

\Cref{tab:acc_vs_resources} reports test accuracy and relative time overhead across varying batch sizes. Our approach consistently improves upon standard training, with particularly pronounced gains on CIFAR-100 and as batch size increases.

Intuitively, \textit{larger batch sizes} include more samples, increasing the probability of conflicting gradients. This presents a \textit{greater opportunity for our method to mitigate conflicts}, potentially yielding higher gains. Indeed, as $B$ scales from 128 to 1024, the absolute accuracy improvement over the baseline steadily increases: for ResNet-44 on CIFAR-100 (+0.25\%, +0.57\%, +0.86\%, +0.88\%), for ViT on CIFAR-10 (+0.49\%, +0.34\%, +0.58\%, +0.70\%), and for ViT on CIFAR-100 (+0.58\%, +0.79\%, +1.38\%, +2.24\%). 

An exception occurs with ResNet-44 on CIFAR-10, where gains ultimately degrade (+0.22\%, +0.26\%, +0.15\%, \mbox{-0.22\%}). Main reason of the performance drop at $B=1024$ is that the last linear layer of ResNet-44 has only 640 parameters ($d_{in}{=}64, d_{out}{=}10$). Consequently, $\mathbf{G}_L \in \mathbb{R}^{1024 \times 640}$, rendering the Gram matrix $\mathbf{Q}_L {=} \mathbf{G}_L \mathbf{G}_L^\top$ rank-deficient and highly ill-conditioned. This induces instability; in fact, this is the only setting where we had to reduce the trust-region radius to $c{=}0.1$ to prevent severe performance degradation.\footnote{In contrast, the ViT has a linear head with 3840 parameters and as $3840> 1024$, similar singularity issues do not appear.}

\begin{table}[t]
\caption{Test accuracy (\%) and relative time overhead (expressed as a multiplication factor) for ResNet-44 and ViT on CIFAR-10/100 across varying batch sizes (BS).}
\label{tab:acc_vs_resources}
\centering
\begin{small}
\begin{sc}
\setlength{\tabcolsep}{3.5pt} % Reduce column spacing
\begin{tabular}{c|ccc|ccc}
\toprule
 & \multicolumn{3}{c|}{\textbf{CIFAR-10}} & \multicolumn{3}{c}{\textbf{CIFAR-100}} \\
\textbf{BS} & \textbf{Base} & \textbf{Ours} & \textbf{Time} & \textbf{Base} & \textbf{Ours} & \textbf{Time} \\
\midrule
\multicolumn{7}{c}{\textit{ResNet-44}} \\
\midrule
128  & 91.71 & \textbf{91.93} & 1.13$\times$ & 65.22 & \textbf{65.50} & 1.30$\times$ \\
256  & 91.22 & \textbf{91.48} & 1.07$\times$ & 64.84 & \textbf{65.41} & 1.11$\times$ \\
512  & 90.60 & \textbf{90.75} & 1.03$\times$ & 63.51 & \textbf{64.37} & 1.05$\times$ \\
1024 & \textbf{89.22} & 89.00 & 1.03$\times$ & 61.43 & \textbf{62.31} & 1.03$\times$ \\
\midrule
\multicolumn{7}{c}{\textit{ViT-c}} \\
\midrule
128  & 89.94 & \textbf{90.34} & 1.08$\times$ & 63.05 & \textbf{63.63} & 1.10$\times$ \\
256  & 89.62 & \textbf{89.96} & 1.02$\times$ & 62.27 & \textbf{63.06} & 1.08$\times$ \\
512  & 89.59 & \textbf{90.17} & 1.01$\times$ & 60.73 & \textbf{62.11} & 1.05$\times$ \\
1024 & 88.91 & \textbf{89.61} & 1.01$\times$ & 58.82 & \textbf{61.06} & 1.04$\times$ \\
\bottomrule
\end{tabular}
\end{sc}
\end{small}
\end{table}

Regarding computational efficiency, represented as a runtime multiplier in \cref{tab:acc_vs_resources}, we observe two key trends. First, transitioning from ResNet-44 (0.6M parameters) to ViT (6.3M parameters) significantly reduces the relative overhead. This aligns with our analysis in \cref{sec:Full Algorithm}: increasing the backbone size increases the time spent in its forward and backward passes, whereas our method's overhead remains fixed, thereby reducing its relative cost. Second, although standard training scales linearly with $B$ and our method scales quadratically, the relative overhead actually decreases for larger batches. We attribute this to GPU utilization: for the tested batch sizes ($B \le 1024$), the $\mathcal{O}(B^2)$ matrix-vector operations used by our method are still small for modern GPUs. At smaller $B$, these operations do not fully utilize the available parallelism, so fixed GPU overheads account for a larger fraction of their runtime. As $B$ increases, the same operations use the hardware more efficiently, reducing the relative overhead.

\begin{table}[t]
\caption{Test accuracy (\%) and relative time overhead  (expressed as a multiplication factor) on Tiny-ImageNet and ImageNet-100.}
\label{tab:acc_vs_resources_imagenet}
\centering
\begin{subtable}[t]{\columnwidth}
\caption{Tiny-ImageNet}
\centering
\begin{small}
\begin{sc}
\setlength{\tabcolsep}{5pt}
\begin{tabular}{lccc}
\toprule
\textbf{Config} & \textbf{Base} & \textbf{Ours} & \textbf{Time} \\
\midrule
\multicolumn{4}{c}{\textit{ResNet-50 ($B=256$)}} \\
\midrule
\text{\normalfont  GroupNorm$\quad$} & 47.94 & \textbf{48.69} & 1.07$\times$ \\
\text{\normalfont  BatchNorm} & \textbf{56.37} & 56.20 & 1.02$\times$ \\
\midrule
\multicolumn{4}{c}{\textit{ConvNeXt-T ($B=4096$)}} \\
\midrule
\text{\normalfont w/o  Mixup$\quad$} & 54.08 & \textbf{54.54} & 1.18$\times$ \\
\text{\normalfont with Mixup} & 59.69 & \textbf{59.78} & 1.17$\times$ \\
\bottomrule
\end{tabular}
\end{sc}
\end{small}
\end{subtable}

\vspace{1.4em}

\begin{subtable}[t]{\columnwidth}
\caption{ImageNet-100}
\label{tab:acc_vs_resources_imagenet100}
\centering
\begin{small}
\begin{sc}
\setlength{\tabcolsep}{5pt}
\begin{tabular}{lccc}
\toprule
\textbf{Config} & \textbf{Base} & \textbf{Ours} & \textbf{Time} \\
\midrule
\multicolumn{4}{c}{\textit{ConvNeXt-T ($B=4096$)}} \\
\midrule
\text{\normalfont without Mixup, 100 epochs$\quad$} & 76.50 & \textbf{76.66} & 1.04$\times$ \\
\text{\normalfont with Mixup, 100 epochs} & 78.60 & \textbf{79.76} & 1.04$\times$ \\
\text{\normalfont standard\footnotemark} & 83.60 & \textbf{83.90} & 1.05$\times$ \\
\text{\normalfont standard + more data aug.} & 83.90 & \textbf{84.26} & 1.05$\times$ \\
\bottomrule
\end{tabular}
\end{sc}
\end{small}
\end{subtable}
\end{table}
\footnotetext{Standard configuration as in the original paper, trained for 300 epochs.}

\Cref{tab:acc_vs_resources_imagenet} presents results for larger models---ResNet-50 (25M parameters) and ConvNeXt-T (29M parameters)---on larger datasets. Since BatchNorm and Mixup/Cutmix break sample independence within a batch, changing the statistics of the required quantities (e.g., $\mathbf{Q}$), we explicitly test their effect by evaluating configurations both with and without them. Excluding ResNet-50 with BatchNorm, our approach consistently improves test accuracy, achieving the  substantial gains on ConvNeXt-T trained with Mixup on ImageNet-100 for 100 epochs.

For ConvNeXt-T, transitioning  from dataset Tiny-ImageNet to ImageNet-100 reduces the relative penalty from around 18\% to less than 5\%. This agrees with our analysis in \cref{sec:Full Algorithm}: as the input size increases (here from images $64{\times}64$ to $224{\times} 224$), the required backbone processing time increases, but our method's overhead remains the same, leading to a decrease in the relative overhead.

\Cref{tab:acc_vs_resources_imagenet100} presents results for ConvNeXt-T on ImageNet-100 across varying levels of training duration and data augmentation. We evaluate configurations trained for 100 epochs (both with and without Mixup/Cutmix) alongside the standard 300-epoch schedule from \cite{convnext}. Additionally, we assess a configuration that further intensifies data augmentation for the 300-epoch standard by increasing the number of RandAugment \cite{randaugment} operations from 2 to 3. In general, we observe that our method accelerates convergence, enabling the model to learn faster. This acceleration proves particularly beneficial when the training budget is constrained or when the learning task is heavily regularized. For instance, reducing the training duration from 300 to 100 epochs (while keeping Mixup) increases the absolute accuracy gain over the baseline from +0.30\% to a substantial +1.16\%, highlighting the critical role of our method under restricted budget. Conversely, when we make the 100-epoch setting easier to learn by removing Mixup, the gain drops back to +0.16\%. Similarly, reintroducing training difficulty into the standard 300-epoch schedule by aggressively increasing the RandAugment operations again demands faster learning, pushing our absolute performance gains from +0.30\% to +0.36\%.

We used a Distributed Data Parallel (DDP) setup (4$\times$T4 GPUs for Tiny-ImageNet, 4$\times$A100 for ImageNet-100) to train ConvNeXt-T. Standard DDP already requires gradient aggregation across GPUs (PyTorch's \texttt{all\_reduce}) and state synchronization (\texttt{broadcast}). Implementing our method adds two synchronization steps: i) an \texttt{all\_gather} to collect the last-layer per-sample gradients from all GPUs and form the global required quantities (e.g., $\mathbf{Q}$), and ii) a \texttt{broadcast} to distribute the solver's computed optimal weights. These extra synchronizations could potentially be removed through a practical relaxation in which each GPU uses its local batch to independently form and solve its own optimization problem (\cref{eq:opt_problem_sgd,eq:opt_problem_adamw}). This effectively approximates the global $\mathbf{Q}_L$ as block-diagonal, assuming that sample gradients across devices are uncorrelated.

\section{Conclusion}
\label{sec:conclusion}

We proposed a practical framework for reducing the interference that mini-batch updates can induce on individual samples. Although the natural formulation is computationally prohibitive, we showed that it can be turned into an efficient surrogate by moving the optimization to batch space, restricting the required quantities to the last linear layer and implicitly computing them without materializing per-sample gradients. We finally introduce a GPU-friendly primal-dual solver to solve the resulting surrogate problem. This yields a lightweight modification of standard optimizers such as SGD with momentum and AdamW. Experiments on image classification benchmarks demonstrate that our method reduces per-sample interference and improves test accuracy with small runtime overhead. 

\bibliography{biblio}
\bibliographystyle{icml2026}

%%%%%%%%%%%%%%%%%%%%%%%%%%%%%%%%%%%%%%%%%%%%%%%%%%%%%%%%%%%%%%%%%%%%%%%%%%%%%%%
%%%%%%%%%%%%%%%%%%%%%%%%%%%%%%%%%%%%%%%%%%%%%%%%%%%%%%%%%%%%%%%%%%%%%%%%%%%%%%%
% APPENDIX
%%%%%%%%%%%%%%%%%%%%%%%%%%%%%%%%%%%%%%%%%%%%%%%%%%%%%%%%%%%%%%%%%%%%%%%%%%%%%%%
%%%%%%%%%%%%%%%%%%%%%%%%%%%%%%%%%%%%%%%%%%%%%%%%%%%%%%%%%%%%%%%%%%%%%%%%%%%%%%%
\newpage
\appendix
\onecolumn

\section{Proof of Proposition \ref{prop:optimal_form}}\label{appendix:proof-proposition-1}

Defining the displacement vector $\boldsymbol{\delta} := \Delta\boldsymbol{\theta} - \Delta\boldsymbol{\theta}_b\in \mathbb{R}^\Theta$, $\mathbf{s} := \mathbf{G}\Delta\boldsymbol{\theta}_b \in \mathbb{R}^B$, the radius $\rho := c \|\Delta\boldsymbol{\theta}_b\|$, and using the definition $H(\mathbf{v}) := \sum_{i=1}^{B} \max(0, [\mathbf{v}]_i)$, we can directly rewrite the formulation \eqref{eq:opt_problem_initial} as:
\begin{equation}
\label{eq:delta_prob}
\begin{aligned}
\min_{\boldsymbol{\delta}} \quad & J(\boldsymbol{\delta}) := H(\mathbf{G}\boldsymbol{\delta} + \mathbf{s}) \\
\text{s.t.} \quad & \|\boldsymbol{\delta}\| \le \rho.
\end{aligned}
\end{equation}

\paragraph{Part 1: Existence of a solution in Range($\mathbf{G}^\top$).}
Any vector $\boldsymbol{\delta} \in \mathbb{R}^\Theta$ can be uniquely decomposed into two orthogonal components:
\[
\boldsymbol{\delta} = \boldsymbol{\delta}_{\parallel} + \boldsymbol{\delta}_{\perp},
\]
where $\boldsymbol{\delta}_{\parallel} \in \mathrm{Range}(\mathbf{G}^\top)$ (i.e., $\boldsymbol{\delta}_{\parallel} = \mathbf{G}^\top \mathbf{v}$ for some $\mathbf{v}\in\mathbb{R}^B$) and $\boldsymbol{\delta}_{\perp} \in \mathrm{Null}(\mathbf{G})$ (i.e., $\mathbf{G}\boldsymbol{\delta}_{\perp}=\mathbf{0}$).
Since $\mathbf{G}\boldsymbol{\delta}=\mathbf{G}\boldsymbol{\delta}_{\parallel}$, we have $J(\boldsymbol{\delta})=J(\boldsymbol{\delta}_{\parallel})$.
Moreover, by orthogonality,
\[
\|\boldsymbol{\delta}\|^2=\|\boldsymbol{\delta}_{\parallel}\|^2+\|\boldsymbol{\delta}_{\perp}\|^2
\quad\Rightarrow\quad
\|\boldsymbol{\delta}_{\parallel}\| \le \|\boldsymbol{\delta}\|.
\]
Therefore, if $\boldsymbol{\delta}$ is feasible (i.e., $\|\boldsymbol{\delta}\|\le \rho$), then $\boldsymbol{\delta}_{\parallel}$ is also feasible and achieves the same objective value. Hence, there always exists an optimal solution lying in $\mathrm{Range}(\mathbf{G}^\top)$.

\paragraph{Part 2: Non-positivity of coefficients ($\mathbf{w}\le \mathbf{0}$).}
Let $\mathcal{S}^*$ denote the set of optimal solutions to \cref{eq:delta_prob}. We aim to prove that the minimum-norm optimal solution $\boldsymbol{\delta}^\dagger\in\mathcal{S}^*$ satisfies
$\boldsymbol{\delta}^\dagger=\mathbf{G}^\top \mathbf{w}$ with $\mathbf{w}\le \mathbf{0}$ (recall that $\Delta\boldsymbol{\theta}^\star = \Delta\boldsymbol{\theta}_b + \boldsymbol{\delta}^\dagger$). To this end, we add a small $L_2$ regularization term and consider the modified problem:
\[
\min_{\boldsymbol{\delta}} \; J(\boldsymbol{\delta}) +\epsilon\|\boldsymbol{\delta}\|^2
\quad \text{s.t.}\quad
\|\boldsymbol{\delta}\|^2 \le \rho^2,
\]
where $\epsilon>0$ is arbitrarily small. The objective is now strictly convex and therefore has a unique minimizer $\boldsymbol{\delta}_\epsilon$. The subdifferential of $J(\boldsymbol{\delta}) = H(\mathbf{G}\boldsymbol{\delta} + \mathbf{s})$ is given by:
\[
\partial J(\boldsymbol{\delta}) =
\left\{
\mathbf{G}^\top \boldsymbol{\gamma}
\;\middle|\;
[\boldsymbol{\gamma}]_i =
\begin{cases}
1, & \text{for }  [\mathbf{G}\boldsymbol{\delta}+\mathbf{s}]_i > 0,\\[0pt]
\in [0,1], &\text{for }  [\mathbf{G}\boldsymbol{\delta}+\mathbf{s}]_i = 0,\\[0pt]
0, & \text{for }  [\mathbf{G}\boldsymbol{\delta}+\mathbf{s}]_i < 0
\end{cases}
\right\}.
\]
We remark that every admissible $\boldsymbol{\gamma}$ satisfies $\boldsymbol{\gamma}\in[0,1]^B$. By the KKT stationarity condition,
$ \mathbf{0} \in \partial J(\boldsymbol{\delta}_\epsilon) + 2(\epsilon+\lambda)\boldsymbol{\delta}_\epsilon$,
where $\lambda\ge 0$ is the Lagrange multiplier associated with the constraint $\|\boldsymbol{\delta}\|^2 \le \rho^2$. Hence, there exists some $\boldsymbol{\gamma}\ge \mathbf{0}$ such that
\[
\mathbf{G}^\top \boldsymbol{\gamma} + 2(\epsilon+\lambda)\boldsymbol{\delta}_\epsilon = \mathbf{0}
\quad\Rightarrow\quad
\boldsymbol{\delta}_\epsilon = \mathbf{G}^\top\!\Big(-\frac{\boldsymbol{\gamma}}{2\epsilon+2\lambda}\Big).
\]
Defining $\mathbf{w}_\epsilon:=-\frac{\boldsymbol{\gamma}}{2(\epsilon+\lambda)}$, we obtain $\mathbf{w}_\epsilon\le \mathbf{0}$ since $\boldsymbol{\gamma}\ge \mathbf{0}$ and $\epsilon+\lambda>0$. Finally, as $\epsilon\to 0$, the sequence $\boldsymbol{\delta}_\epsilon$ converges to the minimum-norm optimal solution $\boldsymbol{\delta}^\dagger\in\mathcal{S}^*$ (since if two solutions are optimal, and so they have the same $J$, the one with the smaller norm will have a smaller $J(\boldsymbol{\delta})+\epsilon\|\boldsymbol{\delta}\|^2$). Therefore, $
\boldsymbol{\delta}^\dagger = \mathbf{G}^\top \mathbf{w}$ for some $\mathbf{w}\le \mathbf{0}$.

\paragraph{Part 3: Uniqueness of minimum distance.}
Consider any optimal solution $\Delta\boldsymbol{\theta}\in\mathcal{S}^*$ with displacement $\boldsymbol{\delta}=\Delta\boldsymbol{\theta}-\Delta\boldsymbol{\theta}_b$, and decompose $\boldsymbol{\delta}=\boldsymbol{\delta}_{\parallel}+\boldsymbol{\delta}_{\perp}$ as above. From Part~1 we know that  $\boldsymbol{\delta}'= \boldsymbol{\delta}_{\parallel}$ is also optimal. Since
\[
\|\boldsymbol{\delta}\|^2=\|\boldsymbol{\delta}_{\parallel}\|^2+\|\boldsymbol{\delta}_{\perp}\|^2 \geq \|\boldsymbol{\delta}_{\parallel}\|^2 = \|\boldsymbol{\delta}'\|^2
\]
any solution with $\boldsymbol{\delta}_{\perp}{\neq}\mathbf{0}$ has strictly larger distance to $\Delta\boldsymbol{\theta}_b$ than the corresponding solution $\boldsymbol{\delta}'$ with $\boldsymbol{\delta}_{\perp}{=}\mathbf{0}$.
Hence, among all optimal solutions, the unique minimizer of $\|\Delta\boldsymbol{\theta}-\Delta\boldsymbol{\theta}_b\|$ is the one with $\boldsymbol{\delta}\in\mathrm{Range}(\mathbf{G}^\top)$, i.e., of the form $\Delta\boldsymbol{\theta}_b+\mathbf{G}^\top \mathbf{w}$.

\paragraph{Connection to the Representer Theorem \cite{scholkopf2001generalized}.}
A closely related statement follows by considering a Lagrangian relaxation of the ball constraint, yielding a regularized objective of the form
\[
\min_{\boldsymbol{\delta}} \; H(\mathbf{G}\boldsymbol{\delta} + \mathbf{s}) + \lambda \|\boldsymbol{\delta}\|^2,
\]
for some $\lambda>0$. The Representer Theorem then directly guarantees that the optimal $\boldsymbol{\delta}$ lies in the span of the data representers $\{\mathbf{G}_i\}_{i=1}^{B}$. This is consistent with Part~1. Our argument additionally characterizes the sign of the coefficients in the representation, i.e., $\mathbf{w}\le \mathbf{0}$.

\section{Computing the Defining Quantities of the Optimization Problem Without Materializing Per-Sample Gradients}
\label{appendix:Khatri-rao}

In this section, we study the case of a single linear layer. Because our focus here is not limited to the last  layer, we drop the subscript $L$—used in the main text to denote the last linear layer—for clarity. Consider a linear layer with parameters $\boldsymbol{\theta}_w \in \mathbb{R}^{d_{out} \times d_{in}}$ and bias $\boldsymbol{\theta}_b \in \mathbb{R}^{d_{out}}$, receiving inputs $\mathbf{x} \in \mathbb{R}^{B \times d_{in}}$ and producing outputs $\mathbf{y} \in \mathbb{R}^{B \times d_{out}}$ such that $ [\mathbf{y}]_i = [\mathbf{x}\boldsymbol{\theta}_w^\top]_i + \boldsymbol{\theta}_b, \forall i\in\{1, 2, \dots, B \}$. With $[\mathbf{A}]_i$ we refer to the $i$-th element if $\mathbf{A}$ is a vector and the $i$-th row if it is a matrix. Let $\mathbf{g_y} = \nabla_{\mathbf{y}} \mathcal{L} \in \mathbb{R}^{B \times d_{out}}$ denote the gradient of the loss with respect to these outputs. The per-sample gradient with respect to the weight parameters is $\mathbf{G}_w = \mathbf{g_y} \otimes \mathbf{x} \in \mathbb{R}^{B \times (d_{out} d_{in})}$, where $\otimes$ denotes the Khatri-Rao (i.e., row-wise Kronecker) product, meaning $[\mathbf{G}_w]_i =  [\mathbf{g_y}]_i \otimes [\mathbf{x}]_i $ gives the flattened gradient for the $i$-th sample\footnote{Note the Kronecker product for one dimensional vectors coincides to their flattened outer product.}. The per-sample gradients with respect to the bias parameters are simply $\mathbf{G}_b = \mathbf{g_y} \in \mathbb{R}^{B \times d_{out}}$.

\subsection{Standard form $\mathbf{Q}=\mathbf{G}\mathbf{G}^\top$ (SGD)}
\label{appendix:Khatri-rao-sgd}
 Leveraging the properties of the Khatri-Rao product, the total Gram matrix $\mathbf{Q} \in \mathbb{R}^{B \times B}$ can be factorized as:
\begin{equation*}
    \mathbf{Q} = \mathbf{G}_w \mathbf{G}_w^\top + \mathbf{G}_b \mathbf{G}_b^\top = (\mathbf{x} \mathbf{x}^\top) \odot (\mathbf{g_y} \mathbf{g_y}^\top) + (\mathbf{g_y} \mathbf{g_y}^\top) = (\mathbf{x} \mathbf{x}^\top + \mathbf{I}) \odot (\mathbf{g_y} \mathbf{g_y}^\top)
\end{equation*}
where $\mathbf{I}$ is the $B \times B$ identity matrix. This factorization allows to compute $\mathbf{Q}$ as $(\mathbf{x} \mathbf{x}^\top + \mathbf{I}) \odot (\mathbf{g_y} \mathbf{g_y}^\top)$ instead of $\mathbf{G}_w \mathbf{G}_w^\top + \mathbf{G}_b \mathbf{G}_b^\top$. Therefore it eliminates the need to materialize the per-sample weight gradients $\mathbf{G}_w \in \mathbb{R}^{B \times (d_{out} d_{in})}$, circumventing a prohibitive $\mathcal{O}(B d_{in}  d_{out})$ memory cost. The time complexity is reduced from $\mathcal{O}(B^2 d_{in} d_{out})$ to $\mathcal{O}(B^2(d_{in} + d_{out}))$. In Pseudocode \ref{lst:sgd_pseudocode}, we provide a minimal example demonstrating how the Khatri-Rao trick can be used for the computation of the Gram matrix $\mathbf{Q}$ compared to the standard baseline.

\renewcommand{\lstlistingname}{Pseudocode}
\begin{lstlisting}[language=Python, caption={Khatri-Rao Trick vs. Baseline in PyTorch}, label={lst:sgd_pseudocode}]
# Dimensions: x [B, d_in], grad_y [B, d_out]

# --- Baseline ---
per_sample_grad_w = grad_y.unsqueeze(2) * x.unsqueeze(1) # [B, d_out, d_in]
grad_w_flat = per_sample_grad_w.flatten(start_dim=1)     # [B, d_out * d_in]
Q_bias = grad_y @ grad_y.t() 
Q_weight = grad_w_flat @ grad_w_flat.t()
Q_baseline = Q_weight + Q_bias 

# --- Khatri-Rao Trick ---
Q_bias = grad_y @ grad_y.t()
Q_optimized = (x @ x.t() + torch.eye(B)) * Q_bias
\end{lstlisting}

In \cref{tab:khatri_rao_bench_summary}, each row corresponds to one batch size. Entries are reported in the format:
\emph{time (ms), memory (MB)} for Baseline, and \emph{time (improvement), memory (improvement)} for Khatri-Rao, where improvements are percentages relative to Baseline. For example, the memory improvement is defined as $\displaystyle 100 \cdot \frac{Mem. Baseline - Mem. Khatri-Rao}{Mem. Baseline}$. We emphasize that the \textit{maximum improvement is }100\%. An NVIDIA Tesla T4 GPU with 16 GB memory was used. Note that Tables (a)-(b) correspond to layer dimensions used for \textit{ResNet-44-gn} architectures for CIFAR-10 and CIFAR-100 respectively. The tables show that the gains are an order of magnitude both in time and in memory, especially as the dimensions are increased.

\begin{table}[h!]
\centering
\begin{subtable}[t]{0.48\textwidth}
\centering
\caption{$d_{in}=64,\ d_{out}=10$}
\begin{tabular}{c|r@{,\ }l|r@{,\ }l}
\toprule
BS & \multicolumn{2}{c|}{Baseline (ms, MB)} & \multicolumn{2}{c}{Khatri-Rao (ms, MB)} \\
\midrule
128  & $\quad\;$ 0.14 & 8.75  & 0.107 (\textbf{25}\%) & 8.41 (\textbf{4}\%) \\
512  & 0.42 & 13.6  & 0.343 (\textbf{19}\%) & 12.3 (\textbf{10}\%) \\
2048 & 2.60  & 79.3  & 1.07 (\textbf{59}\%)  & 72.7 (\textbf{8}\%) \\
\bottomrule
\end{tabular}
\end{subtable}
\hfill
\begin{subtable}[t]{0.48\textwidth}
\centering
\caption{$d_{in}=64,\ d_{out}=100$}
\begin{tabular}{c|r@{,\ }l|r@{,\ }l}
\toprule
BS & \multicolumn{2}{c|}{Baseline (ms, MB)} & \multicolumn{2}{c}{Khatri-Rao (ms, MB)} \\
\midrule
128  & $\quad\;$ 0.22 & 11.9  & 0.10 (\textbf{54}\%) & 8.46 (\textbf{29}\%) \\
512  & 1.63  & 26.2  & 0.35 (\textbf{79}\%) & 12.5 (\textbf{53}\%) \\
2048 & 15.2  & 128   & 1.03 (\textbf{93}\%)  & 73.4 (\textbf{43}\%) \\
\bottomrule
\end{tabular}
\end{subtable}

\vspace{0.6em}

\begin{subtable}[t]{0.48\textwidth}
\centering
\caption{$d_{in}=2048,\ d_{out}=100$}
\begin{tabular}{c|r@{,\ }l|r@{,\ }l}
\toprule
BS & \multicolumn{2}{c|}{Baseline (ms, MB)} & \multicolumn{2}{c}{Khatri-Rao (ms, MB)} \\
\midrule
256  &  $\quad\;$ 9.08 & 231    & 0.185 (\textbf{98}\%) & 11.2 (\textbf{95}\%) \\
1024 & 155  & 913    & 1.95 (\textbf{99}\%)  & 32.5 (\textbf{96}\%) \\
4096 & 3381 & 3818   & 22.3 (\textbf{99}\%)  & 298 (\textbf{92}\%) \\
\bottomrule
\end{tabular}
\end{subtable}
\hfill
\begin{subtable}[t]{0.48\textwidth}
\centering
\caption{$d_{in}=2048,\ d_{out}=1000$}
\begin{tabular}{c|r@{,\ }l|r@{,\ }l}
\toprule
BS & \multicolumn{2}{c|}{Baseline (ms, MB)} & \multicolumn{2}{c}{Khatri-Rao (ms, MB)} \\
\midrule
256  &  $\quad\;$ 127 & 2212 & 0.326 (\textbf{100}\%) & 12.1 (\textbf{100}\%) \\
1024 & \multicolumn{2}{c|}{Out-of-Memory} & 2.44  & 36.1  \\
4096 & \multicolumn{2}{c|}{Out-of-Memory} & 29.9  & 312  \\
\bottomrule
\end{tabular}
\end{subtable}
\vspace{5pt}
\caption{Benchmark comparison of Baseline and Khatri-Rao across batch sizes and layer dimensions.}
\label{tab:khatri_rao_bench_summary}
\end{table}

\newpage
\subsection{Weighted form $\mathbf{Q}=\mathbf{G}\texttt{diag}(\boldsymbol{\tau})\mathbf{G}^\top$ (AdamW)}
\label{appendix:Khatri-rao-adamw}

Let $[\mathbf{M}]_{i,j}$ denote the element in the $i$-th row and $j$-th column of a matrix $\mathbf{M}$. We assume $\boldsymbol{\tau} \in \mathbb{R}^{d_{out} \times d_{in}}$ and with $\operatorname{diag}(\boldsymbol{\tau})$ we mean that first we flatten $\boldsymbol{\tau}$ in the same way that each row (corresponding to each sample of the batch) of $\mathbf{G}_w$ was flattened, and then we diagonalize the resulting vector. Previously we could directly use the Khatri-Rao trick as we could write:
\begin{equation*}
[\mathbf{G}_w \mathbf{G}_w^\top ]_{i,j} = \sum_{q,p} [\mathbf{x}]_{i,q} [\mathbf{x}]_{j,q} [\mathbf{g_y}]_{i,p} [\mathbf{g_y}]_{j,p} = \Big(\sum_q [\mathbf{x}]_{i,q} [\mathbf{x}]_{j,q}\Big) \Big(\sum_p [\mathbf{g_y}]_{i,p} [\mathbf{g_y}]_{j,p}\Big).
\end{equation*}
Unfortunately, the existence of the diagonal matrix complicates the formula as:
\begin{equation*}
[\mathbf{G}_w \operatorname{diag}(\boldsymbol{\tau}) \mathbf{G}_w^\top ]_{i,j} = \sum_{q,p} \boldsymbol{\tau}_{p,q} [\mathbf{x}]_{i,q} [\mathbf{x}]_{j,q} [\mathbf{g_y}]_{i,p} [\mathbf{g_y}]_{j,p},
\end{equation*}
and the $\boldsymbol{\tau}_{p,q}$ inhibits the same factorization.  However, if we manage to factorize $\boldsymbol{\tau}=\mathbf{A}\mathbf{V}^\top$, then because  $\boldsymbol{\tau}_{p,q}= \sum_n [\mathbf{A}]_{p,n} [\mathbf{V}]_{q,n}$, we can write:
\begin{equation*}
[\mathbf{G}_w \operatorname{diag}(\boldsymbol{\tau}) \mathbf{G}_w^\top ]_{i,j} = \sum_n \bigg( \Big(\sum_q [\mathbf{V}]_{q,n} [\mathbf{x}]_{i,q} [\mathbf{x}]_{j,q}\Big) \Big(\sum_p [\mathbf{A}]_{p,n}  [\mathbf{g_y}]_{i,p} [\mathbf{g_y}]_{j,p}\Big) \bigg).
\end{equation*}
Such a factorization could be achieved with SVD, setting $\boldsymbol{\tau}= \mathbf{U} \mathbf{S} \mathbf{V}^\top$ with $\mathbf{U} \in \mathbb{R}^{d_{out} \times d_{out}}, \mathbf{V} \in \mathbb{R}^{d_{in} \times d_{in}}$ unitary matrices and $\mathbf{S} = \operatorname{diag}(\sigma_1, \dots, \sigma_{\min(d_{in},d_{out})})\in \mathbb{R}^{d_{out} \times d_{in}}$ with $\sigma_1 \ge \sigma_2 \ge \dots \ge 0$ the singular values. We choose the SVD factorization because the matrix $\boldsymbol{\tau}^r$ that is closest to $\boldsymbol{\tau}$ in Frobenius norm, i.e., minimizing $\|\boldsymbol{\tau} - \boldsymbol{\tau}^r\|_F$, and can be factorized as $\boldsymbol{\tau}^r = \mathbf{A}^r {\mathbf{V}^r}^\top$ with $\mathbf{A}^r \in \mathbb{R}^{d_{out} \times r}$ and $\mathbf{V}^r \in \mathbb{R}^{d_{in} \times r}$, is given by $\mathbf{A} = \mathbf{U}^r \operatorname{diag}(\sigma_1, \dots, \sigma_r)$, where $\mathbf{U}^r \in \mathbb{R}^{d_{out} \times r}$ is the first $r$ columns of $\mathbf{U}$ and $\mathbf{V}^r$ is the first $r$ columns of $\mathbf{V}$. Therefore, given $r$, we know we have the best possible approximation and we can write:
\begin{align*}
[\mathbf{G}_w \operatorname{diag}(\boldsymbol{\tau}) \mathbf{G}_w^\top ]_{i,j} &\approx \sum_{k=1}^r \sigma_k \bigg( \Big(\sum_q [\mathbf{V}]_{q,k} [\mathbf{x}]_{i,q} [\mathbf{x}]_{j,q}\Big) \Big(\sum_p [\mathbf{U}]_{p,k} [\mathbf{g_y}]_{i,p} [\mathbf{g_y}]_{j,p}\Big) \bigg) \\
&= \sum_{k=1}^r \sigma_k \big( \mathbf{x} \operatorname{diag}(\mathbf{v}_k) \mathbf{x}^\top \big) \odot \big( \mathbf{g_y} \operatorname{diag}(\mathbf{u}_k) \mathbf{g_y}^\top \big)
\end{align*}
denoting by $\mathbf{u}_k, \mathbf{v}_k$ the $k$-th column of $\mathbf{U}^r$ and $\mathbf{V}^r$ respectively.

Similar to before, this factorization bypasses the prohibitive $\mathcal{O}(B d_{in}  d_{out})$ memory cost associated with materializing per-sample gradients. In terms of computation, while the baseline requires $\mathcal{O}(B^2  d_{in} d_{out})$ operations to compute the exact weighted Gram matrix, our proposed method reduces the overall time complexity to $\mathcal{O}\big(r d_{in} d_{out}  + r B^2 (d_{in} + d_{out})\big)$. The first term accounts for the cost of computing a low-rank SVD on $\boldsymbol{\tau}$ (e.g., using \texttt{torch.pca\_lowrank}). Although performing a full SVD (e.g., using \texttt{torch.linalg.svd}) results in higher time complexity  $\mathcal{O}\big(\min(d_{in}, d_{out})d_{in} d_{out}\big)$, for the matrix dimensions tested in our experiments there was no substantial overhead, so we keep the full-SVD implementation.

\renewcommand{\lstlistingname}{Pseudocode}
\begin{lstlisting}[language=Python, caption={Khatri-Rao Trick with SVD vs. Baseline in PyTorch (AdamW)}, label={lst:adamw_pseudocode}]
# Dimensions: x [B, d_in], grad_y [B, d_out], tau_w [d_out, d_in], tau_b [d_out]

# --- Baseline ---
per_sample_grad_w = grad_y.unsqueeze(2) * x.unsqueeze(1) # [B, d_out, d_in]
grad_w_flat = per_sample_grad_w.flatten(start_dim=1)     # [B, d_out * d_in]
Q_weight = (grad_w_flat * tau_w.unsqueeze(0)) @ grad_w_flat.t()   
Q_bias = (grad_y * tau_b.unsqueeze(0)) @ grad_y.t()
Q_baseline = Q_weight + Q_bias

# --- Khatri-Rao Trick with SVD ---
U, S, Vh = torch.linalg.svd(tau_w, full_matrices=False)
u = U[:, :r].t().unsqueeze(1)       # [r, 1, d_out]
s = S[:r].view(r, 1, 1)             # [r, 1, 1]
v = Vh[:r, :].unsqueeze(1)          # [r, 1, d_in]

Gx = (x * v) @ x.t()                # [r, B, B]
Gy = (grad_y * u) @ grad_y.t()      # [r, B, B]

Q_weight = (s * Gx * Gy).sum(dim=0) # Sum over rank -> [B, B]
Q_bias = (grad_y * tau_b.unsqueeze(0)) @ grad_y.t()
Q_optimized = Q_weight + Q_bias
\end{lstlisting}

In \cref{tab:khatri_rao_svd_bench_summary}, we summarize the benchmark execution. The setup mirrors the one used for the standard form (\cref{appendix:Khatri-rao-sgd}). We used rank-5 approximation  ($r=5$), as is the one used throughout the paper. Table (a) and (b) match the dimensions of ViT-c for CIFAR-10 and CIFAR-100 respectively. Table (c) matches the dimensions of ConvNeXt-Tiny on the Tiny-ImageNet dataset. SVD computations introduce non-negligible overhead that can sometimes increase execution time and memory consumption compared to the baseline; we highlight this performance deterioration in red. However, as layer dimensions progressively scale, the asymptotic benefits of the decomposition prevail, yielding substantial speedups, massive memory reductions and avoiding Out-of-Memory errors. Finally, we see that as batch-size increases memory gains become worse, but the speed gains improve.

\begin{table}[h!]
\centering
\begin{subtable}[t]{0.48\textwidth}
\centering
\caption{$d_{in}=384,\ d_{out}=10$}
\begin{tabular}{c|r@{,\ }l|r@{,\ }l}
\toprule
BS & \multicolumn{2}{c|}{Baseline (ms, MB)} & \multicolumn{2}{c}{Khatri-Rao SVD (ms, MB)} \\
\midrule
128  & $\quad\;$0.37 & 17.7  & 1.77 (\textcolor{mutedred}{\textbf{-384}\%}) & 16.9 (\textbf{5}\%) \\
512  & 1.20  & 32.2  & 2.17 (\textcolor{mutedred}{\textbf{-81}\%})  & 34.4 (\textcolor{mutedred}{\textbf{-7}\%}) \\
2048 & 12.5  & 112   & 8.40 (\textbf{33}\%)   & 337 (\textcolor{mutedred}{\textbf{-201}\%}) \\
\bottomrule
\end{tabular}
\end{subtable}
\hfill
\begin{subtable}[t]{0.48\textwidth}
\centering
\caption{$d_{in}=384,\ d_{out}=100$}
\begin{tabular}{c|r@{,\ }l|r@{,\ }l}
\toprule
BS & \multicolumn{2}{c|}{Baseline (ms, MB)} & \multicolumn{2}{c}{Khatri-Rao SVD (ms, MB)} \\
\midrule
128  & $\quad\;$ 1.34 & 51.6  & 4.77 (\textcolor{mutedred}{\textbf{-257}\%}) & 17.5 (\textbf{66}\%) \\
512  & 7.22 & 169   & 6.01 (\textbf{17}\%)   & 34.9 (\textbf{79}\%) \\
2048 & 113  & 647   & 14.0 (\textbf{88}\%)   & 332 (\textbf{49}\%) \\
\bottomrule
\end{tabular}
\end{subtable}

\vspace{0.6em}

\begin{subtable}[t]{0.48\textwidth}
\centering
\caption{$d_{in}=768,\ d_{out}=200$}
\begin{tabular}{c|r@{,\ }l|r@{,\ }l}
\toprule
BS & \multicolumn{2}{c|}{Baseline (ms, MB)} & \multicolumn{2}{c}{Khatri-Rao SVD (ms, MB)} \\
\midrule
256  & $\quad\;$  9.03 & 316   & 9.71 (\textcolor{mutedred}{\textbf{-8}\%}) & 21.5 (\textbf{93}\%) \\
1024 & 137  & 1222  & 14.8 (\textbf{89}\%) & 93.2 (\textbf{92}\%) \\
4096 & 2882 & 4966  & 73.8 (\textbf{97}\%) & 1305 (\textbf{74}\%) \\
\bottomrule
\end{tabular}
\end{subtable}
\hfill
\begin{subtable}[t]{0.48\textwidth}
\centering
\caption{$d_{in}=768,\ d_{out}=1000$}
\begin{tabular}{c|r@{,\ }l|r@{,\ }l}
\toprule
BS & \multicolumn{2}{c|}{Baseline (ms, MB)} & \multicolumn{2}{c}{Khatri-Rao SVD (ms, MB)} \\
\midrule
256  & $\quad\;$   62.3 & 1519 & 89.2 (\textcolor{mutedred}{\textbf{-43}\%}) & 34.3 (\textbf{98}\%) \\
1024 & 830  & 6027 & 98.4 (\textbf{88}\%)  & 103 (\textbf{98}\%) \\
4096 & \multicolumn{2}{c|}{Out-of-Memory} & 212 & 1324  \\
\bottomrule
\end{tabular}
\end{subtable}
\vspace{5pt}
\caption{Benchmark comparison of Baseline and Khatri-Rao with rank-5 SVD approximation (AdamW form).}
\label{tab:khatri_rao_svd_bench_summary}
\end{table}

\subsection{Implicit computation of $\Delta\boldsymbol{\theta}_b$ and $\mathbf{G} \Delta\boldsymbol{\theta}_b$}
\label{appendix:implicit_gdg}
To fully realize the computational and memory benefits described above, it is crucial that \emph{all} quantities defining the optimization problem are computed without explicitly instantiating the per-sample gradients $\mathbf{G}$. Evidently, it is pointless to avoid materializing $\mathbf{G}$ for the Gram matrix $\mathbf{Q}$, only to materialize it later to compute the vector $\mathbf{G} \Delta\boldsymbol{\theta}_b \in \mathbb{R}^B$. 

Let the standard parameter update proposed by the underlying optimizer be $\Delta\boldsymbol{\theta}_b$. To form the defining quantities of the optimization problem, we explicitly invoke the underlying optimizer (e.g., SGD with momentum or AdamW) to compute $\Delta\boldsymbol{\theta}_b$, but we do so \emph{strictly} for the layers involved in the approximation---which we advocate to be solely the last  layer ($\ell=L$). We emphasize that the gradients fed to the optimizer are the standard batch-averaged gradients, efficiently calculated via a regular backward pass that avoids materializing per-sample gradients. Partitioning this update into the weight component $\Delta\boldsymbol{\theta}_{b,w} \in \mathbb{R}^{d_{out} d_{in}}$ (in flattened format) and the bias component $\Delta\boldsymbol{\theta}_{b,b} \in \mathbb{R}^{d_{out}}$, we can then compute the product $\mathbf{G} \Delta\boldsymbol{\theta}_b = \mathbf{G}_w \Delta\boldsymbol{\theta}_{b,w} + \mathbf{G}_b \Delta\boldsymbol{\theta}_{b,b}$ implicitly as:
\begin{equation*}
\mathbf{G} \Delta\boldsymbol{\theta}_b = \big( (\mathbf{x} \Delta\boldsymbol{\theta}_{b,w}^\top) \odot \mathbf{g_y} \big) \mathbf{1} + \mathbf{g_y} \Delta\boldsymbol{\theta}_{b,b}
\end{equation*}
where $\mathbf{1}\in\mathbb{R}^{d_{out}}$ is a vector of ones. In PyTorch, this translates to the following highly efficient operations:

\begin{lstlisting}[language=Python, caption={Implicit computation of $\mathbf{G} \Delta\boldsymbol{\theta}_b$ without explicitly forming $\mathbf{G}$}]
# Dimensions: feat_all (x) [B, d_in], grad_logits_all (g_y) [B, d_out]
# dg_w [d_out, d_in], dg_b [d_out]
Gdg = (feat_all @ dg_w.t() * grad_logits_all).sum(dim=1) 
Gdg += grad_logits_all @ dg_b
\end{lstlisting}
It is worth noting that if $\mathbf{G}$ were already computed and stored in memory, performing the direct matrix-vector multiplication (e.g., \texttt{grad\_w\_flat @ dg\_w\_flat}) is fast. In fact, for small layer dimensions (e.g., $d_{in}=64, d_{out} \in \{10, 100\}$), the explicit method using the materialized $\mathbf{G}$ is usually a bit faster.  However, when evaluating larger sizes like $d_{in}=2048$ with $d_{out} \in \{100, 1000\}$, the implicit method is an order of magnitude faster and less susceptible Out-of-Memory (OOM) errors. We note that our setup uses an NVIDIA Tesla T4 GPU with 16 GB, where OOM is triggered for the explicit baseline computing $\mathbf{G}$ when evaluating at $d_{in}=2048, d_{out}=1000$ for batch sizes $BS \ge 1024$. Thus, by avoiding the materialization of $\mathbf{G}$ everywhere, we retain excellent scaling capabilities across both memory footprint and runtime.

\section{Guarding against instabilities}
\label{appendix:stability_preconditioning}

Before running the constrained optimization, we apply two lightweight safeguards to improve numerical robustness for the computed $\|\Delta\boldsymbol{\theta}_b\|^2$, and $\mathbf{Q}$ for SGD and $\mathbf{Q}_1, \mathbf{Q}_2$ for AdamW. Since in practice we only use the last layer to approximate these quantities (i.e. $\|\Delta\boldsymbol{\theta}_{b,L}\|^2$, $\mathbf{Q}_L$, $\mathbf{Q}_{L,1}, \mathbf{Q}_{L,2}$), it is possible that the computed values occasionally diverge from the actual global values, which can inadvertently affect optimization stability. Specifically, we found that $\|\Delta\boldsymbol{\theta}_b\|^2$ may occasionally spike up, making the constraint of \cref{eq:opt_problem_sgd} (for SGD) or \cref{eq:opt_problem_adamw} (for AdamW) overly loose and distancing the update too much from the trusted region. Additionally, the proxy matrices (e.g., $\mathbf{Q}_L$, $\mathbf{Q}_{L,1}, \mathbf{Q}_{L,2}$) may also become ill-conditioned, necessitating a mechanism to regularize them so as to ensure our solver remains stable.

\subsection{Safeguard 1: Spike control for $\|\Delta\boldsymbol{\theta}_b\|^2$}\label{sec:spike guard}

\begin{algorithm}[h!]
\caption{Spike guard}
\label{alg:dgdg spike guard}
\begin{algorithmic}[1]
\STATE At current iteration $t$ compute $\|\Delta\boldsymbol{\theta}_{b,L}\|^2$ , $\mathbf{G}_L \Delta\boldsymbol{\theta}_{b,L}$
\STATE \textbf{if} $t=1$: $\mathrm{ema\_dg} \leftarrow \|\Delta\boldsymbol{\theta}_{b,L}\|^2 + 1$
\STATE $\mathrm{max\_dg} \leftarrow \mathrm{max\_raise\_dg} \cdot \mathrm{ema\_dg}$
\IF{$\|\Delta\boldsymbol{\theta}_{b,L}\|^2 > \mathrm{max\_dg}$}
  \STATE $\mathbf{G}_L \Delta\boldsymbol{\theta}_{b,L} \leftarrow \frac{\mathrm{max\_dg}}{\|\Delta\boldsymbol{\theta}_{b,L}\|^2} \mathbf{G}_L \Delta\boldsymbol{\theta}_{b,L} \quad$ \COMMENT{Scaling down $\mathbf{G}_L \Delta\boldsymbol{\theta}_{b,L}$} 
  \STATE $\|\Delta\boldsymbol{\theta}_{b,L}\|^2 \leftarrow \mathrm{max\_dg} \qquad \qquad$ \COMMENT{Clipping before used in \cref{eq:opt_problem_sgd} or \eqref{eq:opt_problem_adamw}}
\ENDIF
\STATE $\mathrm{ema\_dg} \leftarrow \mathrm{mom\_dg} \cdot \mathrm{ema\_dg} + (1-\mathrm{mom\_dg})\cdot \|\Delta\boldsymbol{\theta}_{b,L}\|^2  $
\end{algorithmic}
\end{algorithm}

The idea is to keep an exponential moving average (EMA), which we note as $\mathrm{ema\_dg}$, on the computed $\|\Delta\boldsymbol{\theta}_{b,L}\|^2$.  If at some iteration $t$ we observe a more than $\mathrm{max\_raise\_dg}$ times increase of $\|\Delta\boldsymbol{\theta}_{b,L}\|^2$ with respect to EMA value, then we will clip. In \cref{alg:dgdg spike guard} we assume that $t=1$ is the first iteration. The reason behind adding the offset +1 in the initialization (see line 2) is because some training routines include a warm-up phase for the first iterations. During this phase the tiny learning rates make the $\|\Delta\boldsymbol{\theta}_{b,L}\|$ also tiny. Without the offset +1 the ema values $\mathrm{ema\_dg}$ takes too many rounds to increase and this leads to unnecessarily many clippings. We used $ \mathrm{mom\_dg}=0.98$ and $\mathrm{max\_raise\_dg} = 1.25$.

\subsection{Safeguard 2: Avoiding ill-conditioning for $\mathbf{Q}$}
\label{sec:ill conditioning guard}

Although the matrix $\mathbf{Q}$  is positive semi-definite, it can become highly ill-conditioned—either practically due to numerical precision limits, or natively due to rank deficiency in the extreme case where two per-sample gradients are linearly dependent. This resultant ill-conditioning leads to unstable updates. To prevent this, we monitor the condition number of $\mathbf{Q}$. Computing the exact condition number requires full spectral decomposition, which is computationally prohibitive as it has $\mathcal{O}(B^3)$ time complexity. Instead, we use the Lanczos algorithm to approximate the extremal eigenvalues of $\mathbf{Q}$ by projecting it onto a Krylov subspace. If $I_{L}$ is the number of iterations of the algorithm, then a $I_{L} \times I_{L}$ symmetric tridiagonal matrix $\mathbf{T}$ is efficiently constructed, whose largest and smallest eigenvalues provide efficient approximations to those of $\mathbf{Q}$. The overall time complexity of this approximation is only $\mathcal{O}(I_{L} B^2)$ (dominated by $I_{L}$ matrix-vector multiplications), representing a substantial speedup over exact decomposition since typically $I_{L} \ll B$. In practice, on an NVIDIA Tesla T4 GPU, we found that for smaller batch sizes ($B \le 256$) computing the exact eigenvalues directly via standard PyTorch routines overhead remains negligible and is preferred. For larger batches where $\mathcal{O}(B^3)$ becomes noticeably slow, we switch to the Lanczos approximation, setting the number of iterations to $I_{L} = \min(\max(20, \lfloor B/32 \rfloor),128)$.

Once the extremal eigenvalues ($\lambda_{\max}$ and $\lambda_{\min}$) are estimated, we evaluate the condition number $\kappa = \lambda_{\max} / \lambda_{\min}$. We maintain an exponential moving average (EMA) of the condition number, denoted $\bar{\kappa}$. If the current condition number exceeds a maximum allowed growth factor relative to the average ($\kappa > \mathrm{max\_raise\_k}\cdot \bar{\kappa}$), we apply adaptive preconditioning by adding a scaled identity matrix $\alpha \mathbf{I}$ to $\mathbf{Q}$. The scalar $\alpha$ is analytically derived to precisely enforce the target condition number $\mathrm{max\_raise\_k} \cdot\bar{\kappa}$.  We used $ \mathrm{mom\_k}=0.98$ and $\mathrm{max\_raise\_k} = 1.25$.

\begin{algorithm}[H]
\caption{Extremal Eigenvalue Approximation (Lanczos)}
\label{alg:lanczos}
\begin{algorithmic}[1]
\STATE \textbf{Input:}  Symmetric matrix $\mathbf{Q} \in \mathbb{R}^{B \times B}$, iterations $I_L$.
\STATE Sample random vector $\mathbf{v} \sim \mathcal{N}(\mathbf{0}, \mathbf{I})$
\STATE $\mathbf{v} \leftarrow \mathbf{v} / \|\mathbf{v}\|$
\STATE $\mathbf{v}_{\text{prev}} \leftarrow \mathbf{0}$, $\beta_0 \leftarrow 0$
\FOR{$i = 1$ \textbf{to} $I_{L}$}
    \STATE $\mathbf{w} \leftarrow \mathbf{Q} \mathbf{v}$
    \STATE $\alpha_i \leftarrow \mathbf{v}^\top \mathbf{w}$
    \STATE $\mathbf{w} \leftarrow \mathbf{w} - \alpha_i \mathbf{v} - \beta_{i-1} \mathbf{v}_{\text{prev}}$
    \STATE $\beta_i \leftarrow \|\mathbf{w}\|$
    \STATE $\mathbf{v}_{\text{prev}} \leftarrow \mathbf{v}$
    \STATE $\mathbf{v} \leftarrow \mathbf{w} / \beta_i$
\ENDFOR
\STATE Construct a $I_{L} \times I_{L}$ symmetric tridiagonal matrix $\mathbf{T}$ with $\alpha_1, \dots, \alpha_{I_L}$ on the main diagonal and $\beta_1, \dots, \beta_{I_{L}-1}$ on the first sub-diagonal and first super-diagonal.
\STATE Compute extremal eigenvalues $\lambda_{\min}, \lambda_{\max}$ of $\mathbf{T}$.
\STATE \textbf{Output:} $\lambda_{\min}, \lambda_{\max}$
\end{algorithmic}
\end{algorithm}

\begin{algorithm}[H]
\caption{Adaptive Preconditioning of $\mathbf{Q}$}
\label{alg:cond_monitor}
\begin{algorithmic}[1]
\STATE \textbf{Input:} Matrix $\mathbf{Q}$, EMA state $\bar{\kappa}$, momentum $\mathrm{mom\_k}$, threshold multiplier $\mathrm{max\_raise\_k}$.
\STATE \textbf{if} $B \le 256$ \textbf{then} $\lambda_{\min}, \lambda_{\max}$ from exact $\texttt{torch.linalg.eigvalsh}(\mathbf{Q})$ 
\STATE \textbf{else} $\lambda_{\min}, \lambda_{\max} \leftarrow \text{Algorithm~\ref{alg:lanczos}}(\mathbf{Q}, I_L=\min(\max(20, \lfloor B/32 \rfloor),128) )$
\STATE $\kappa \leftarrow \lambda_{\max} / \lambda_{\min}$
\STATE $\kappa_{\max} \leftarrow \mathrm{max\_raise\_k} \cdot \bar{\kappa}$
\IF{$\kappa > \kappa_{\max}$}
    \STATE $\alpha \leftarrow (\lambda_{\max} - \kappa_{\max}\lambda_{\min})/(\kappa_{\max} - 1.0)$
    \STATE $\mathbf{Q} \leftarrow \mathbf{Q} + \alpha \mathbf{I}$
    \STATE $\kappa \leftarrow \kappa_{\max}$
\ENDIF
\STATE $\bar{\kappa} \leftarrow\mathrm{mom\_k}\cdot \bar{\kappa} + (1 -\mathrm{mom\_k}) \cdot \kappa$ 
\STATE \textbf{Output:} $\mathbf{Q}, \bar{\kappa}$
\end{algorithmic}
\end{algorithm}

\section{Solver and implementation details}
\label{appendix:chambolle_pock_details}

\subsection{Overview of the Chambolle--Pock algorithm}
The Chambolle--Pock (CP) algorithm \cite{chambolle_pock} is a first-order primal--dual method for convex optimization problems of the form:
\[
\min_{\bm{w}\in\mathbb{R}^B} \; F(\mathbf{K}\bm{w}) + G(\bm{w}),
\]
where $\mathbf{K}: \mathbb{R}^B\to\mathbb{R}^B$ is a linear operator and $F, G$ are proper, convex (possibly non-smooth) functions. CP introduces a dual variable $\mathbf{y}$ associated with the linear term $\mathbf{K}\bm{w}$ and performs alternating proximal updates. 
The algorithm is particularly attractive when the proximal operators of $G$ and $F^\ast$ (the convex conjugate of $F$) are computationally inexpensive. Convergence is guaranteed provided the step sizes $\tau, \sigma > 0$ satisfy the condition $\tau\sigma\|\mathbf{K}\|^2_\sigma < 1$, where $\|\mathbf{K}\|_\sigma$ denotes the spectral norm (largest singular value) of $\mathbf{K}$. The general update steps at iteration $k$ are:
\begin{align*}
\mathbf{y}^{k+1} &= \operatorname{prox}_{\sigma F^{*}}\!\left(\mathbf{y}^{k} + \sigma \mathbf{K} \widetilde{\bm{w}}^{k}\right), \\
\bm{w}^{k+1} &= \operatorname{prox}_{\tau G}\!\left(\bm{w}^{k} - \tau \mathbf{K}^{\top} \mathbf{y}^{k+1}\right), \\
\widetilde{\bm{w}}^{k+1} &= \bm{w}^{k+1} + \theta\bigl(\bm{w}^{k+1} - \bm{w}^{k}\bigr).
\end{align*}
We now specialize this framework to our specific problem.

\subsection{Problem statement}
Our optimization problem is:
\begin{equation}\label{eq:app_problem}
\begin{aligned}
\bm{w}^\star 
&= \arg\min_{\bm{w} \in \mathcal{C}}
    \; H\big( -\mathbf{Q}_1 \bm{w} + \mathbf{q} \big) \\
\end{aligned}
\end{equation}
where $\mathcal{C} = \left\{ \bm{w}\in\mathbb{R}^B \; \middle| \bm{w}^\top \mathbf{Q}_2 \bm{w} \le r,\; \mathbf{0}\le \bm{w}\le \mathbf{t} \right\}$, $H(\mathbf{v}) = \sum_{i=1}^{B} \max(0, [\mathbf{v}]_i)$, $\mathbf{Q}_1, \mathbf{Q}_2 \in \mathbb{R}^{B\times B}$ are symmetric positive semi-definite matrices (for SGD case it is $\mathbf{Q}_1=\mathbf{Q}_2$), $\mathbf{q} \in \mathbb{R}^B$, $r > 0$, $\mathbf{t}= t\mathbf{1}$ for some $t>0$ and $\mathbf{1}$ a vector with all its elements equal to one. Let $I_{\mathcal{S}}(\mathbf{x})$ denote the indicator function of a set $\mathcal{S}$ in the context of convex analysis, i.e., $I_{\mathcal{S}}(\mathbf{x}) = 0$ if $\mathbf{x} \in \mathcal{S}$ and $+\infty$ otherwise. We map our problem to the CP template by choosing:
\[
\mathbf{K} = -\mathbf{Q}_1, \qquad 
F(\mathbf{z}) = H(\mathbf{z} + \mathbf{q}), \qquad 
G(\bm{w}) = I_\mathcal{C}  (\bm{w})= I_{\{\mathbf{0}\le\bm{w}\le \mathbf{t}\}}(\bm{w}) \;+\; I_{\{\bm{w}^\top \mathbf{Q}_2 \bm{w} \le r\}}(\bm{w}) 
\]

\subsection{Derivation of the updates}
CP requires the proximal operator of the conjugate $F^\ast$ and the proximal operator of $G$.

\paragraph{Conjugate of the (shifted) hinge sum.}
 Given that the convex conjugate of $f_i(u)=\max(0, u+[\mathbf{q}]_i)$ with $u\in\mathbb{R}$ is:
\[
f_i^\ast(y) = \sup_{u\in\mathbb{R}} \{ u y - f_i(u) \} = 
\begin{cases} 
-[\mathbf{q}]_iy, & \text{if } y\in[0,1],\\
+\infty, & \text{otherwise.}
\end{cases}
\]
the sum of conjugates:
\[
F^\ast(\mathbf{y}) = \sum_{i=1}^B f_i^\ast([\mathbf{y}]_i) = I_{[0,1]^B}(\mathbf{y}) - \sum_{i=1}^B [\mathbf{y}]_i[\mathbf{q}]_i = I_{[0,1]^B}(\mathbf{y}) - \langle \mathbf{y}, \mathbf{q} \rangle,
\]
where $\langle \mathbf{y}, \mathbf{q} \rangle$ is the inner product of the vectors  $\mathbf{y}$ and $\mathbf{q}$.

\paragraph{Proximal of $\sigma F^\ast$.}
The proximal operator of $\sigma F^\ast$ at $\mathbf{u}$ is defined as:
\begin{align*}
\operatorname{prox}_{\sigma F^\ast}(\mathbf{u})
&= \arg\min_{\mathbf{y}} \big\{\sigma\left( I_{[0,1]^B}(\mathbf{y}) - \langle \mathbf{y}, \mathbf{q}\rangle \right)+ \frac{1}{2}\|\mathbf{y}-\mathbf{u}\|^2 \big\} \\
&= \arg\min_{\mathbf{y}} \left\{2\sigma I_{[0,1]^B}(\mathbf{y}) -2 \langle \mathbf{y}, \sigma\mathbf{q}\rangle + \|\mathbf{y}\|^2 - 2\langle \mathbf{y}, \mathbf{u}\rangle + \|\mathbf{u}\|^2  \right\} \\
&= \arg\min_{\mathbf{y}} \left\{2\sigma I_{[0,1]^B}(\mathbf{y}) + \|\mathbf{y}\|^2 - 2\langle \mathbf{y}, \mathbf{u} + \sigma \mathbf{q}\rangle \right\}\\
&= \arg\min_{\mathbf{y} \in [0,1]^B} \big\{ \|\mathbf{y} - (\mathbf{u} + \sigma \mathbf{q})\|^2 -\|\mathbf{u} + \sigma \mathbf{q}\|^2 \big\} \\
&= \arg\min_{\mathbf{y} \in [0,1]^B} \left\{ \|\mathbf{y} - (\mathbf{u} + \sigma \mathbf{q})\|^2  \right\} = \Pi_{[0,1]^B}(\mathbf{u} + \sigma \mathbf{q}).
\end{align*}
Note that  the term $\|\mathbf{u} + \sigma \mathbf{q}\|^2$ is constant with respect to the minimization.

\paragraph{Proximal of $G$.} 
The proximal operator is $\operatorname{prox}_{\tau G}(\mathbf{v}) = \arg\min_{\bm{w}} \left(\tau G(\bm{w})+  \frac{1}{2}\|\bm{w}-\mathbf{v}\|^2 \right)$. This  actually amounts to an Euclidean projection on a convex set problem as:
\[
\operatorname{prox}_{\tau G}(\mathbf{v})
=\arg\min_{\bm{w}\in\mathcal{C}}
\left\|\bm{w}-\mathbf{v}\right\|^2
=\Pi_{\mathcal{C}}\!\left(\mathbf{v}\right), % \frac{1}{1+2\tau\epsilon}
\]
where the (convex) set $\mathcal{C}$ is the intersection of a box and an ellipsoid,
\[
\mathcal{C}=\mathcal{B}\cap\mathcal{E},\qquad
\mathcal{B}:=[0,t]^B,\qquad
\mathcal{E}:=\{\bm{w}\in\mathbb{R}^B:\bm{w}^\top \mathbf{Q}_2\bm{w}\le r\}.
\]

Evaluating such proximal operators is often non-trivial \cite{proximal_splitting}. A standard approach for projecting onto an intersection of convex sets is \textit{Dykstra's projection algorithm} \cite{Dykstra_projection_onto_intersection}, which alternates projections onto individual sets while incorporating correction vectors to ensure convergence to the true Euclidean projection $\Pi_{\mathcal{B}\cap\mathcal{E}}(\bm{w}_0)$. However, the exact projection onto the ellipsoid $\Pi_{\mathcal{E}}$ is computationally demanding, typically requiring an eigendecomposition of $\mathbf{Q}_2$ or an iterative root-finding procedure involving $B\times B$ linear systems, leading to an $\mathcal{O}(B^3)$ routine.

Since high precision is not critical within our CP iterations, we instead employ a computationally efficient alternating scheme \cite{POCS_algorithms_review}. 
\begin{enumerate}
    \item \textbf{Box Projection:} We first project onto $\mathcal{B}$ via element-wise clipping into $[0, t]^B$ (i.e. $\Pi_{\mathcal{B}}$).
    \item \textbf{Radial Scaling:} If the resulting vector violates the ellipsoidal constraint (i.e., $\bm{w}^\top \mathbf{Q}_2 \bm{w} > r$), we scale it by a factor $\eta = \sqrt{r / (\bm{w}^\top \mathbf{Q}_2 \bm{w})}$. 
\end{enumerate}
This radial scaling provides a cheap $\mathcal{O}(B^2)$ alternative to the exact Euclidean projection. Crucially, by performing the radial scaling second, we ensure that the final point is guaranteed to lie in the feasible set $\mathcal{C}$. The opposite ordering cannot guarantee it.

\subsection{Stopping criterion}\label{sec:stopping criterion}

\paragraph{Theoretical criterion: Primal-dual gap.}
A theoretically rigorous stopping criterion for primal-dual algorithms is the duality gap, defined as $\text{Gap}(\bm{w}, \mathbf{y}) = \text{Primal}(\bm{w}) - \text{Dual}(\mathbf{y})$, where Primal and Dual are the primal and dual objectives. The standard saddle-point form of our problem is:
\[
\min_{\bm{w}} \max_{\mathbf{y}} \langle -\mathbf{Q}_1 \bm{w}, \mathbf{y} \rangle + G(\bm{w}) - F^*(\mathbf{y}).
\]
The dual objective is given by $\text{Dual}(\mathbf{y}) = -F^*(\mathbf{y}) - G^*(\mathbf{Q}_1^\top \mathbf{y})$.
While the conjugate $F^*(\mathbf{y}) = I_{[0,1]^B}(\mathbf{y}) - \mathbf{y}^\top \mathbf{q}$ is trivial to compute, the conjugate of $G$ is not. The conjugate of  $G(\bm{w}) = I_{\mathcal{C}}(\bm{w})$ is:
\[
G^*(\mathbf{v}) = \sup_{\bm{w} \in \mathcal{C}} \left( \mathbf{v}^\top \bm{w} \right).
\]
Evaluating $G^*(\mathbf{v})$ requires solving a maximization problem over the intersection of an ellipsoid and the non-negative orthant. This sub-problem is as computationally difficult as the original primal problem itself. Consequently, computing the exact duality gap at every iteration is computationally prohibitive. 

\paragraph{Practical heuristic: Mean absolute update.}
Instead, we monitor convergence via the iterates in the original space. Let $\bm{w}^k$ denote the iterate in the primal domain at step $k$. We terminate the algorithm once the mean absolute update falls below a predefined tolerance:
\begin{equation}
    \frac{1}{B} \sum_{i=1}^B | [\bm{w}^{k+1}]_i - [\bm{w}^k]_i | \le \text{tol}.
\end{equation}

\subsection{Warm Start using Projected Gradient Descent}
To accelerate convergence, we found that initializing the primal and dual variables using a few steps ($I_{PGD}=10$) of projected (sub)gradient descent (PGD) is effective. In \cref{alg:approx_proj} we use $\mathbbm{1}(\cdot)$ to denote the standard binary indicator function that evaluates to $1$ when its argument is true and $0$ otherwise. With this notation, the (sub)gradient of the objective $H(\mathbf{x})$ can be compactly written as the element-wise boolean vector $\mathbbm{1}(\mathbf{x} > \mathbf{0})$.

\subsection{Efficient Spectral Norm Estimation}
Computing the exact spectral norm $\|\mathbf{Q}_1\|_\sigma$ needed for defining the step sizes$\tau, \sigma$ using PyTorch's standard routine (\texttt{torch.linalg.matrix\_norm}) becomes a computational bottleneck as the batch size $B$ grows, given its $\mathcal{O}(B^3)$ time complexity. Since the matrix $\mathbf{Q}_1$ is symmetric and positive semi-definite (PSD), its spectral norm coincides with its largest eigenvalue, $\lambda_{\max}$. As described in \cref{appendix:stability_preconditioning}, we already perform an efficient estimation of $\lambda_{\max}$ at each iteration using the Lanczos algorithm (or exact eigensolver for small $B$) to monitor and prevent ill-conditioning. Therefore, we reuse the $\lambda_{\max}$ output from the safeguarding check as the spectral norm $\|\mathbf{Q}_1\|_\sigma$.

\subsection{Final algorithm}

We use a safety factor of 0.5 (instead of 1.0) for the step sizes and 0.9 (again instead of 1.0)  for the relaxation parameter to ensure numerical stability and dampen oscillations that might arise from the operator splitting approximation. In our implementation, the solver runs for a total of $I_{solver}=50$ iterations, inclusive of the PGD warm-up steps. 

\begin{algorithm}[h!]
\caption{Approximate projection onto $\Pi_\mathcal{C}$}
\label{alg:approx_proj}
\begin{algorithmic}[1]
\STATE {\bfseries Input:} $\bm{w}, \mathbf{Q}_2, r$
\STATE $\bm{w}  \leftarrow \Pi_{[0,2]^B}(\bm{w} )$ $\qquad\qquad\qquad$ \COMMENT{Box projection (component-wise clipping)}
\STATE $\bm{w}  \leftarrow \displaystyle\sqrt{\frac{r}{\max(r,\bm{w}^\top \mathbf{Q}_2 \bm{w})}}\bm{w} \quad\;\;$ \COMMENT{Radial scaling for ellipsoid constraint}
\STATE {\bfseries Output:} $\bm{w}$
\end{algorithmic}
\end{algorithm}

\begin{algorithm}[h!]
\caption{PGD Warm Start}
\label{alg:pgd_warm_start}
\begin{algorithmic}[1]
\STATE {\bfseries Input:} $\mathbf{Q}_1, \mathbf{Q}_2, \mathbf{q}, r, I_{PGD}, \eta$
\STATE \textbf{if} $I_{PGD}=0$  \textbf{then output:} $\mathbf{0}, \mathbf{0}$
\STATE $\bm{w} = \mathbf{0}$
\STATE $\mathbf{y} = \mathbbm{1}(\mathbf{q} > \mathbf{0})$
\FOR{$k=1, \dots, I_{PGD}$}
  \STATE $\mathbf{g} \leftarrow -\mathbf{Q}_1^\top \mathbf{y} \qquad\qquad\qquad\qquad\quad$ \COMMENT{Subgradient of $H$}
  \STATE $\bm{w}  \leftarrow \text{Algorithm~\ref{alg:approx_proj}}(\bm{w} - \eta \mathbf{g}, \mathbf{Q}_2, r) $
  \STATE $\mathbf{y} \leftarrow \mathbbm{1}(-\mathbf{Q}_1 \bm{w} + \mathbf{q} > \mathbf{0})\qquad\qquad$  \COMMENT{Update active set}
\ENDFOR
\STATE {\bfseries Output:} $\bm{w}, \mathbf{y}$
\end{algorithmic}
\end{algorithm}

\begin{algorithm}[h!]
\caption{Chambolle--Pock solver}
\label{alg:cp_full}
\begin{algorithmic}[1]
\STATE {\bfseries Input:} $\mathbf{Q}_1, \mathbf{Q}_2, \mathbf{q}, r, \text{tol}, I_{PGD}, I_{solver}, \lambda_{\max}$
\STATE $\|\mathbf{Q}_1\|_\sigma \leftarrow \lambda_{\max}\qquad\qquad\qquad\qquad\qquad\qquad\;\;$  \COMMENT{Reused from Safeguard 2}
\STATE $\tau \leftarrow 0.5/\|\mathbf{Q}_1\|_\sigma$ 
\STATE $ \sigma \leftarrow 0.5/\|\mathbf{Q}_1\|_\sigma$ 
\STATE $ \bm{w}^1, \mathbf{y}^1\leftarrow$ \text{Algorithm~\ref{alg:pgd_warm_start}}$(\mathbf{Q}_1, \mathbf{Q}_2, \mathbf{q}, r, I_{PGD}, \tau) \;$ \COMMENT{Warm start}
\STATE $\widetilde{\bm{w}}^1 = \bm{w}^1$
\FOR{$k=1, \dots,I_{solver}-I_{PGD}$}
  \STATE $\mathbf{y}^{k+1} \leftarrow \Pi_{[0,1]^B}\big(\mathbf{y}^k + \sigma ( -\mathbf{Q}_1 \widetilde{\bm{w}}^{k}  + \mathbf{q})\big)\quad\quad$\COMMENT{Proximal of $\sigma F^\ast$ (Dual step)}
  \STATE $\bm{w} \leftarrow \bm{w}^{k} + \tau \mathbf{Q}_1^\top \mathbf{y}^{k+1}$  
  \STATE $\bm{w}^{k+1} \leftarrow$ \text{Algorithm~\ref{alg:approx_proj}}$(\bm{w}, \mathbf{Q}_2, r)\quad \qquad\qquad$ \COMMENT{Proximal of $\tau G$ (Primal step)}
  \STATE $\widetilde{\bm{w}}^{\,k+1} \leftarrow \bm{w}^{k+1}  + 0.9(\bm{w}^{k+1} - \bm{w}^{k})\quad\qquad\;\;$ \COMMENT{Extrapolation}
  \STATE \textbf{if} $\frac{1}{B}\sum_{i=1}^B \big|[\bm{w}^{k+1} - \bm{w}^{k}]_i\big| \le \text{tol}$ \textbf{then} break
\ENDFOR
\STATE {\bfseries Output:} $\bm{w}$
\end{algorithmic}
\end{algorithm}

\newpage

\section{Architectural and Training Details} 
\label{appendix:architectural_details}

The \textbf{ViT} architecture consists of 7 transformer layers with 12 attention heads. For the standard version, both the hidden and MLP dimensions are set to 384, while for the \textbf{ViT-small} variant we use a dimension of 192 for both; the only difference from the reference implementation of \cite{vit_cifar_github} is that we utilize the AdamW optimizer instead of Adam. 

For convolutional models on CIFAR datasets, we follow the \textbf{ResNet-44} (and the smaller \textbf{ResNet-20}) specifications for CIFAR \cite{ResNet}, which utilize three stages of residual blocks with 16, 32, and 64 channels respectively. However, because PyTorch's \texttt{vmap} does not support Batch Normalization, we replace it with Group Normalization (GN) \cite{GroupNorm} using 8 groups across all stages to enable efficient per-sample gradient computation.

For both ResNet and ViT architectures, we adopt a baseline batch size of 128. The initial learning rate is 0.1 for the former (using SGD optimizer)  and 0.001 for the latter (using AdamW optimizer). When the batch size is increased, we keep the total number of epochs constant, and we adjust the learning rate according to the square root scaling rule \cite{lr_square_law}: 
\begin{equation}
    \mathrm{lr}_{\text{new}} = \mathrm{lr}_{\text{base}} \sqrt{\frac{B_{\text{new}}}{B_{\text{base}}}},
\end{equation}
where $B_{\text{base}} = 128$ and $\mathrm{lr}_{\text{base}}$ is the corresponding baseline learning rate. The rest of the hyperparameters match those in \cite{vit_cifar_github} for ViT and/or \cite{ResNet} for ResNet. 

For Tiny-ImageNet, a downsampled ImageNet-1K variant with 200 classes and RGB images of size $64\times 64$, we train ResNet-50 using the same recipe as in \cite{ResNet}. For ConvNeXt-T, we follow the training recipe of \cite{convnext} and report the accuracy of the exponential moving average (EMA) model, using FP16 training to avoid memory issues.

For ImageNet-100, \Cref{tab:acc_vs_resources_imagenet100} reports configurations based on the standard ConvNeXt recipe \cite{convnext}. For our "standard" baseline, the only modification is the use of half-precision (FP16) instead of FP32 to fit the total batch size of 4096 (1024 per A100 GPU) into memory. For the "standard + more data aug." configuration, we further intensify the regularization: we increase the number of RandAugment \cite{randaugment} operations from 2 to 3, and slightly raise the Mixup/Cutmix coefficients from 0.8/1.0 to 1.0/1.2 (though the RandAugment increase drives most of the difficulty). Finally, when training under the reduced 100-epoch budget, we proportionally decrease the warmup period from 20 to 5 epochs. Across all settings, we report the non-EMA model accuracy. The EMA coefficient calibrated for the standard 300-epoch ImageNet-1K recipe \cite{convnext} adapts too slowly when the total number of iterations is drastically reduced—a consequence of both the smaller dataset (100 classes) and, in some settings, the shortened 100-epoch training schedule.

For CIFAR-10, CIFAR-100, and Tiny-ImageNet, all reported results are averaged over 5 random seeds. For ImageNet-100, we report results from a single seed.

%%%%%%%%%%%%%%%%%%%%%%%%%%%%%%%%%%%%%%%%%%%%%%%%%%%%%%%%%%%%%%%%%%%%%%%%%%%%%%%
%%%%%%%%%%%%%%%%%%%%%%%%%%%%%%%%%%%%%%%%%%%%%%%%%%%%%%%%%%%%%%%%%%%%%%%%%%%%%%%

\section{Additional figures on the validity of the approximation and update quality}\label{appendix:additional_figures_approx}

In this section, we provide the derivation for the optimal scaling factor $\alpha^\star$ and present supplementary figures further validating our approximation strategy. 
As detailed in the main text, we assess performance via two metrics: (1) the \textit{structural approximation error}, measured by the relative Frobenius norm $\|\mathbf{Q} - \alpha^\star \mathbf{Q}_{\ell:L}\|_F / \|\mathbf{Q}\|_F$ for experiments with SGD optimizer and $\|\mathbf{Q}_1 - \alpha^\star \mathbf{Q}_{1,\ell:L}\|_F / \|\mathbf{Q}_1\|_F$ for experiments with AdamW; and (2) the \textit{relative effective objective reduction} (REOR), defined as $\left(J(\Delta\boldsymbol{\theta}_b) - J(\Delta\boldsymbol{\theta}^\star)\right)/J(\Delta\boldsymbol{\theta}_b)$. This second metric quantifies the reduction of the original per-sample interference objective (\cref{eq:conceptual_problem}), where $\Delta\boldsymbol{\theta}^\star = \Delta\boldsymbol{\theta}_b + \mathbf{G}^\top \mathbf{w}^\star$ and $\mathbf{w}^\star$ is the solution the solver outputs to the respective surrogate problem.

\subsection{Derivation of the Optimal Scaling Factor}
We seek the scalar $\alpha^\star$ that minimizes the Frobenius distance between the full Gram matrix $\mathbf{Q}$ and the scaled partial Gram matrix $\mathbf{Q}_{\ell:L}$. The optimization problem is:
\[
\alpha^\star = \arg\min_{\alpha \in \mathbb{R}} \|\mathbf{Q} - \alpha \mathbf{Q}_{\ell:L}\|_F^2
\]
Expanding the squared norm:
\[
\begin{aligned}
\|\mathbf{Q} - \alpha \mathbf{Q}_{\ell:L}\|_F^2 &= \langle \mathbf{Q} - \alpha \mathbf{Q}_{\ell:L}, \mathbf{Q} - \alpha \mathbf{Q}_{\ell:L} \rangle_F \\
&= \|\mathbf{Q}\|_F^2 - 2\alpha \langle \mathbf{Q}, \mathbf{Q}_{\ell:L} \rangle_F + \alpha^2 \|\mathbf{Q}_{\ell:L}\|_F^2
\end{aligned}
\]
where $\langle \mathbf{A}, \mathbf{B} \rangle_F = \operatorname{Tr}(\mathbf{A}^\top \mathbf{B})$ is the Frobenius inner product. Taking the derivative with respect to $\alpha$ and setting it to zero yields:
\[
\frac{d}{d\alpha} = -2 \langle \mathbf{Q}, \mathbf{Q}_{\ell:L} \rangle_F + 2\alpha \|\mathbf{Q}_{\ell:L}\|_F^2 = 0
\]
Solving for $\alpha$:
\[
\alpha^\star = \frac{\langle \mathbf{Q}, \mathbf{Q}_{\ell:L} \rangle_F}{\|\mathbf{Q}_{\ell:L}\|_F^2} = \frac{\operatorname{Tr}(\mathbf{Q}^\top \mathbf{Q}_{\ell:L})}{\|\mathbf{Q}_{\ell:L}\|_F^2}
\]
Since $\mathbf{Q}$ and $\mathbf{Q}_{\ell:L}$ are positive semi-definite (and non-zero), both the numerator and denominator are positive, ensuring $\alpha^\star > 0$. Substituting  $\mathbf{Q}$ with $\mathbf{Q}_1$ and  $\mathbf{Q}_{\ell:L}$ with  $\mathbf{Q}_{1,\ell:L}$ gives the corresponding result for the AdamW case.

\subsection{Additional Results}
Consistent with the main text, the top rows of the following figures display the structural error (lower is better), while the bottom rows display the distribution of the effective objective reduction (higher/positive is better).

\begin{figure}[h!]
    \centering
    \begin{subfigure}[b]{0.24\textwidth}
        \centering
        \includegraphics[width=\textwidth]{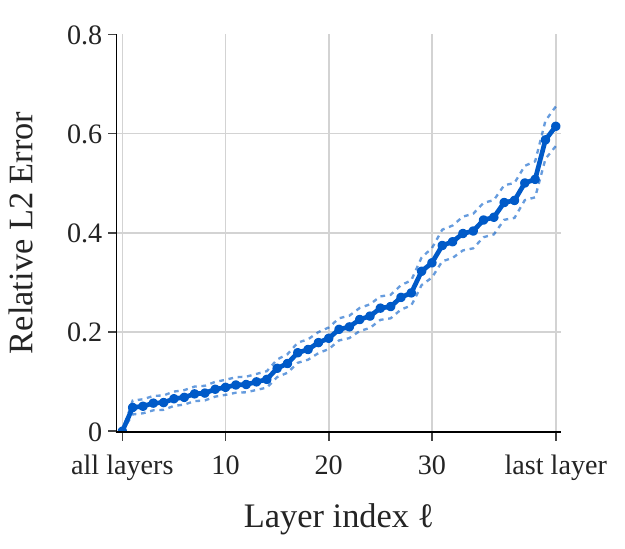}
        \vspace{2pt}
        \includegraphics[width=\textwidth]{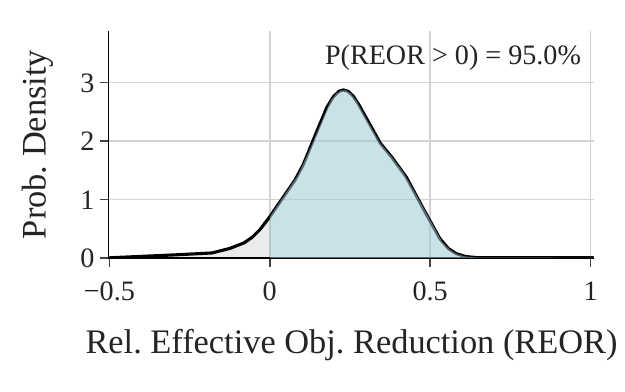}
        \caption{Epoch 10}
    \end{subfigure}
    \hfill
    \begin{subfigure}[b]{0.24\textwidth}
        \centering
        \includegraphics[width=\textwidth]{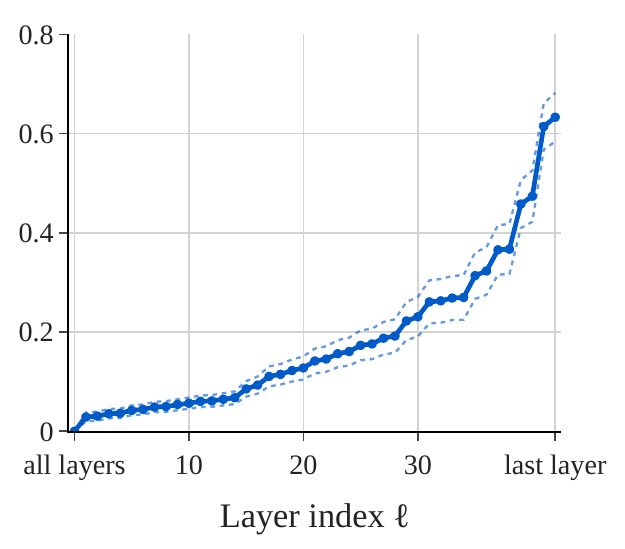}
        \vspace{2pt}
        \includegraphics[width=\textwidth]{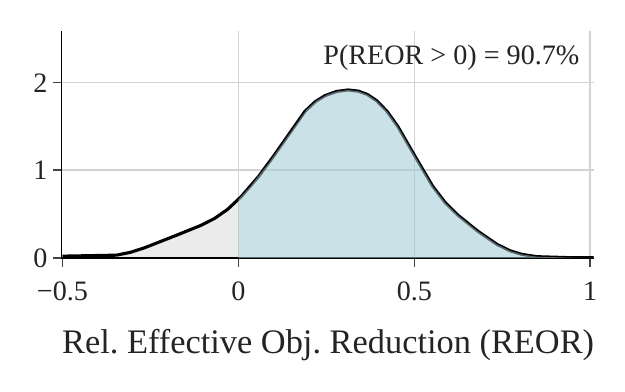}
        \caption{Epoch 50}
    \end{subfigure}
    \hfill
    \begin{subfigure}[b]{0.24\textwidth}
        \centering
        \includegraphics[width=\textwidth]{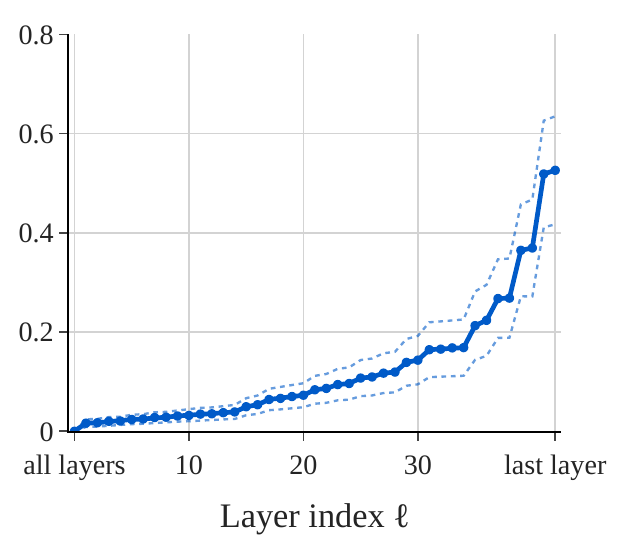}
        \vspace{2pt}
        \includegraphics[width=\textwidth]{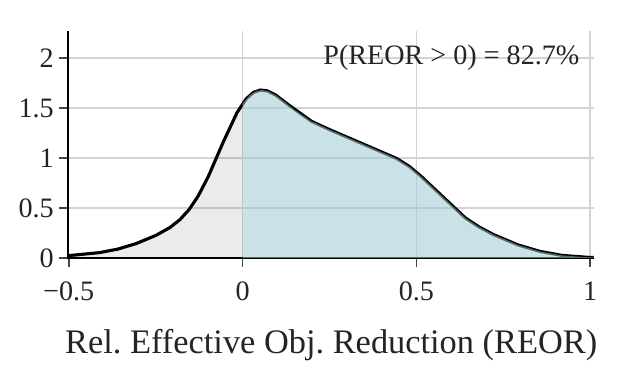}
        \caption{Epoch 100}
    \end{subfigure}
    \hfill
    \begin{subfigure}[b]{0.24\textwidth}
        \centering
        \includegraphics[width=\textwidth]{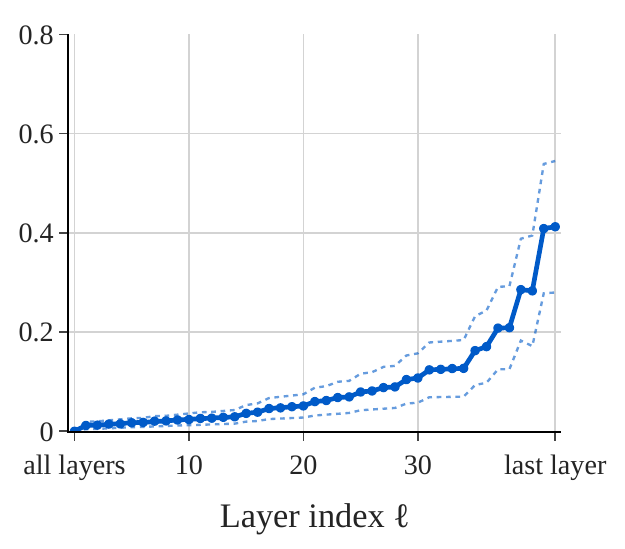}
        \vspace{2pt}
        \includegraphics[width=\textwidth]{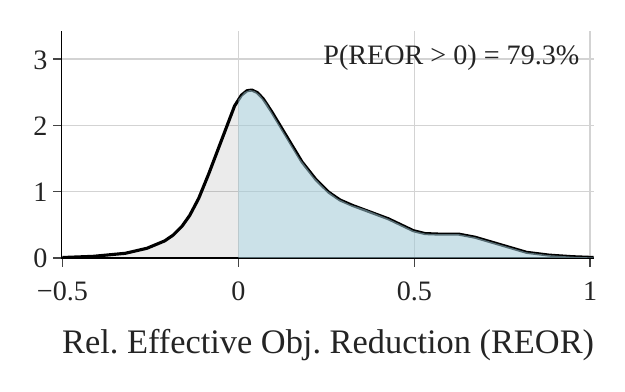}
        \caption{Epoch 150}
    \end{subfigure}
     \caption{Approximation quality for ResNet-20 on CIFAR-10 for Batch Size = 128 at different training epochs. \textbf{Top Row:} Relative $L_2$ error of the Gram matrix approximation $\mathbf{Q}$. \textbf{Bottom Row:} Distribution of the relative effective objective reduction (REOR).}   
    \label{fig:rel_l2_c10_resnet_bs128}
\end{figure}

\begin{figure}[h!]
     \centering
     \begin{subfigure}[b]{0.24\textwidth}
         \centering
         \includegraphics[width=\textwidth]{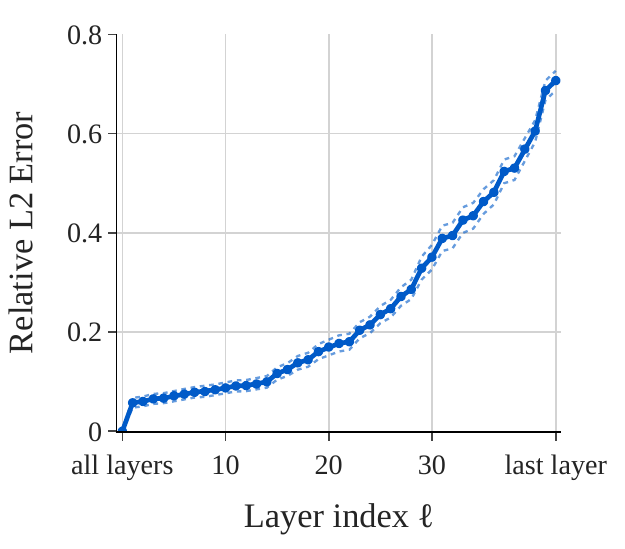}
         \vspace{2pt}
         \includegraphics[width=\textwidth]{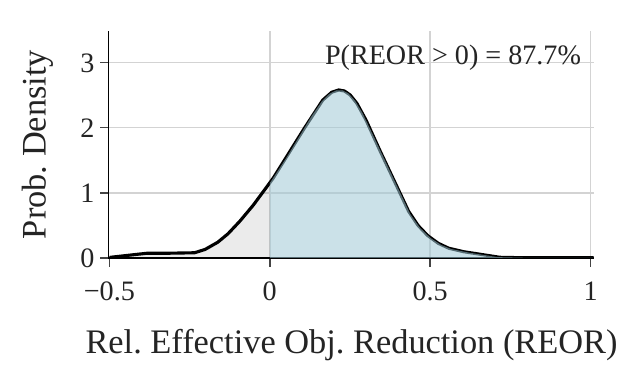}
         \caption{Epoch 10}
     \end{subfigure}
     \hfill
     \begin{subfigure}[b]{0.24\textwidth}
         \centering
         \includegraphics[width=\textwidth]{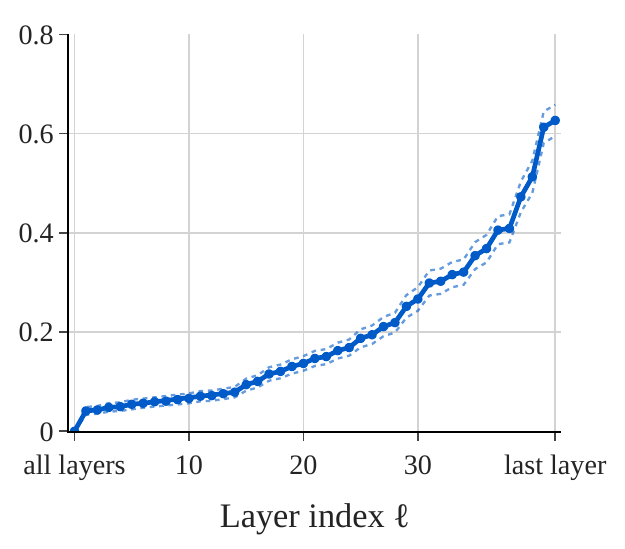}
         \vspace{2pt}
         \includegraphics[width=\textwidth]{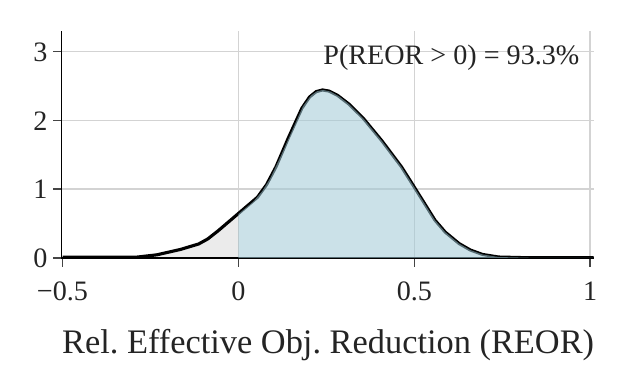}
         \caption{Epoch 50}
     \end{subfigure}
     \hfill
     \begin{subfigure}[b]{0.24\textwidth}
         \centering
         \includegraphics[width=\textwidth]{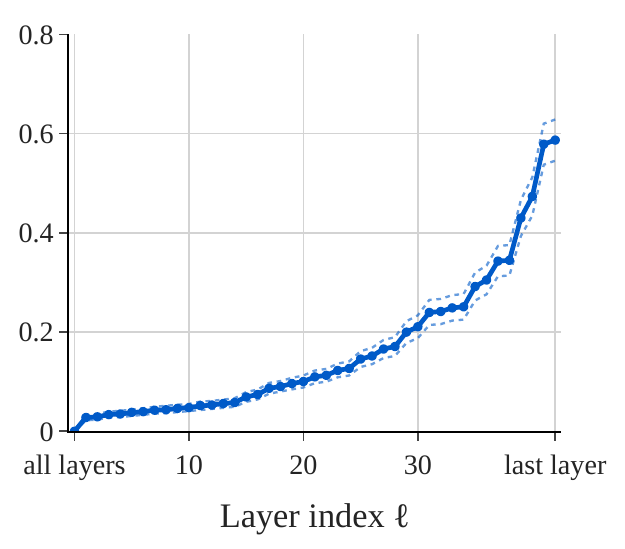}
         \vspace{2pt}
         \includegraphics[width=\textwidth]{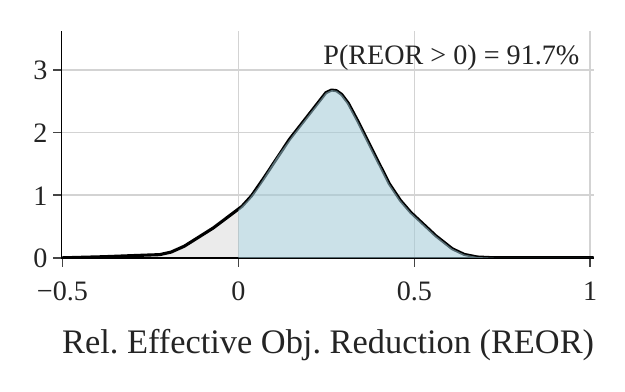}
         \caption{Epoch 100}
     \end{subfigure}
     \hfill
     \begin{subfigure}[b]{0.24\textwidth}
         \centering
         \includegraphics[width=\textwidth]{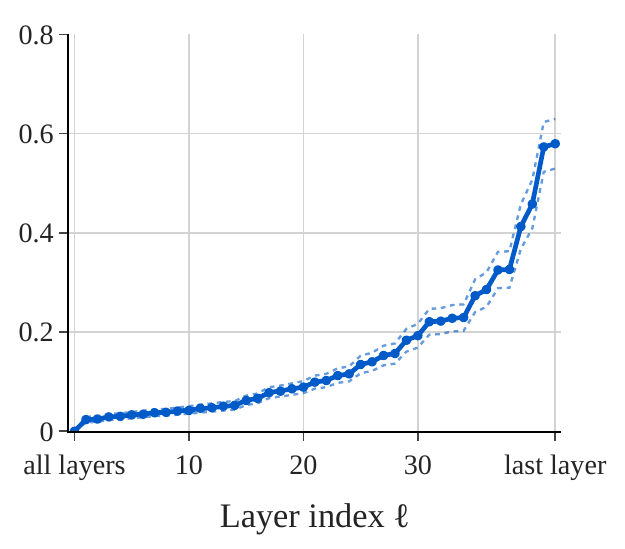}
         \vspace{2pt}
         \includegraphics[width=\textwidth]{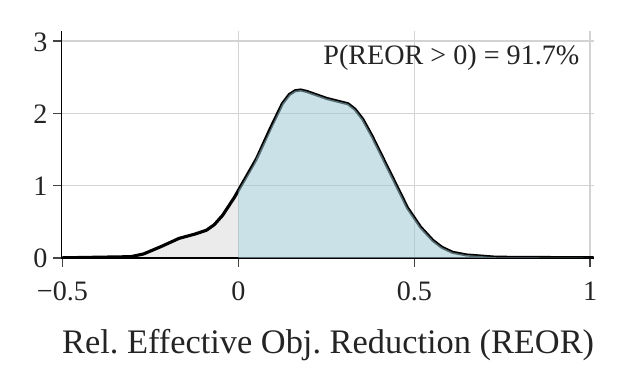}
         \caption{Epoch 150}
     \end{subfigure}        
     \caption{Approximation quality for ResNet-20 on CIFAR-100 for Batch Size = 128 at different training epochs. \textbf{Top Row:} Relative $L_2$ error of the Gram matrix approximation $\mathbf{Q}$. \textbf{Bottom Row:} Distribution of the relative effective objective reduction (REOR).}
        \label{fig:l2_c100_resnet_bs128}
\end{figure}

\newpage

\begin{figure}[h!]
     \centering
     \begin{subfigure}[b]{0.24\textwidth}
         \centering
         \includegraphics[width=\textwidth]{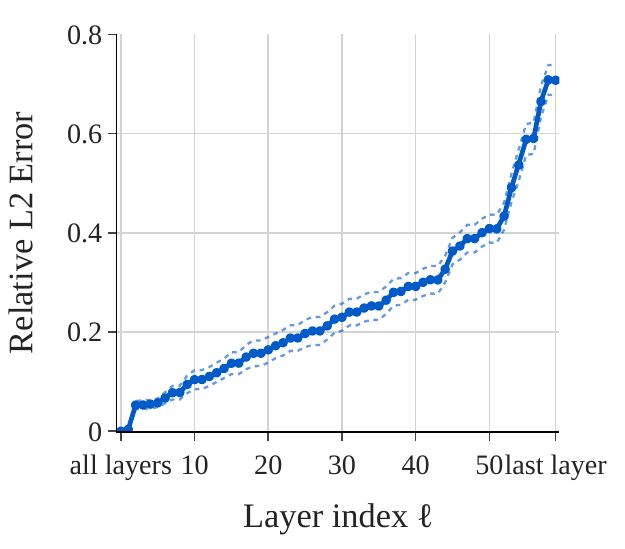}
         \vspace{2pt}
         \includegraphics[width=\textwidth]{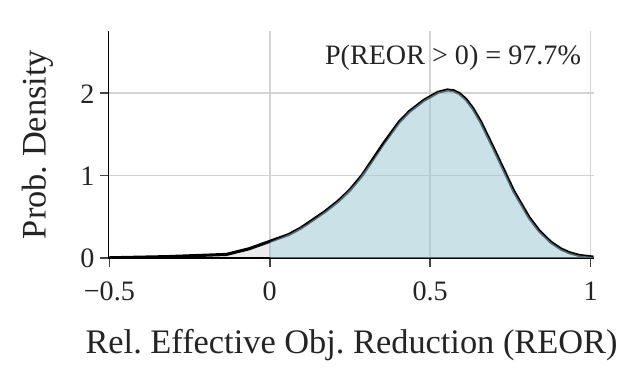}
         \caption{Epoch 10}
     \end{subfigure}
     \hfill
     \begin{subfigure}[b]{0.24\textwidth}
         \centering
         \includegraphics[width=\textwidth]{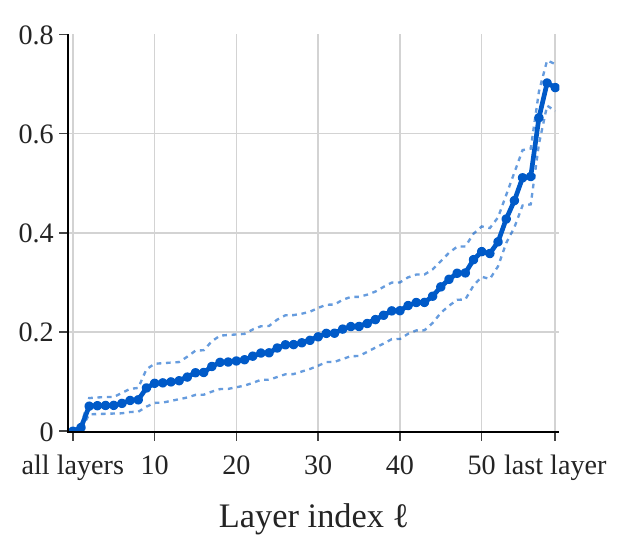}
         \vspace{2pt}
         \includegraphics[width=\textwidth]{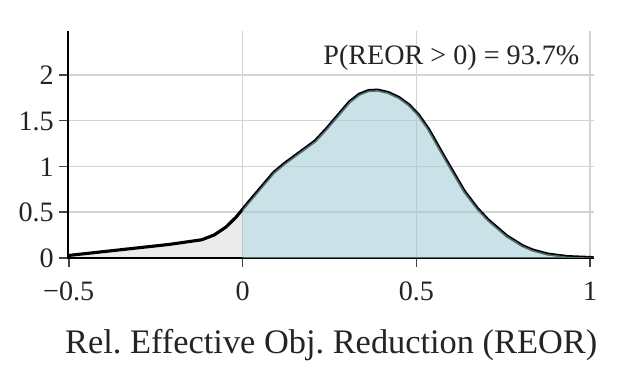}
         \caption{Epoch 50}
     \end{subfigure}
     \hfill
     \begin{subfigure}[b]{0.24\textwidth}
         \centering
         \includegraphics[width=\textwidth]{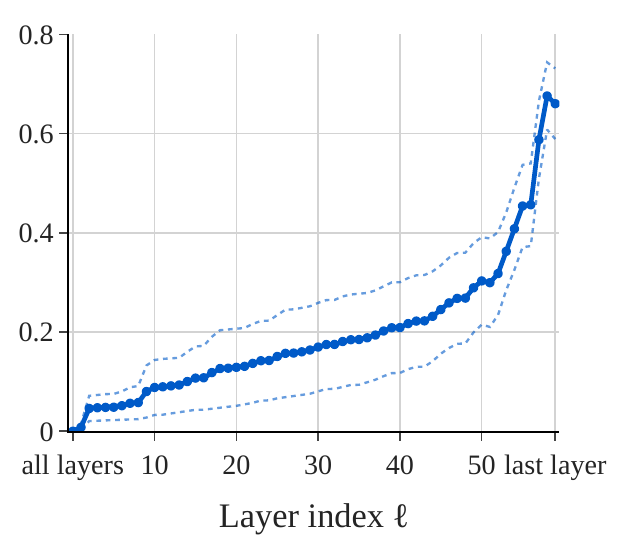}
         \vspace{2pt}
         \includegraphics[width=\textwidth]{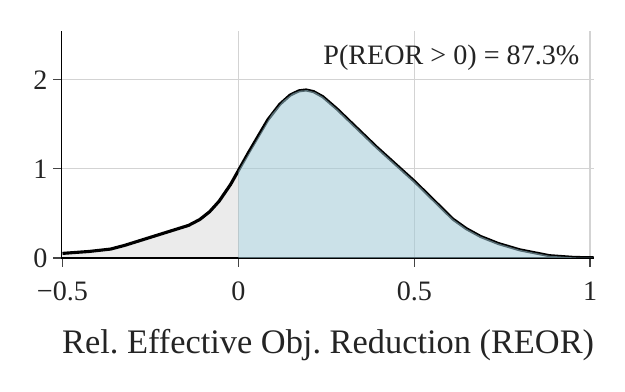}
         \caption{Epoch 100}
     \end{subfigure}
     \hfill
     \begin{subfigure}[b]{0.24\textwidth}
         \centering
         \includegraphics[width=\textwidth]{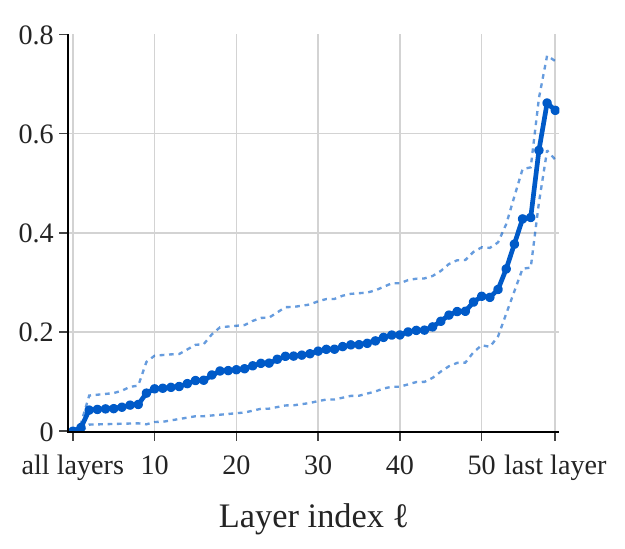}
         \vspace{2pt}
         \includegraphics[width=\textwidth]{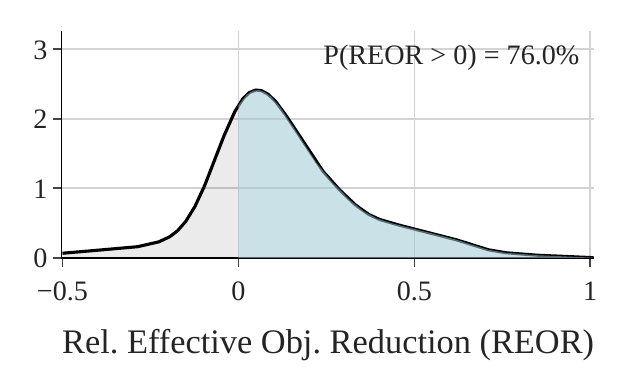}
         \caption{Epoch 150}
     \end{subfigure}
     \caption{Approximation quality for ViT-small on CIFAR-10 for Batch Size = 128 at different training epochs. \textbf{Top Row:} Relative $L_2$ error of the Gram matrix approximation $\mathbf{Q}_1$. \textbf{Bottom Row:} Distribution of the relative effective objective reduction (REOR).}
        \label{fig:rel_l2_c10_vit_small_bs128}
\end{figure}

\begin{figure}[h!]
     \centering
     \begin{subfigure}[b]{0.24\textwidth}
         \centering
         \includegraphics[width=\textwidth]{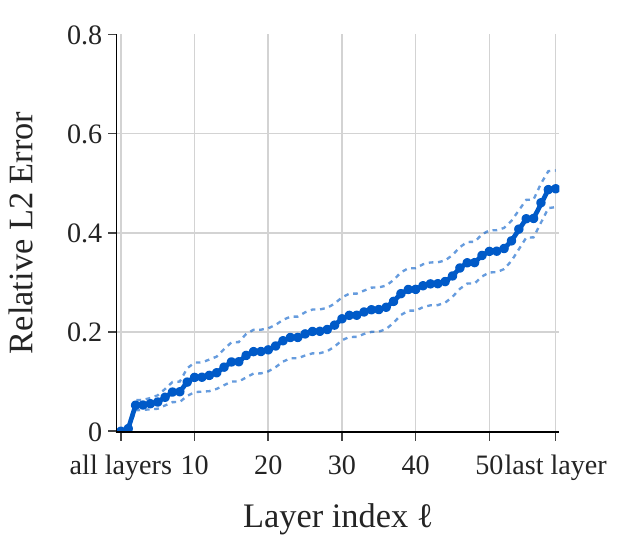}
         \vspace{2pt}
         \includegraphics[width=\textwidth]{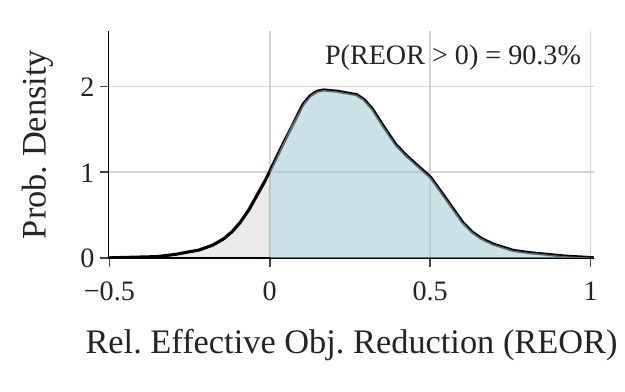}
         \caption{Epoch 10}
     \end{subfigure}
     \hfill
     \begin{subfigure}[b]{0.24\textwidth}
         \centering
         \includegraphics[width=\textwidth]{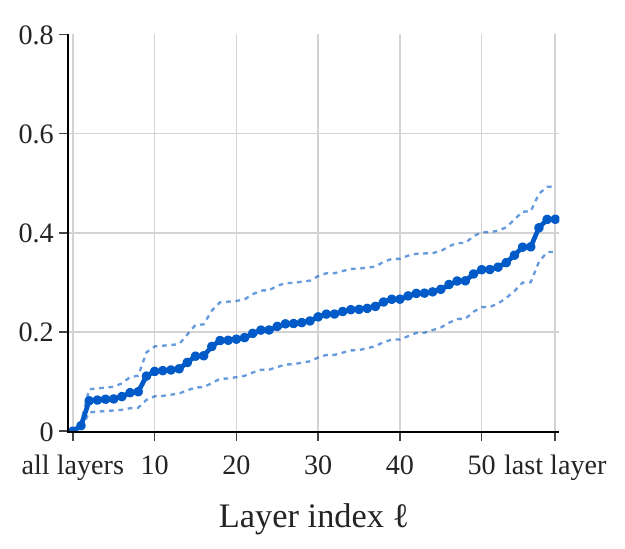}
         \vspace{2pt}
         \includegraphics[width=\textwidth]{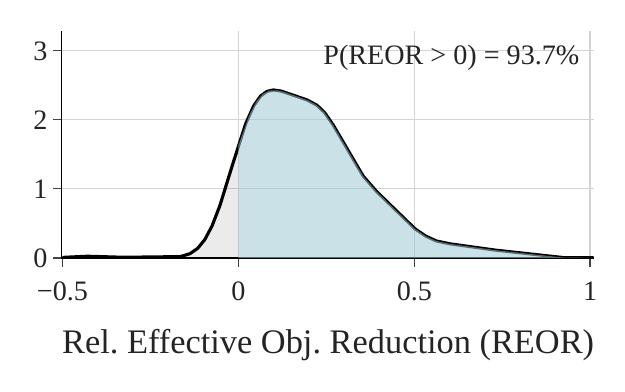}
         \caption{Epoch 50}
     \end{subfigure}
     \hfill
     \begin{subfigure}[b]{0.24\textwidth}
         \centering
         \includegraphics[width=\textwidth]{./figures_optim_approx/epoch_100_rel_l2_c100_vit_small_bs128.pdf}
         \vspace{2pt}
         \includegraphics[width=\textwidth]{./figures_optim_approx/epoch_100_improvement_relaxed_c100_vit_small_bs128.pdf}
         \caption{Epoch 100}
     \end{subfigure}
     \hfill
     \begin{subfigure}[b]{0.24\textwidth}
         \centering
         \includegraphics[width=\textwidth]{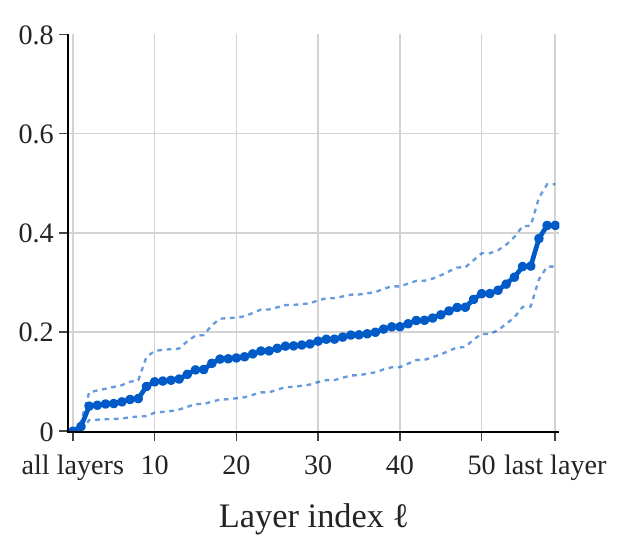}
         \vspace{2pt}
         \includegraphics[width=\textwidth]{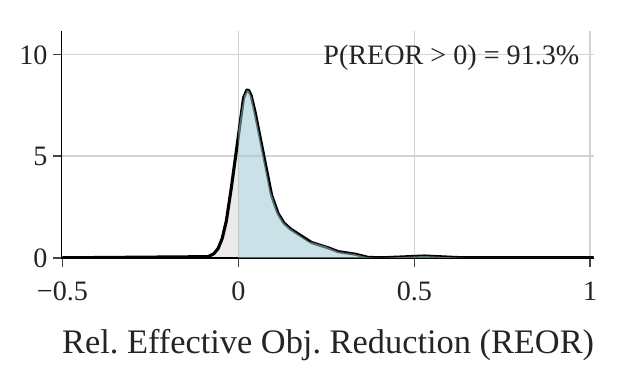}
         \caption{Epoch 150}
     \end{subfigure}     
     \caption{Approximation quality for ViT-small on CIFAR-100 for Batch Size = 128 at different training epochs. \textbf{Top Row:} Relative $L_2$ error of the Gram matrix approximation $\mathbf{Q}_1$. \textbf{Bottom Row:} Distribution of the relative effective objective reduction (REOR).}
        \label{fig:rel_l2_c100_vit_small_bs128}
\end{figure}

\clearpage

\section{Additional figures on the efficiency of the solver}\label{appendix:solver_efficiency_plots}
We present additional results demonstrating the efficiency of the proposed solver (detailed in \cref{appendix:chambolle_pock_details}) on the optimization problem \eqref{eq:app_problem}. We evaluate the solver's performance based on convergence speed and the quality of the resulting parameter solutions. Mirroring the main analysis, solution quality is quantified via the \textit{relative solver objective reduction}, defined as:
\begin{equation}
    \frac{H(\mathbf{q}) - H(-\mathbf{Q}_X\bm{w}^\star + \mathbf{q})}{H(\mathbf{q})}
\end{equation}
where $\mathbf{Q}_X$ corresponds to the matrix $\mathbf{Q}_L$ for SGD (\cref{eq:opt_problem_sgd}) or $\mathbf{Q}_{1,L}$ for AdamW (\cref{eq:opt_problem_adamw}) and $\bm{w}^\star$ is the solver's output optimal solution.

The following figures illustrate these convergence metrics across varying batch sizes. In each pair, the top plot highlights the relative objective reduction (interference avoided), while the bottom plot tracks the number of iterations required by the CP solver to reach a solution. The maximum allowed number of iterations our solver may use is set at $I_{solver}=50$.

\begin{figure}[h!]
    \centering
    \begin{subfigure}[b]{0.24\textwidth}
        \centering
        \includegraphics[width=\textwidth]{./figures_solver/c10_res_BS128_obj_reduction.pdf}
        \vspace{2pt}
        \includegraphics[width=\textwidth]{./figures_solver/c10_res_BS128_num_iter.pdf}
        \caption{Batch Size = 128}
    \end{subfigure}
    \hfill
    \begin{subfigure}[b]{0.24\textwidth}
        \centering
        \includegraphics[width=\textwidth]{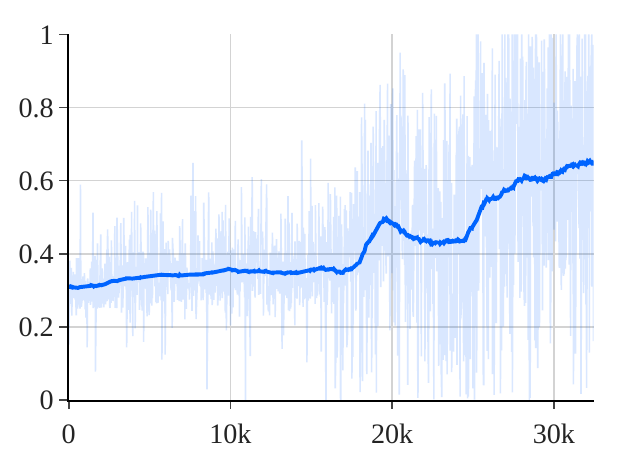}
        \vspace{2pt}
        \includegraphics[width=\textwidth]{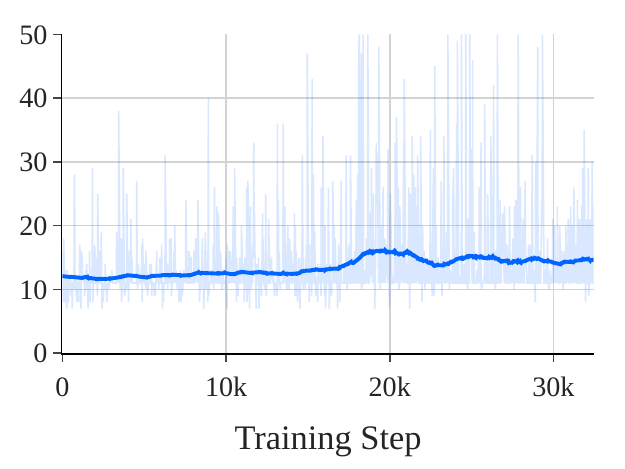}
        \caption{Batch Size = 256}
    \end{subfigure}
    \hfill
    \begin{subfigure}[b]{0.24\textwidth}
        \centering
        \includegraphics[width=\textwidth]{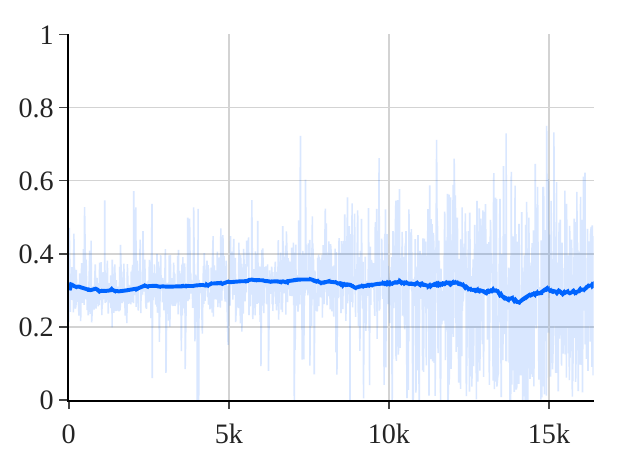}
        \vspace{2pt}
        \includegraphics[width=\textwidth]{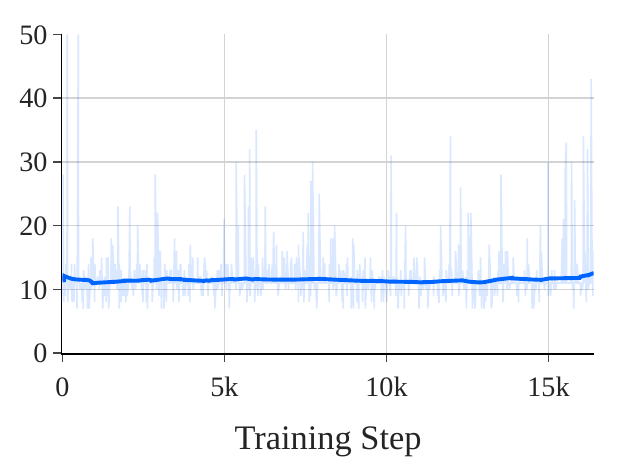}
        \caption{Batch Size = 512}
    \end{subfigure}
    \hfill
    \begin{subfigure}[b]{0.24\textwidth}
        \centering
        \includegraphics[width=\textwidth]{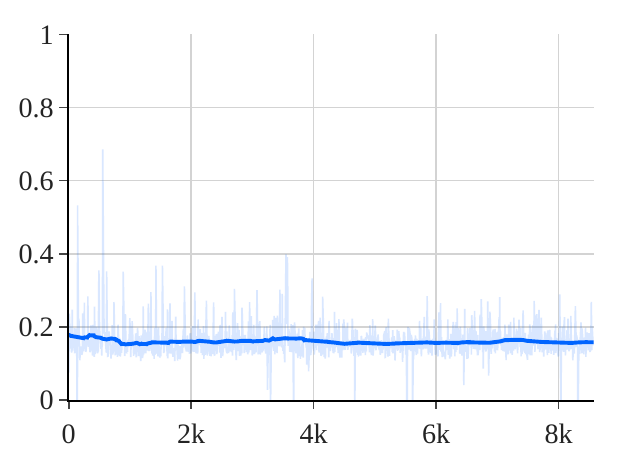}
        \vspace{2pt}
        \includegraphics[width=\textwidth]{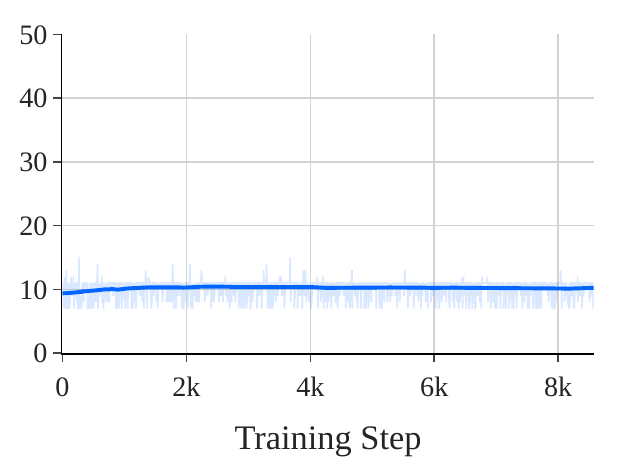}
        \caption{Batch Size = 1024}
    \end{subfigure}
    \caption{Relative objective reduction (top) and number of solver iterations (bottom) across different batch sizes for ResNet-44 on CIFAR-10.}
    \label{fig:solver_convergence_bs_comparison_res_C10}
\end{figure}

\begin{figure}[h!]
    \centering
    \begin{subfigure}[b]{0.24\textwidth}
        \centering
        \includegraphics[width=\textwidth]{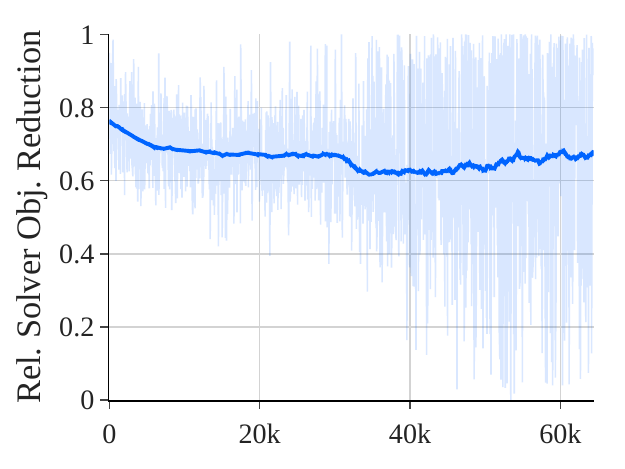}
        \vspace{2pt}
        \includegraphics[width=\textwidth]{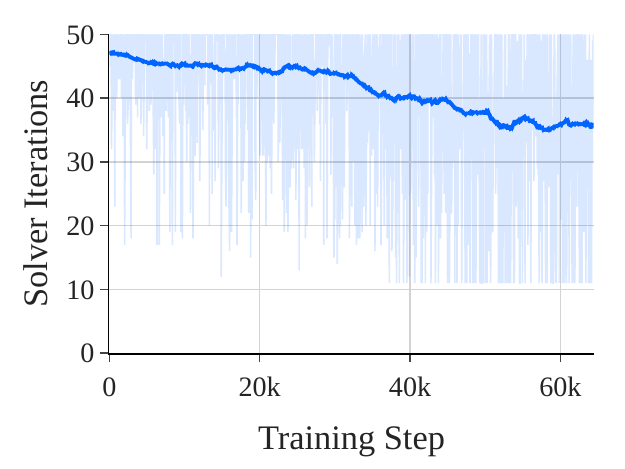}
        \caption{Batch Size = 128}
    \end{subfigure}
    \hfill
    \begin{subfigure}[b]{0.24\textwidth}
        \centering
        \includegraphics[width=\textwidth]{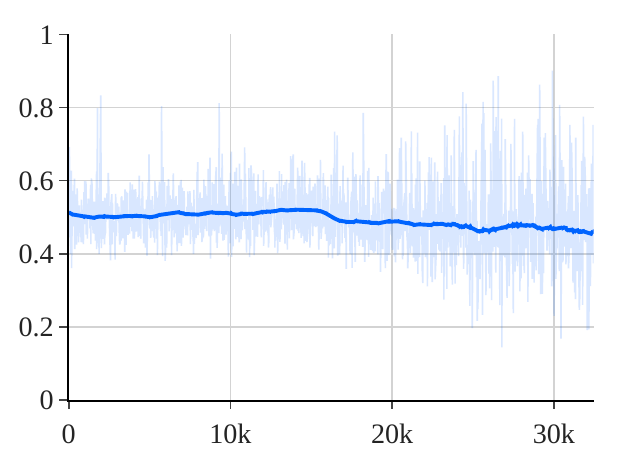}
        \vspace{2pt}
        \includegraphics[width=\textwidth]{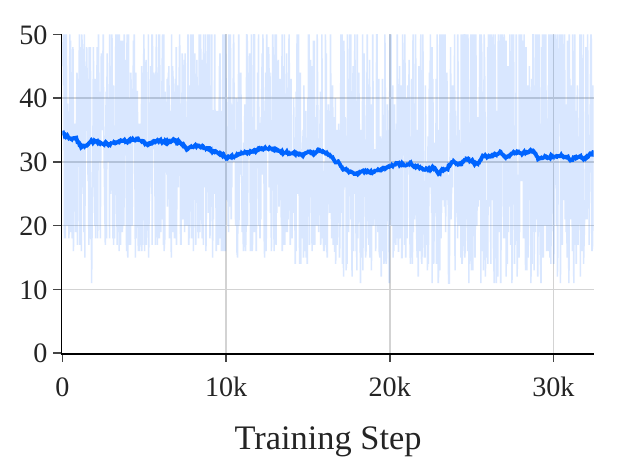}
        \caption{Batch Size = 256}
    \end{subfigure}
    \hfill
    \begin{subfigure}[b]{0.24\textwidth}
        \centering
        \includegraphics[width=\textwidth]{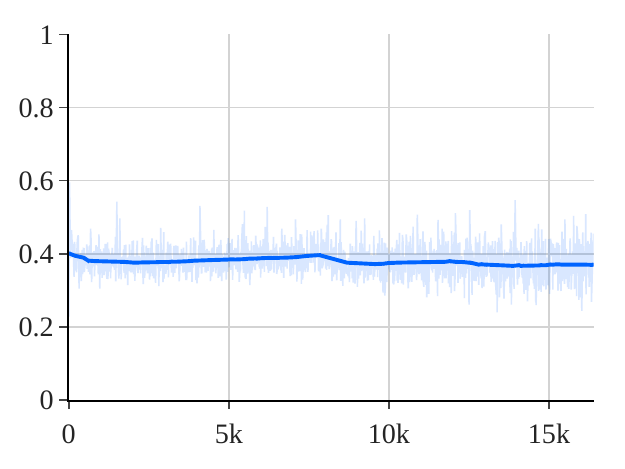}
        \vspace{2pt}
        \includegraphics[width=\textwidth]{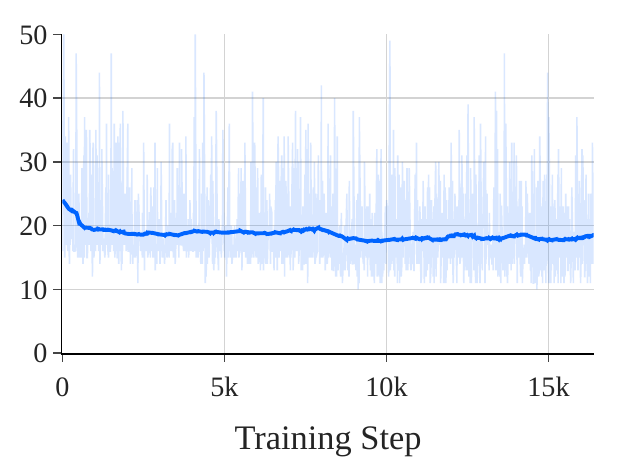}
        \caption{Batch Size = 512}
    \end{subfigure}
    \hfill
    \begin{subfigure}[b]{0.24\textwidth}
        \centering
        \includegraphics[width=\textwidth]{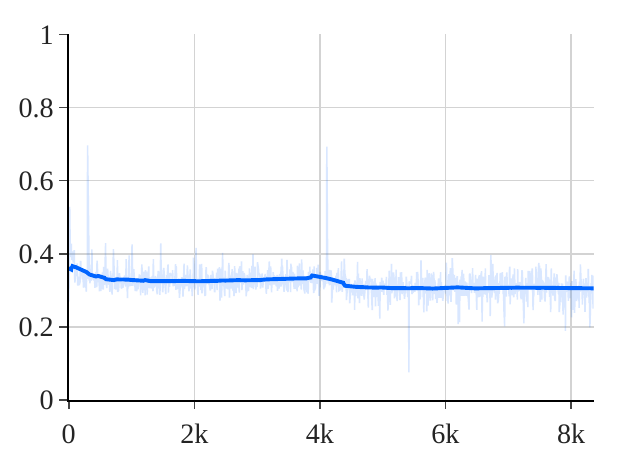}
        \vspace{2pt}
        \includegraphics[width=\textwidth]{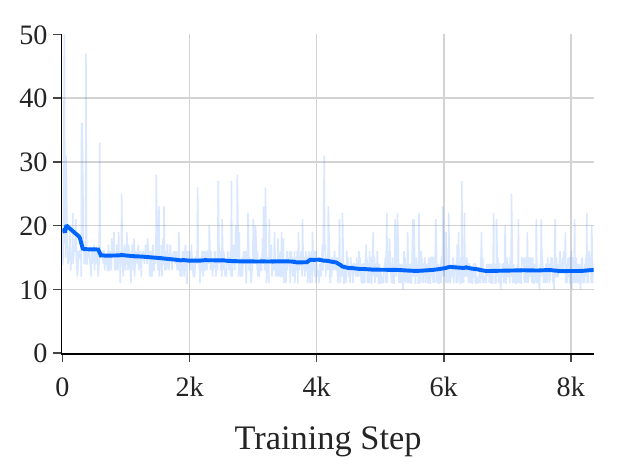}
        \caption{Batch Size = 1024}
    \end{subfigure}
    \caption{Relative objective reduction (top) and number of solver iterations (bottom) across different batch sizes for ResNet-44 on CIFAR-100.}
    \label{fig:solver_convergence_bs_comparison_res_C100}
\end{figure}

\begin{figure}[h!]
    \centering
    \begin{subfigure}[b]{0.24\textwidth}
        \centering
        \includegraphics[width=\textwidth]{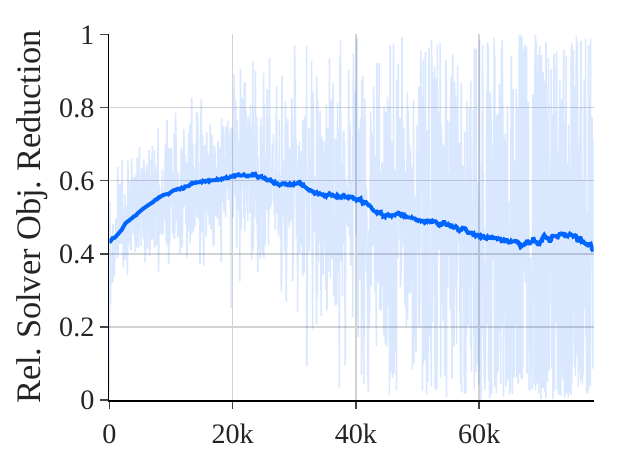}
        \vspace{2pt}
        \includegraphics[width=\textwidth]{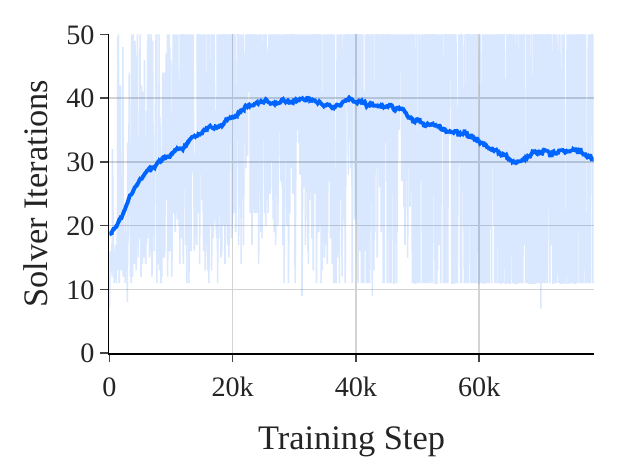}
        \caption{Batch Size = 128}
    \end{subfigure}
    \hfill
    \begin{subfigure}[b]{0.24\textwidth}
        \centering
        \includegraphics[width=\textwidth]{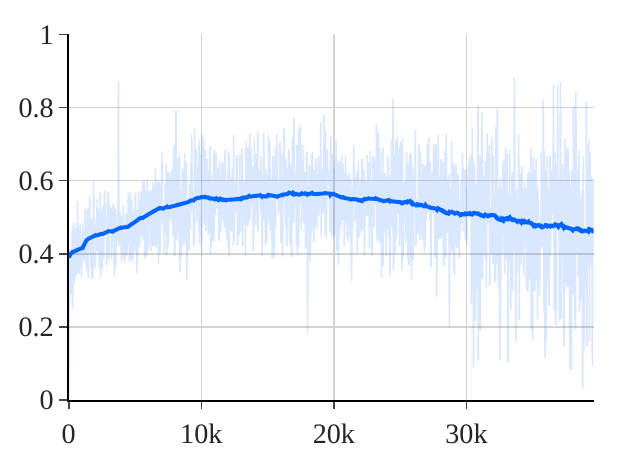}
        \vspace{2pt}
        \includegraphics[width=\textwidth]{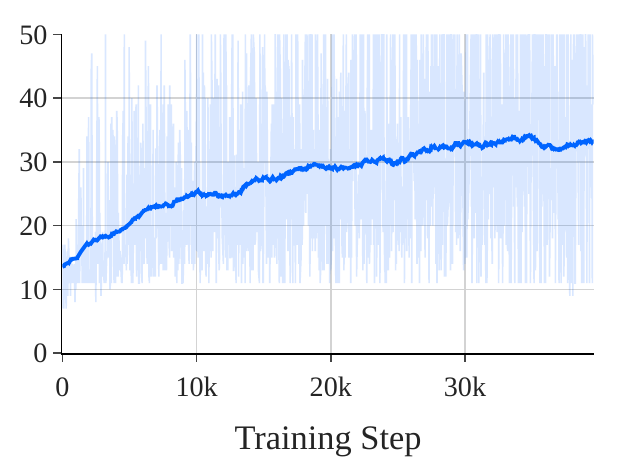}
        \caption{Batch Size = 256}
    \end{subfigure}
    \hfill
    \begin{subfigure}[b]{0.24\textwidth}
        \centering
        \includegraphics[width=\textwidth]{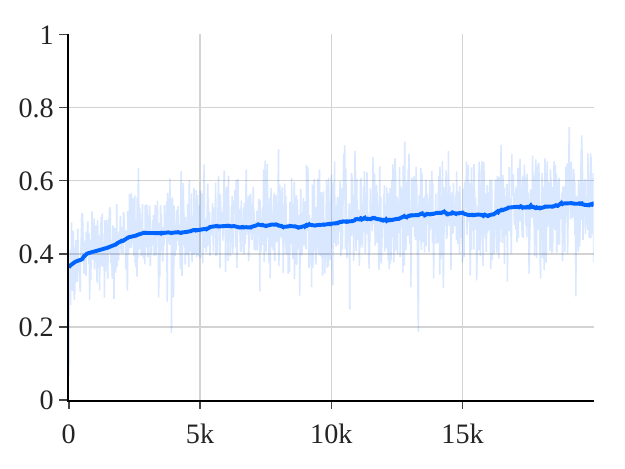}
        \vspace{2pt}
        \includegraphics[width=\textwidth]{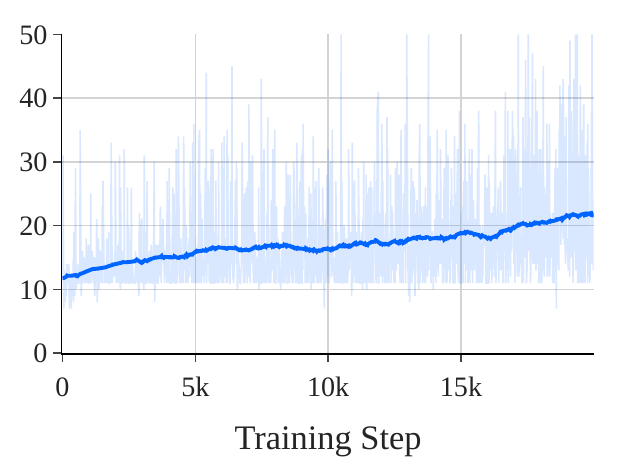}
        \caption{Batch Size = 512}
    \end{subfigure}
    \hfill
    \begin{subfigure}[b]{0.24\textwidth}
        \centering
        \includegraphics[width=\textwidth]{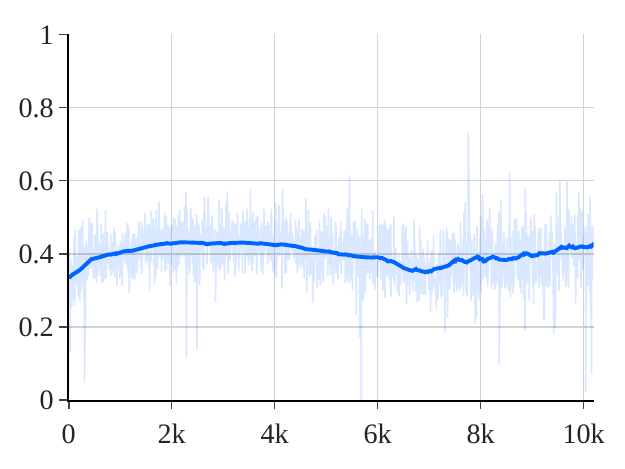}
        \vspace{2pt}
        \includegraphics[width=\textwidth]{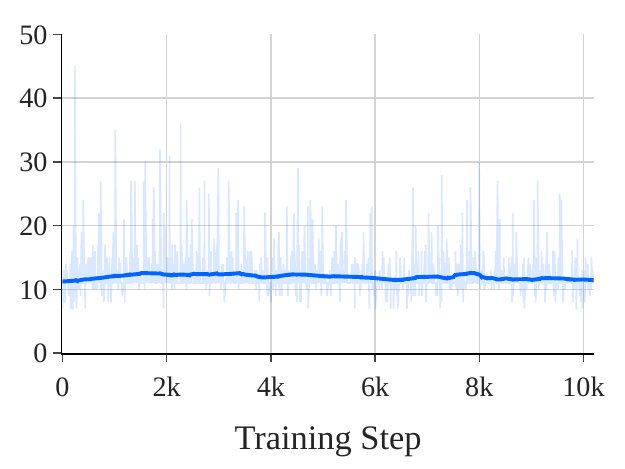}
        \caption{Batch Size = 1024}
    \end{subfigure}
    \caption{Relative objective reduction (top) and number of solver iterations (bottom) across different batch sizes for ViT-c on CIFAR-10.}
    \label{fig:solver_convergence_bs_comparison_vitC10}
\end{figure}

\begin{figure}[h!]
    \centering
    \begin{subfigure}[b]{0.24\textwidth}
        \centering
        \includegraphics[width=\textwidth]{./figures_solver/c100_vit_BS128_obj_reduction.pdf}
        \vspace{2pt}
        \includegraphics[width=\textwidth]{./figures_solver/c100_vit_BS128_num_iter.pdf}
        \caption{Batch Size = 128}
    \end{subfigure}
    \hfill
    \begin{subfigure}[b]{0.24\textwidth}
        \centering
        \includegraphics[width=\textwidth]{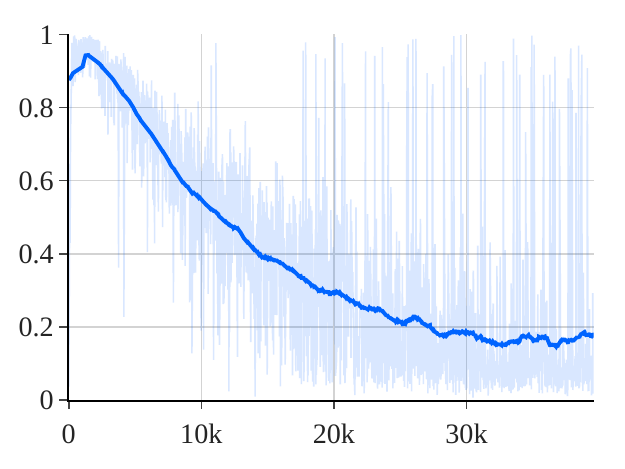}
        \vspace{2pt}
        \includegraphics[width=\textwidth]{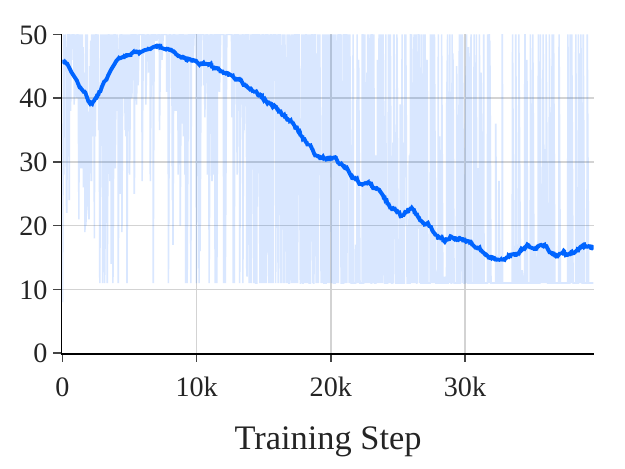}
        \caption{Batch Size = 256}
    \end{subfigure}
    \hfill
    \begin{subfigure}[b]{0.24\textwidth}
        \centering
        \includegraphics[width=\textwidth]{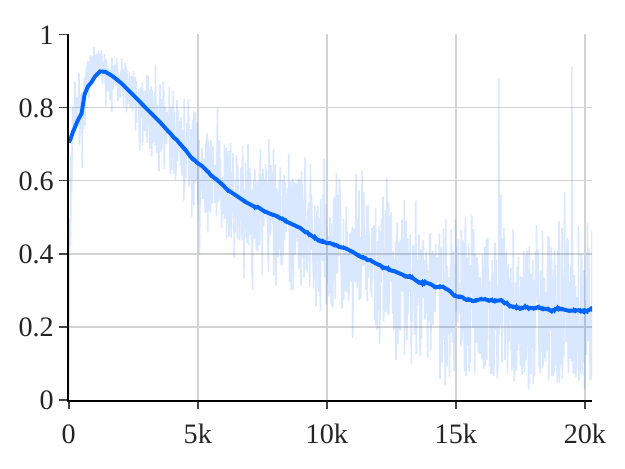}
        \vspace{2pt}
        \includegraphics[width=\textwidth]{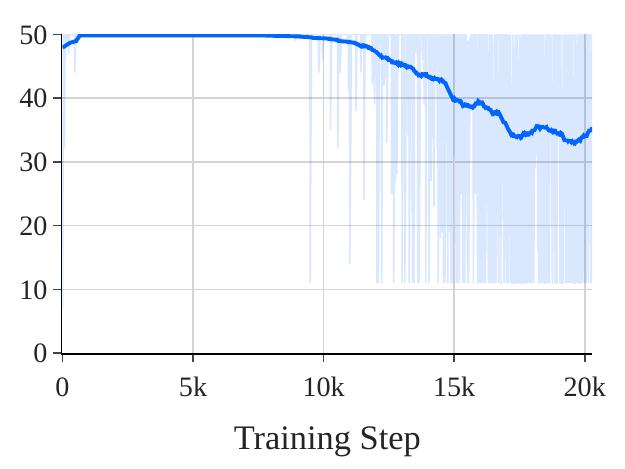}
        \caption{Batch Size = 512}
    \end{subfigure}
    \hfill
    \begin{subfigure}[b]{0.24\textwidth}
        \centering
        \includegraphics[width=\textwidth]{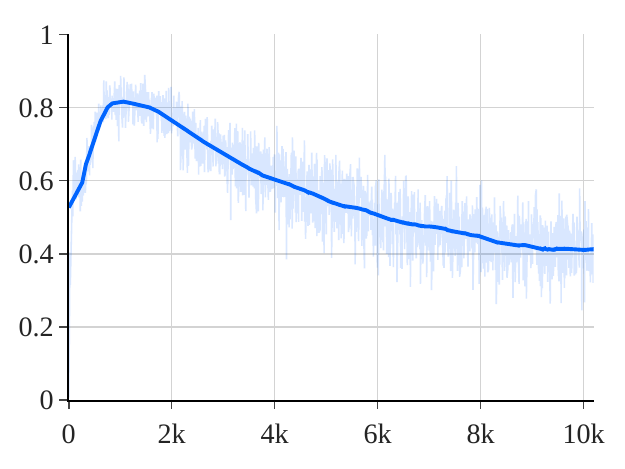}
        \vspace{2pt}
        \includegraphics[width=\textwidth]{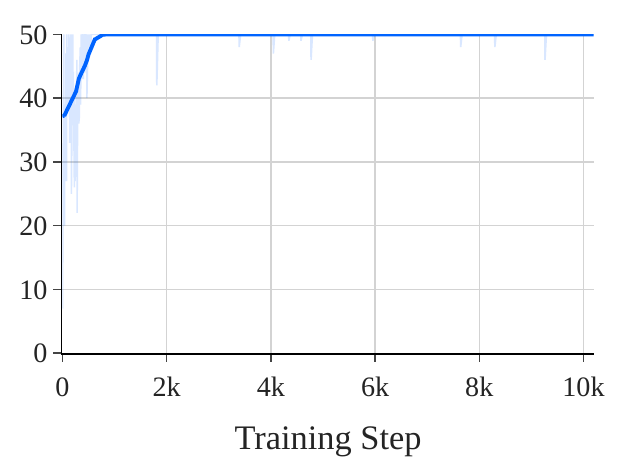}
        \caption{Batch Size = 1024}
    \end{subfigure}
    \caption{Relative objective reduction (top) and number of solver iterations (bottom) across different batch sizes for ViT-c on CIFAR-100.}
    \label{fig:solver_convergence_bs_comparison_vitC100}
\end{figure}

\section{Generalized Multi-Layer Algorithm and Precision Analysis}\label{sec:more layers full alg}

\cref{alg:training_loop_multi} details the generalized training iteration using SGD when expanding the approximation to include multiple layers ($\ell < L$). We decompose the architecture as $F(\mathbf{x};\boldsymbol{\theta}) = F_{\ell:L}(F_{1:\ell-1}(\mathbf{x}; \boldsymbol{\theta}_{1:\ell-1}); \boldsymbol{\theta}_{\ell:L})$. Because intermediate blocks contain also non-linear operations, we explicitly compute the per-sample gradients $\mathbf{G}_{\ell:L}$ using PyTorch's \texttt{vmap}. For notational clarity, we treat the quantities $\mathbf{G}_{\ell:L}$ and $\nabla_{y_\ell}\mathcal{L}$ as flattened two-dimensional tensors of shape $[B, D]$, though in practice they may possess arbitrary trailing dimensions. The multiplication with the weight vector $(\mathbf{1} + \bm{w}^\star)$ denotes a weighted sum across the batch dimension. In modern deep learning frameworks, this can be efficiently implemented using Einstein summation (e.g., \texttt{einsum('b...,b->...', tensor, weights)}). The extension to AdamW is straightforward with the difference being that the optimal $\bm{w}^\star$ is obtained by solving \cref{eq:opt_problem_adamw} instead of \cref{eq:opt_problem_sgd}.

\begin{algorithm}[h!]
\caption{Generalized Training Iteration for SGD (Multi-Layer)}
\label{alg:training_loop_multi}
\begin{algorithmic}[1]
\STATE {\bfseries Input:} Batch $\mathbf{x}$
\STATE $y_\ell \leftarrow F_{1:\ell-1}(\mathbf{x}; \boldsymbol{\theta}_{1:\ell-1})$ 
\STATE $\mathcal{L} \leftarrow \mathcal{L}(F_{\ell:L}(y_\ell; \boldsymbol{\theta}_{\ell:L}))$ 
\STATE Compute per-sample gradients $\mathbf{G}_{\ell:L}$ and $\nabla_{y_\ell} \mathcal{L}$ 
\STATE $\Delta\boldsymbol{\theta}_{b,\ell:L} \leftarrow \text{SGD}(\mathbf{G}_{\ell:L}^\top \mathbf{1})$ 
\STATE Solve \cref{eq:opt_problem_sgd} for $\bm{w}^\star$ \; \COMMENT{\cref{alg:cp_full} in \cref{appendix:chambolle_pock_details}}
\STATE $\widetilde{\mathbf{g}}_{\ell:L} \leftarrow \mathbf{G}_{\ell:L}^\top (\mathbf{1} + \bm{w}^\star)$ 
\STATE $\bar{\delta}_{y_\ell} \leftarrow (\nabla_{y_\ell}\mathcal{L})^\top (\mathbf{1} + \bm{w}^\star)$ 
\STATE $\widetilde{\mathbf{g}}_{1:\ell-1} \leftarrow \text{backward}(\bar{\delta}_{y_\ell}, F_{1:\ell-1})$ 
\STATE Execute an SGD($[\widetilde{\mathbf{g}}_{1:\ell-1}, \widetilde{\mathbf{g}}_{\ell:L}]$) step
\end{algorithmic}
\end{algorithm}

\paragraph{Impact of Per-Sample Gradients on Precision:}
We investigate the numerical stability of explicitly computing per-sample gradients via \texttt{vmap}. In \cref{fig:efficiency_vmap_resnet_bs128}, we show the effect of increasing the number of blocks handled by \texttt{vmap} (i.e., decreasing $\ell$). For this experiment, we set $\bm{w}=\mathbf{0}$, making the parameter update theoretically equivalent to standard SGD. As expected, time per batch and peak memory increase substantially as more blocks are included. However, we also observe fluctuations in test accuracy depending on the approximation depth. Theoretically, these fluctuations should not occur: all experiments use identical random seeds, and summing per-sample gradients is mathematically equivalent to computing the batch-averaged gradients directly. Their presence indicates that explicitly materializing and summing per-sample gradients over numerous layers can introduce numerical imprecisions compared to the standard, fused backward pass. This finding provides further motivation for our design choice: restricting the per-sample gradient computation exclusively to the last linear layer as it minimizes not only computational and memory overheads but also potential numerical drift.

\begin{figure}[h!]
    \centering
    \includegraphics[width=0.6\columnwidth]{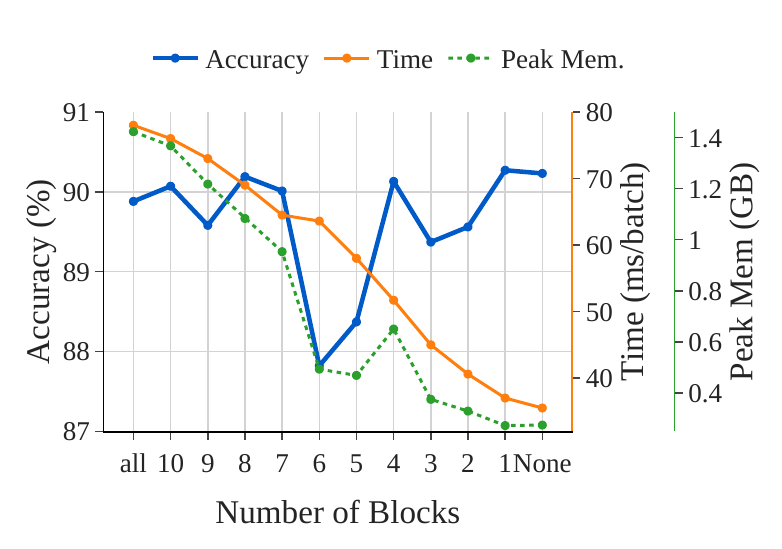}
    \caption{Efficiency comparison for ResNet-20-gn (accuracy, time per batch, and peak memory across number of blocks).}
    \label{fig:efficiency_vmap_resnet_bs128}
\end{figure}

\section{Ablation on the Trust-Region Radius}
\label{appendix:trust_region_ablation}

In our main experiments, we fixed the trust-region radius hyperparameter to $c=0.2$. This choice, along with other algorithmic parameters, was established using early exploratory runs on ResNet-44 (with Group Normalization) with a batch size of 128 as the primary driver. To better understand the sensitivity of our method to this constraint, we conducted an ablation study on ResNet-44 across both CIFAR-10 and CIFAR-100 for batch sizes 128 and 512. The results are averaged over 3 random seeds (except for the cases of $c=0$ and $c=0.2$ where it is over 5 seeds as we use the result from \cref{sec:experiments}) and they are summarized in \cref{tab:trust_region_ablation}. Note that $c=0.0$, enforcing a strictly zero radius, coincides with the standard training baseline.

\begin{table}[h]
\caption{Test accuracy (\%) of ResNet-44 under varying trust-region radius values. The baseline corresponds to $c=0.0$.}
\label{tab:trust_region_ablation}
\centering
\begin{small}
\begin{sc}
\begin{tabular}{l|cccccc}
\toprule
 & \multicolumn{6}{c}{\textbf{Trust-Region Radius $\bm{c}$}} \\
\textbf{Dataset, BS} & \textbf{0.0} & \textbf{0.05} & \textbf{0.1} & \textbf{0.2} & \textbf{0.4} & \textbf{0.8} \\
\midrule
C10, BS=128  & 91.71 & 91.62 & 91.66 & \textbf{91.93} & 91.84 & 91.65 \\
C10, BS=512  & 90.60 & 90.82 & \textbf{90.85} & 90.75 & 90.67 & 85.59 \\
\midrule
C100, BS=128 & 65.22 & 65.54 & 65.73 & 65.50 & \textbf{66.11} & 66.01 \\
C100, BS=512 & 63.51 & 63.63 & 63.79 & 64.37 & 64.88 & \textbf{65.02} \\
\bottomrule
\end{tabular}
\end{sc}
\end{small}
\end{table}

For CIFAR-10, selecting any value between 0.1 and 0.4 yields robust improvements over the baseline. However, as $c$ increases considerably, potential instabilities may occur (for example, at BS=512 with $c=0.8$, the accuracy drops sharply to 85.59\%). Interestingly, for the more complex CIFAR-100 dataset, we observe that performance generally continues to improve as the trust-region constraint is relaxed. 

While our default configuration generalized well, these results suggest that better performance might be possible for a broader range of datasets and architectures by specifically tuning the trust region $c$ or maybe varying the weight constraint box $\mathbf{w} \le \mathbf{2}$. Beyond this hyperparameter tuning, performance improvements might be realized by refining the solver, developing stronger numerical safeguards targeting the robustness of the surrogate optimization problem, or even fundamentally modifying the mathematical definition of per-sample interference.

\end{document}